\documentclass[twoside,10pt,preprint]{article}

\usepackage{jmlr2e}

\usepackage{amsmath,amssymb,mathrsfs,bm,mathtools}
\usepackage{latexsym,amscd,amsbsy,dsfont,amsfonts,xfrac}
\usepackage{physics,tensor,empheq,enumitem}
\usepackage{tikz}

\usepackage[cal=euler]{mathalfa}
\usepackage{libertine}
\usepackage[T1]{fontenc}

\usepackage{hyperref}
\hypersetup{pdfauthor={IdePHICS},pdftitle={Ensembling},%
            colorlinks, linktocpage=true, pdfstartpage=1, pdfstartview=FitV,%
    breaklinks=true, pdfpagemode=UseNone, pageanchor=true, pdfpagemode=UseOutlines,%
    plainpages=false, bookmarksnumbered, bookmarksopen=true, bookmarksopenlevel=1,%
    hypertexnames=true, pdfhighlight=/O,%
    urlcolor=orange, linkcolor=blue, citecolor=blue, 
        }

\newcommand{\comprimi}{\medmuskip=0mu
\thinmuskip=0mu
\thickmuskip=0mu}

\DeclareMathOperator{\e}{e}

\def\sign{\text{sign}}
\def\erf{\text{erf}}
\def\diag{\text{diag}}

\renewcommand{\vec}[1]{\bm{#1}}

\newcommand{\R}{{\mathds{{R}}}}

\newcommand{\by}{{\mathbf{y}}}
\newcommand{\be}{{\mathbf{e}}}
\newcommand{\bx}{{\vec{{x}}}}
\newcommand{\bz}{{\vec{{z}}}}

\newcommand{\bw}{{\vec{{w}}}}
\newcommand{\bff}{{\vec{{f}}}}
\newcommand{\bmu}{{\vec{{\mu}}}}

\newcommand{\btheta}{{\boldsymbol{{\theta}}}}

\newcommand{\bsTheta}{{\boldsymbol{\mathsf{\Theta}}}}
\newcommand{\bTheta}{{\boldsymbol{{\Theta}}}}
\newcommand{\bxi}{{\boldsymbol{{\xi}}}}
\newcommand{\bXi}{{\boldsymbol{{\Xi}}}}
\newcommand{\bzeta}{{\boldsymbol{{\zeta}}}}
\newcommand{\bomega}{{\boldsymbol{{\omega}}}}
\newcommand{\bI}{{\vec{{I}}}}
\newcommand{\bUno}{{\vec{{1}}}}
\newcommand{\bM}{{\vec{{m}}}}
\newcommand{\bV}{{\vec{{V}}}}
\newcommand{\bv}{{\vec{{v}}}}
\newcommand{\bQ}{{\vec{{Q}}}}
\newcommand{\bR}{{\vec{{R}}}}
\newcommand{\bsA}{{\vec{{\mathsf A}}}}

\newcommand{\bsB}{{\vec{{\mathsf B}}}}

\newcommand{\bsC}{{\vec{{\mathsf C}}}}

\newcommand{\bsM}{{\vec{{\mathsf m}}}}
\newcommand{\bsV}{{\vec{{\mathsf V}}}}
\newcommand{\bsv}{{\vec{{\mathsf v}}}}

\newcommand{\bsQ}{{\vec{{\mathsf Q}}}}

\newcommand{\bA}{{\vec{{A}}}}

\newcommand{\bG}{{\vec{{G}}}}
\newcommand{\bfB}{{\vec{{B}}}}

\newcommand{\bU}{{\vec{{U}}}}
\newcommand{\bS}{{\vec{{S}}}}
\newcommand{\bu}{{\vec{{u}}}}

\newcommand{\bF}{{\vec{{F}}}}
\newcommand{\bX}{{\vec{{X}}}}
\newcommand{\bY}{{\vec{{Y}}}}
\newcommand{\bZ}{{\vec{{Z}}}}
\newcommand{\bhM}{{\vec{{\hat m}}}}
\newcommand{\bhV}{{\vec{{\hat V}}}}
\newcommand{\bhQ}{{\vec{{\hat Q}}}}
\newcommand{\bhR}{{\vec{{\hat R}}}}
\newcommand{\bW}{{\vec{{W}}}}

\newcommand{\bhPhi}{{\boldsymbol{\hat{\Phi}}}}

\newcommand{\bhS}{{\hat{\mathbf{S}}}}

\newcommand{\bOmega}{{\boldsymbol{{\Omega}}}}
\newcommand{\bSigma}{{\boldsymbol{{\Sigma}}}}

\newcommand{\bh}{{\vec{{h}}}}

\newcommand{\bsOmega}{{\vec{\mathsf \Omega}}}
\newcommand{\bphi}{{\boldsymbol{\phi}}}
\newcommand{\bT}{{\boldsymbol{T}}}
\DeclareMathOperator{\Prox}{prox}

\makeatletter
\newsavebox{\@brx}
\newcommand{\llangle}[1][]{\savebox{\@brx}{\(\m@th{#1\langle}\)}%
  \mathopen{\copy\@brx\kern-0.5\wd\@brx\usebox{\@brx}}}
\newcommand{\rrangle}[1][]{\savebox{\@brx}{\(\m@th{#1\rangle}\)}%
  \mathclose{\copy\@brx\kern-0.5\wd\@brx\usebox{\@brx}}}
\makeatother

\jmlrheading{1}{2021}{1-48}{4/00}{10/00}{meila00a}{Bruno Loureiro, C\'edric Gerbelot, Maria Refinetti, Gabriele Sicuro and Florent Krzakala}

\ShortHeadings{Bias, Variances, Fluctuations, \& Ensembling in high-dimension}{Loureiro, Gerbelot, Refinetti, Sicuro and Krzakala}
\firstpageno{1}

\allowdisplaybreaks

\begin{document}
\title{Fluctuations, Bias, Variance \& Ensemble of Learners:\\
Exact Asymptotics for Convex Losses in High-Dimension}

\author{\name Bruno Loureiro \email bruno.loureiro@epfl.ch\\
       \addr \'Ecole Polytechnique F\'ed\'erale de Lausanne (EPFL)\\
       Information, Learning and Physics lab.\\
       CH-1015 Lausanne, Switzerland
       \AND
       \name C\'edric Gerbelot \email cedric.gerbelot@ens.fr\\
       \addr Laboratoire de Physique de l'ENS, Universit\'e PSL, CNRS, Sorbonne Universit\'e\\
       24 Rue Lhomond, 75005 Paris, France
       \AND
       \name Maria Refinetti \email mariaref@gmail.com\\
       \addr Laboratoire de Physique de l'ENS, Universit\'e PSL, CNRS, Sorbonne Universit\'e\\
       24 Rue Lhomond, 75005 Paris, France
       \AND
       \name Gabriele Sicuro \email gabriele.sicuro@kcl.ac.uk \\
       \addr Department of Mathematics\\
       King's College London\\
       Strand WC2R 2LS, London, United Kingdom
       \AND
       \name Florent Krzakala \email florent.krzakala@epfl.ch\\
       \addr \'Ecole Polytechnique F\'ed\'erale de Lausanne (EPFL)\\
       Information, Learning and Physics lab.\\
       CH-1015 Lausanne, Switzerland}

\maketitle
\begin{abstract}
From the sampling of data to the initialisation of parameters, randomness is ubiquitous in modern Machine Learning practice. Understanding the statistical fluctuations engendered by the different sources of randomness in prediction is therefore key to understanding robust generalisation. In this manuscript we develop a quantitative and rigorous theory for the study of fluctuations in an ensemble of generalised linear models trained on different, but correlated, features in high-dimensions.  In particular, we provide a complete description of the asymptotic joint distribution of the empirical risk minimiser for generic convex loss and regularisation in the high-dimensional limit. Our result encompasses a rich set of classification and regression tasks, such as the lazy regime of overparametrised neural networks, or equivalently the random features approximation of kernels. While allowing to study directly the mitigating effect of ensembling (or bagging) on the bias-variance decomposition of the test error, our analysis also helps disentangle the contribution of statistical fluctuations, and the singular role played by the interpolation threshold
that are at the roots of the ``double-descent'' phenomenon. 
\end{abstract}

\section{Introduction}
\label{sec:intro}

Randomness is ubiquitous in Machine Learning. It is present in the data (e.g., noise in acquisition and annotation), in commonly used statistical models (e.g., random features \citep{rahimi2007random}),  or in the algorithms used to train them (e.g., in the choice of initialisation of weights of neural networks~\citep{narkhede2021review}, or when sampling a mini-batch in Stochastic Gradient Descent \citep{bottou2012stochastic}). Strikingly, fluctuations associated to different sources of randomness can have a major impact in the generalisation performance of a model. For instance, this is the case in least-squares regression with random features, where it has been shown \citep{geiger2020scaling,dascoli2020,jacot2020implicit} that the variance associated with the random projections matrix is responsible for poor generalisation near the interpolation peak \citep{advani2017highdimensional,Spigler_2019,Belkin10627}. As a consequence, this \emph{double-descent} behaviour can be mitigated by averaging over a large \emph{ensemble} of learners, effectively suppressing this variance. Indeed, considering an ensemble (sometimes also refereed to as a committee \citep{drucker1994boosting}) of independent learners provide a natural framework to study the contribution of the variance of prediction in the estimation accuracy. 
In this manuscript we leverage this idea to provide an exact asymptotic characterisation of the statistics of fluctuations in empirical risk minimisation with generic convex losses and penalties in  high-dimensional models. We focus on the case of synthetic datasets, and we apply our results to random feature learning in particular.

\begin{figure}
\centering
\def\layersep{1.5cm}
\begin{tikzpicture}[shorten >=1pt,->,draw=black!50, node distance=\layersep]
  \draw[black,rounded corners=10,dotted,fill=gray!10]
     (0.75*\layersep,0.2) rectangle (2.25*\layersep,1.8);
  \draw[black,rounded corners=10,dotted,fill=gray!10]
     (0.75*\layersep,-0.2) rectangle (2.25*\layersep,-1.8);
    \tikzstyle{every pin edge}=[<-,shorten <=1pt]
    \tikzstyle{neuron}=[circle,draw,fill=black!25,minimum size=10pt,inner sep=0pt]
    \tikzstyle{input neuron}=[neuron, fill=gray!20];
    \tikzstyle{output neuron}=[neuron, fill=white];
    \tikzstyle{hidden neuron}=[neuron, fill=gray!70];
    \tikzstyle{annot} = [text width=4em, text centered]

    \foreach \name / \y in {1,...,4}{
    \node[input neuron] (I-\name) at (-1,-1.25+0.5*\y) {};
    \node[draw=none,inner sep=0] (I1-\name) at (0.1*\layersep,-0.25+0.5*\y) {};
    \node[draw=none,inner sep=0] (I2-\name) at (0.1*\layersep,0.25-0.5*\y) {};
    }

    \foreach \name / \y in {1,...,3}{
        \node[hidden neuron] (H-\name) at (\layersep,0.5*\y cm) {};
        \node[hidden neuron] (H0-\name) at (0.5*\layersep,0.5*\y cm) {};}
    \foreach \name / \y in {4/1,5/2,6/3}{
        \node[hidden neuron] (H-\name) at (\layersep,-0.5*\y cm) {};
        \node[hidden neuron] (H0-\name) at (0.5*\layersep,-0.5*\y cm) {};}

    \node[output neuron, right of=H-2] (O1) {};
    \node[output neuron, right of=H-5] (O2) {};
    \node[output neuron,pin={[pin edge={->}]right:{$\hat y$}}] (O) at (2.5*\layersep,0) {};
    \node[output neuron,pin={[pin edge={->}]left:{$y$}}] (OT) at (-\layersep,0) {};

    \foreach \source in {1,...,6}
        \path (H0-\source) edge (H-\source);
    \foreach \source in {1,...,4}
        \path (I-\source) [out=0, in=180]edge (I1-\source);
    \foreach \source/\dest in {1/4,2/3,3/2,4/1}
        \path (I-\source) [out=0, in=180]edge (I2-\dest);
    \foreach \source in {1,...,3}
        \path (H-\source) edge (O1);
    \foreach \source in {4,...,6}
        \path (H-\source) edge (O2);
    \foreach \source in {1,2}
        \path (O\source) edge (O);
    \foreach \source in {1,...,4}
        \path (I-\source) edge (OT);
    \node[annot,above of=I-4,node distance=10] {$\bx$};
    \draw[gray,rounded corners=1,fill=gray!30]
     (0.1*\layersep,0.05) rectangle (0.65*\layersep,2.1);
    \draw[gray,rounded corners=1,fill=gray!30]
     (0.1*\layersep,-0.05) rectangle (0.65*\layersep,-2.1);
     \node at (0.55,1) {$\bu_1$};
     \node at (0.55,-1) {$\bu_2$};
\end{tikzpicture}    
\caption{Pictorial representation of the model considered in the paper for $K=2$. Two learners with the same architecture (in gray) receive a correlated input generated from the same vector $\bx\sim\mathcal N(\mathbf 0_d,\bI_d)$. The output $\hat y$ is an average of their outputs. While the study of an ensemble of learners is already interesting {\it per se}, it is also pivotal to study the fluctuation between learners, and the error steaming from the difference in the weights in random features and lazy training.}
\label{fig:setting}
\vspace{-0.4cm}
\end{figure}
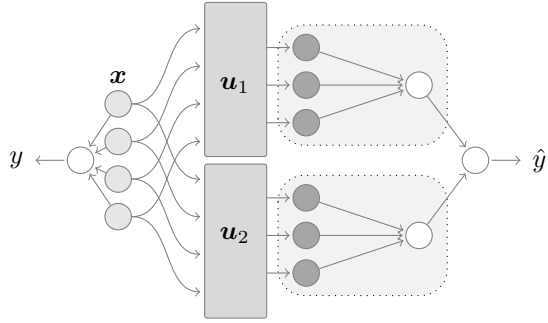

\subsection{Setting} 
Let $\left(\bx^{\mu}, y^{\mu}\right)\in\R^{d}\times\mathcal{Y}$, $\mu\in [n]\coloneqq\{1,\dots,n\}$, denote a labelled data set composed of $n$ independent samples from a joint density $p(\bx,y)$ (e.g., $\mathcal{Y}=\{-1,1\}$ for a binary classification problem). In this manuscript we are interested in studying an ensemble of $K$ parametric predictors, each of them depending on a vector of parameters $\bw_k\in\R^p$, $k\in[K]$, and independently trained on the dataset $\{(\bx^\mu,y^\mu)\}_{\mu\in[n]}$. Note that even if the vectors of parameters $\{\bw_{k}\}_{k\in[K]}$ are trained independently, they correlate through the training data. Statistical fluctuations in the learnt parameters can then arise for different reasons. For instance, a common practice is to initialise the parameters randomly during optimisation, which will induce statistical variability between the different predictors. Alternatively, each predictor could be trained on a subsample of the data, as it is commonly done in bagging \citep{Breiman1996}. The statistical model can also be inherently stochastic, e.g., the random features approximation for kernel methods \citep{rahimi2007random}. Finally, the predictors could also be jointly trained, e.g., coupling them through the loss or penalty as it is done in boosting \citep{schapire1990strength}.

Our goal in this work is to provide a sharp characterisation of the statistical fluctuations of the ensemble of parameters $\{\bw_{k}\}_{k\in[K]}$ in a particular, mathematically tractable, class of predictors: \emph{generalised linear models},
\begin{equation}
\label{eq:intro:estimator}
    \hat{y}(\bx) = \hat{f}\left(\frac{\hat\bw_1^\top\bu_1(\bx)}{\sqrt p},\dots,\frac{\hat\bw_K^\top\bu_K(\bx)}{\sqrt p}\right)
\end{equation}
where $\bu_{k}:\mathbb{R}^{d}\to\mathbb{R}^{p}$, $k\in [K]$ is an ensemble of possibly correlated features and $\hat{f}\colon \R^K\to \mathcal{Y}$ is an activation function. For most of this work, we discuss the case in which the predictors are \emph{independently trained} through regularised empirical risk minimisation:
\begin{equation}
\label{eq:erm}
    \hat{\bw}_{k} = \underset{\bw\in\R^{p}}{\arg\min}\left[\frac{1}{n}\sum_{\mu=1}^n\ell\left(y^\mu,\frac{\bw^\top\bu_k(\bx^\mu)}{\sqrt p}\right)+\frac{\lambda}{2} \|\bw\|^2_{2} \right]
\end{equation}
with a convex loss function $\ell\colon\mathcal{Y}\times \R\to\R$ (e.g., the logistic loss) and ridge penalty whose strength is given by $\lambda\in\R^+$. However, our analysis also includes the case in which the learners are jointly trained with a generic convex penalty. This case will be further discussed in Sec.~\ref{sec:general}.
In what follows we will also concentrate in the random features case where $\bu_{k}(\bx) = \phi\left(\bF_{k}\bx\right)$ with $\phi:\R\to\R$ an activation function acting component-wise and $\bF_{k}\in\R^{p\times d}$ a family of independently sampled random matrices. Besides being an efficient approximation for kernels \citep{rahimi2007random}, random features are often studied as a simple model for neural networks in the lazy  and neural tangent kernel regimes of deep neural networks \citep{NEURIPS2019_ae614c55,NEURIPS2018_5a4be1fa}, in which case the matrices $\bF_{k}$ correspond to different random initialisation of hidden-layer weights. Moreover, the random features model displays some of the exotic behaviours of high-dimensional overparametrised models, such as double-descent \citep{mei2020generalization, pmlr-v119-gerace20a} and benign overfitting \citep{Bartlett30063}, therefore providing an ideal playground to study the interplay between fluctuations and overparametrisation. A broader class of features maps is also discussed in Sec.~\ref{sec:general}.

To provide an exact characterisation of the statistics of the estimators in eq.~\eqref{eq:erm}, we shall assume data is generated from a target
\begin{equation}
y=f_{0}\left(\frac{\btheta^{\top}\bx}{\sqrt{d}}\right), \qquad\btheta\sim\mathcal{N}(\mathbf 0_d,\rho\bI_{d}), \qquad \rho\in\R^+_0,
\end{equation}
with $f_{0}:\R\to\mathcal{Y}$ and $\bI_d$ $d$-dimensional identity matrix. The dataset is then constructed generating {\it i.i.d.} $n$ vectors $\bx^\mu\sim\mathcal N(\mathbf 0_d,\bI_d)$, $\mu\in[n]$. 

An illustration summary of the setting considered here in given in Figure \ref{fig:setting}. Note that such architecture can be interpreted as a two-layer tree neural network, also known in some contexts as the \emph{tree-committee} or \emph{parity machine} \citep{Schwarze1992}.

\begin{figure}[t!]
    \centering
    \includegraphics[height=0.5\columnwidth]{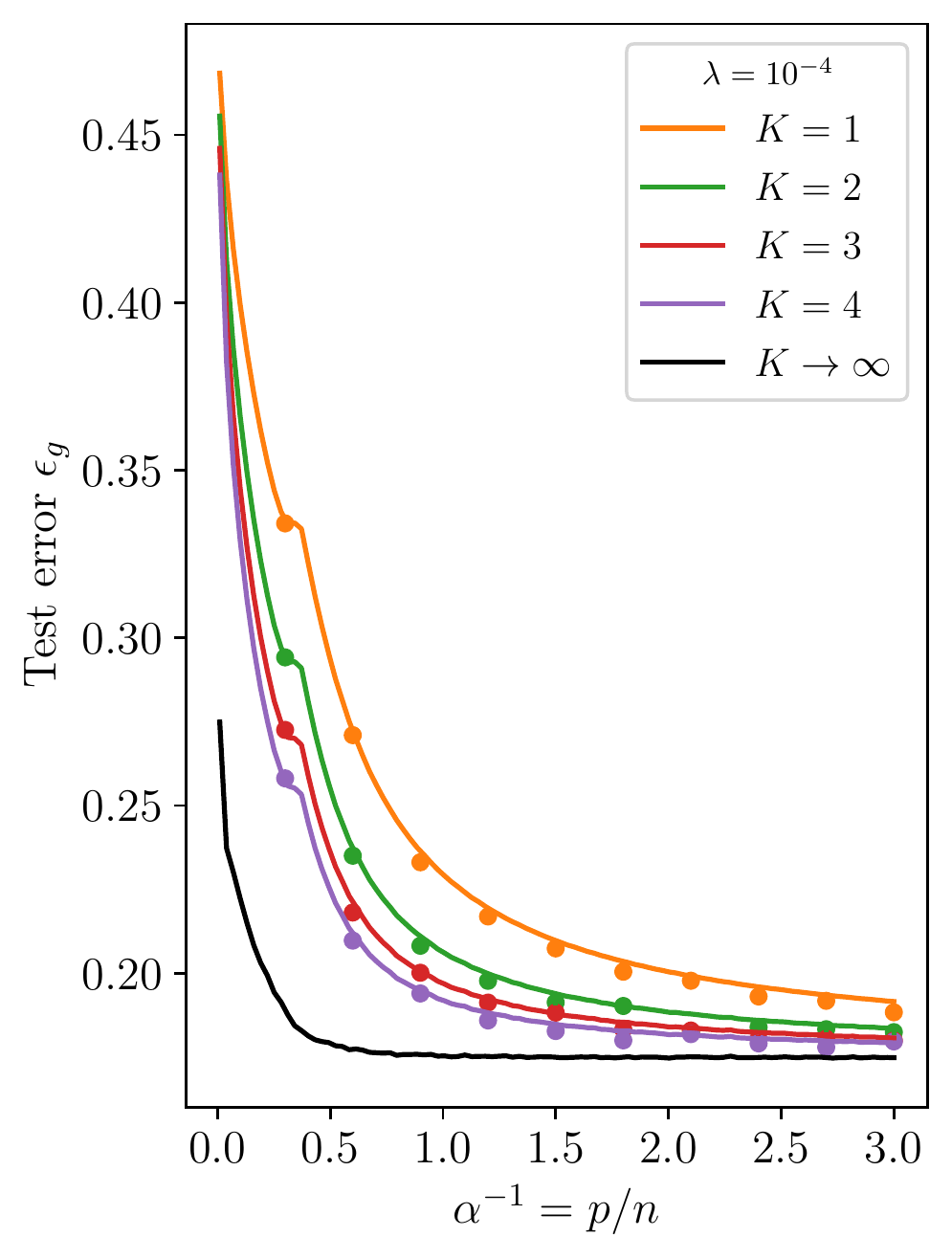}\qquad
    \includegraphics[height=0.5\columnwidth]{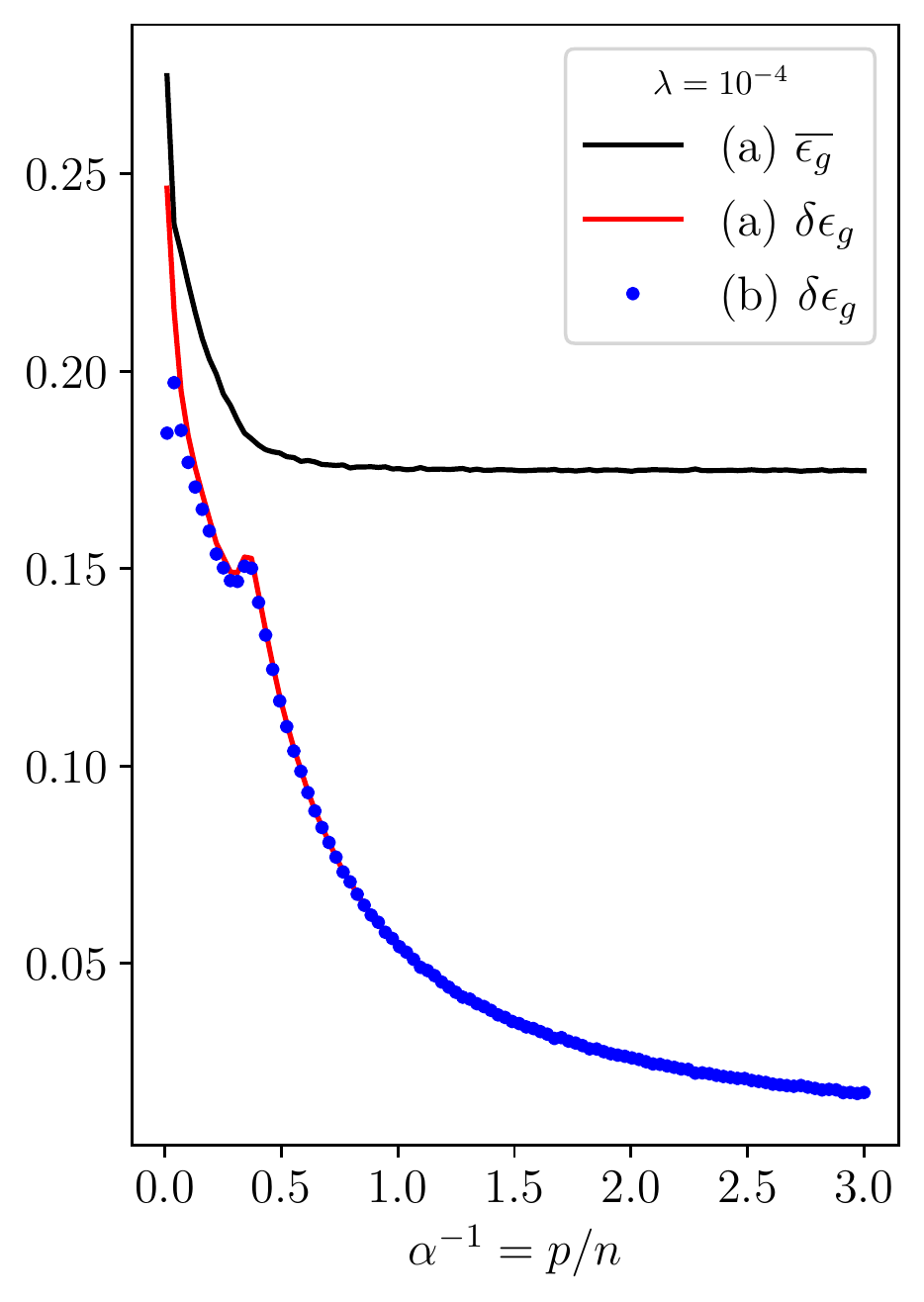}
    \caption{\textit{Left.} Test error for logistic regression with $\lambda=10^{-4}$ and different values of $K$ as function of $\sfrac{p}{n}=\sfrac{1}{\alpha}$ with $\sfrac{n}{d}=2$ and $\rho=1$. Dots represent the average of the outcomes of $10^3$ numerical experiments. Here we adopted $\phi(x)=\mathrm{erf}(x)$ and estimator $\hat f(\bv)=\sign(\sum_k v_k)$. \textit{Right.} Decomposition of the $K=1$ test error $\epsilon_g=\overline{\epsilon_g}+\delta \epsilon_g$ for the estimator \textbf{(a)}, with $\sfrac{n}{d}=2$ and $\lambda=10^{-4}$. We plot also the contribution $\delta\epsilon_g$ corresponding to the estimator \textbf{(b)}: we numerically observed that such decomposition coincides in the two cases. Note also the presence of a kink in $\delta\epsilon_g$ at the interpolation transition.}
    \label{fig:log2}
\end{figure}

\paragraph{Main contributions ---} The results in this manuscript can be listed as follows. 
\begin{itemize}[wide = 1pt]
\item We provide a sharp asymptotic characterisation of the joint statistics of the ensemble of empirical risk minimisers $\{\hat{\vec{w}}_{k}\}_{k\in[K]}$ in the high-dimensional limit where $p, n\to+\infty$ with $\sfrac{n}{p}$ kept constant, for any convex loss and penalty. In particular, we show that the pre-activations $\{\hat{\vec{w}}_{k}^{\top}\vec{u}_{k}\}_{k\in[K]}$ are jointly Gaussian, with sufficient statistics obeying a set of explicit closed-form equations. Note that the analysis of ensembling with non-square losses is out of the grasp of the most commonly adopted theoretical tools (e.g., random matrix theory). Therefore, our proof method based on recent progress on Approximate Message Passing techniques \citep{javanmard2013state,berthier2020state,gerbelot2021graph} is of independent interest. Different versions of our theorem are discussed throughout the manuscript. First, in Sec.~\ref{sec:mainres} for the particular case of independently trained learners on random features (Theorem \ref{th:TeoSimplified}). Later, in Sec.~\ref{sec:general} for the general case of jointly trained learners on correlated Gaussian covariates (Theorem \ref{th:TheGen}).

\item We discuss the role played by fluctuations in the non-monotonic behaviour of the generalisation performance of interpolators (a.k.a.~double-descent behaviour). In particular ---as discussed in \citep{geiger2020scaling,dAscoli_2021} for the ridge case---  the interpolation peak arises from the model overfitting the particular realisation of the random weights. We show the test error can be decomposed $\epsilon_g(K=1) =\overline{\epsilon_g} + \delta\epsilon_g$ in terms of a fluctuation-free term $\overline{\epsilon_g}$ and a fluctuation term $\delta\epsilon_g$ responsible for the double-descent behavior, see Fig.~\ref{fig:log2} for the case of max-margin classification. 

\item In the context of classification, we discuss how \emph{majority vote} and \emph{score averaging}, two popular ensembling procedures, compare in terms of generalisation performance. More specifically, we show that in the setting we study score averaging consistently outperforms the majority vote predictor. However, for a large number of learners $K\gg 1$ these two predictors agree, see Fig.~\ref{fig:log1} (right).

\item Finally, we discuss how ensembling can be used as a tool for uncertainty quantification. In particular, we connect the correlation between two learners to the probability of disagreement, and show that it decreases with overparametrisation, see Fig.~\ref{fig:log1} (center). We provide a full characterisation of the joint probability density of the confidence score between two independent learners, see Fig.~\ref{fig:log1} (left).

\end{itemize}

\paragraph{Related works ---}
The idea of reducing the variance of a predictor by averaging over independent learners is quite old in Machine Learning \citep{Hansen1990, Perrone93whennetworks, NIPS1993_0537fb40, NIPS1994_b8c37e33}, and an early asymptotic analysis of the regression case was given in \cite{Krogh1997}. In particular, a variety of methods to combine an ensemble of learners appeared in the literature \citep{Opitz1999}. In a very inspiring work, \citet{geiger2020scaling} carried out an extensive series of experiments in order to shed light on the generalisation properties of neural networks, and reported many observations and empirical arguments about the role of the variance due to the random initialisation of the weights in the double-descent curve using an ensemble of learners. This was a major motivation for the present work. Closest to our setting is the work of \citet{neal2018modern,dascoli2020,jacot2020implicit} which disentangles the various sources of variance in the process of training deep neural networks. Indeed, here we adopt the model defined by \citet{dascoli2020}, and provide a rigorous justification of their results for the case of ridge regression. A slightly finer decomposition of the variance in terms of the different sources of randomness in the problem was later proposed by \citet{NEURIPS2020_7d420e2b}. \citet{JMLR:v22:20-1211} show that such decomposition is not unique, and can be more generally understood from the point of view of the \emph{analysis of variance} (ANOVA) framework. Interestingly, subsequent papers were able to identity a series of triple (and more) descent, e.g., \citep{dAscoli_2021,adlam2020neural,chen2020multiple}.

The Random Features (RF) model was introduced in the seminal work of \citet{rahimi2007random} as an efficient approximation for kernel methods. Drawing from early ideas of \citet{10.1214/08-AOS648}, \citet{NIPS2017_0f3d014e} showed that the empirical distribution of the Gram matrix of RF is asymptotically equivalent to a linear model with matched second statistics, and characterised in this way memorisation with RF regression. The learning problem was first analysed by \citet{mei2020generalization}, who provided an exact asymptotic characterisation of the training and generalisation errors of RF regression. This analysis was later extended to generic convex losses by \citet{pmlr-v119-gerace20a} using the heuristic replica method, and later proved by \citet{dhifallah2020precise} using convex Gaussian inequalities.

The aforementioned asymptotic equivalence between the RF model and a Gaussian model with matched moments has been named the \emph{Gaussian Equivalence Principle} (GEP) \citep{PhysRevX.10.041044}. Rigorous proofs in the memorisation and learning setting with square loss appeared in \citep{NIPS2017_0f3d014e, mei2020generalization}, and for general convex penalties in \citep{goldt2021gaussian, hu2020universality}. \citet{goldt2021gaussian} and \citet{loureiro2021learning} provided extensive numerical evidence that the GEP holds for more generic feature maps, including features stemming from trained neural networks. 

Most of the previously mentioned works deriving exact asymptotics for the RF model in the proportional limit use either Random Matrix Theory techniques or Convex Gaussian inequalities. While these tools have been recently used in many different contexts, they ultimately fall short when considering an ensemble of predictors with generic convex loss and regularisation, along with structured design matrices. Therefore, to prove the results herein we employ an \emph{Approximate Message Passing} (AMP) proof technique \citep{bayati2011dynamics,donoho2016high}, leveraging on recently introduced progresses in \citep{loureiro2021learning, gerbelot2021graph} which enables to capture the full complexity of the problem and obtain the asymptotic joint distribution of the ensemble of predictors.

\section{Learning with an ensemble of random features}
\label{sec:mainres}
In this section give a first formulation of our main result, namely the exact asymptotic characterisation of the statistics of the ensembling estimator introduced in eq.~\eqref{eq:intro:estimator}. We prove that, in the proportional high dimensional limit, the statistics of the arguments of the activation function in eq.~\eqref{eq:intro:estimator} is simply given by a multivariate Gaussian, whose covariance matrix we can completely specify. This result holds for any convex loss, any convex regularisation, and for all models of generative networks $\bu_k\colon\R^d\to\R^p$, as we will show in full generality in Sec.~\ref{sec:general}. However, for simplicity, in this section and in the following we focus on the setting described in Sec.~\ref{sec:intro}, in which the statistician averages over an independent ensemble of random features, i.e., $\bu_{k}(\bx) = \phi(\bF_{k}\bx)$. In this case, our result can be formulated as follows:
\begin{theorem}[Simplified version]
\label{th:TeoSimplified} 
Assume that in the high-dimensional limit where $d, p, n\to+\infty$ with $\alpha\coloneqq \sfrac{n}{p}$ and $\gamma\coloneqq\sfrac{d}{p}$ kept $\Theta(1)$ constants, the Wishart matrix $\bF\bF^\intercal$ has a well-defined asymptotic spectral distribution. Then in this limit, for any pseudo-Lispchitz function of order 2 $\varphi\colon \R\times \R^K\to \R$, we have
\begin{equation}
\mathbb E_{(\bx,y)}\left[\varphi\left(y,\frac{\hat\bw_1^\intercal\bu_1}{\sqrt p},\dots,\frac{\hat\bw_K^\intercal\bu_K}{\sqrt p}\right)\right]\xrightarrow{\rm P}\mathbb E_{(\nu,\bmu)}\left[\varphi\left(f_0(\nu),\bmu\right)\right],
\end{equation}
where $(\nu,\bmu)\in\R^{K+1}$ is a jointly Gaussian vector  $(\nu,\bmu)\sim\mathcal{N}(\mathbf{0}_{K+1},\bSigma)$ with covariance
\begin{equation}
\label{eq:def:covariance}
\bSigma=\begin{psmallmatrix}
\rho&m\bUno_K^\intercal\\m\bUno_K&\bQ
\end{psmallmatrix},\qquad 
\bQ\coloneqq(q_0-q_1)\bI_K+q_1\bUno_{K,K},
\end{equation}
with $\bUno_{K,K}\in\R^{K\times K}$ and $\bUno_K\in\R^K$ are a matrix and a vector of ones respectively. The entries of $\bSigma$ are solutions of a set of self-consistent equations given in Corollary \ref{th:Teo1}.
\end{theorem}
As discussed in the introduction, the asymptotic statistics of the \textit{single} learner has been studied in \citep{pmlr-v119-gerace20a,dhifallah2020precise,loureiro2021learning}. Their result amounts to the analysis of the estimator solving the empirical risk minimisation problem in eq.~\eqref{eq:erm} and it is recovered imposing $K=1$ in the theorem above. For $K=1$, $(\nu,\mu)\in\R^{2}$ is jointly Gaussian with zero mean and covariance $\bSigma=\begin{psmallmatrix}
\rho&m\\m&q_0
\end{psmallmatrix}$.

However, such result is not enough to quantify the correlation between different learners, induced by the training on the same dataset, which is required to compute, e.g., the test error associated with an ensembling predictor as in eq.~\eqref{eq:intro:estimator}. For example, in the simple case where $f_{0}(u) = u$ and $\hat f(\bv)=\frac{1}{K}\sum_k v_k$, the mean-squared error on the labels is given by $\epsilon_{g} = \mathbb{E}_{(\bx,y)}[(y-\hat y(\bx))^2]=\rho+(q_0-q_1)K^{-1}+q_1-2m$, which crucially depends on the average correlation between two independent learners\footnote{Note that since all learners are here assumed to be statistically equivalent, their pair-wise correlation is the same on average. In the general case, discussed in Sec.~\ref{sec:general}, the correlation matrix $\bQ\in\R^{K\times K}$ can have a more complex structure.} 
$q_{1}\coloneqq \frac{1}{p}\mathbb{E}[\hat{\bw}_{1}^\intercal\hat{\bw}_{2}]$. Our main result is precisely an exact asymptotic characterisation of this correlation in the proportional limit of the previous theorem. Once $m$, $q_0$ and $q_1$ have been determined, the generalisation error can be computed as
\begin{equation}\label{generror}
\epsilon_g\coloneqq\mathbb E_{(\bx,y)}[\Delta\left(y,\hat y(\bx)\right)]\xrightarrow{n\to+\infty}\mathbb E_{(\nu,\bmu)}\left[\Delta\left(f_0(\nu),\hat f(\bmu)\right)\right]
\end{equation}
for any error measure $\Delta\colon\mathcal Y\times \mathcal Y\to\R^+$. 

Suppose now that 
\begin{equation}
\hat f(\bv)\equiv \hat f_0\left(\frac{1}{K}\sum_k v_k\right)    
\end{equation}
for some $\hat f_0\colon\R\to\mathcal Y$ activation function of the single learner. In this case we can introduce the random variable
$\hat\mu\stackrel{\rm d}{=}\lim_{K\to+\infty}\frac{1}{K}\sum_k\mu_k$. It is not difficult to see that the joint probability $p(\nu,\hat\mu)\sim\mathcal N(\mathbf 0_2,\hat\bSigma)$ where $
\hat\bSigma=\begin{psmallmatrix}
\rho&m\\m&q_1
\end{psmallmatrix}$. This formally coincides with the joint distribution for the activation fields for $K=1$ \citep{pmlr-v119-gerace20a}, but with $q_0$ replaced by $q_1\leq q_0$. The smaller variance is due to the fact that the fluctuations of the activation fields are averaged out by the ensembling process. The test error in the $K\to+\infty$ limit is then
\begin{equation}
\overline{\epsilon_g}\coloneqq\mathbb E_{(\nu,\hat\mu)}[\Delta(f_0(\nu),\hat f_0(\hat\mu))],
\end{equation}
so that the fluctuation contribution to the test error for $K=1$ can be defined as
\begin{equation}
\delta\epsilon_g\coloneqq \mathbb E_{(\nu,\mu)}[\Delta(f_0(\nu),\hat f_0(\mu))]-
\overline{\epsilon_g}.  
\end{equation}
The term $\delta\epsilon_g$ is by definition the contribution suppressed by ensembling and corresponds to the \textit{ambiguity} introduced by \citet{NIPS1994_b8c37e33} for the square loss. This contribution expresses the variance in the ensemble and it is responsible for the non-monotonic behaviour in the test error of interpolators, also known as the double-descent behavior.
\section{Applications}
\label{sec:applications}
We will consider now two relevant examples of separable losses, namely a ridge loss and a logistic loss. In both cases, it is possible to derive the explicit expression of the training loss and generalisation error in terms of the elements of the correlation matrix introduced above. 
\subsection{Ridge regression}
\begin{figure}
\begin{center}
\includegraphics[height=0.5\columnwidth]{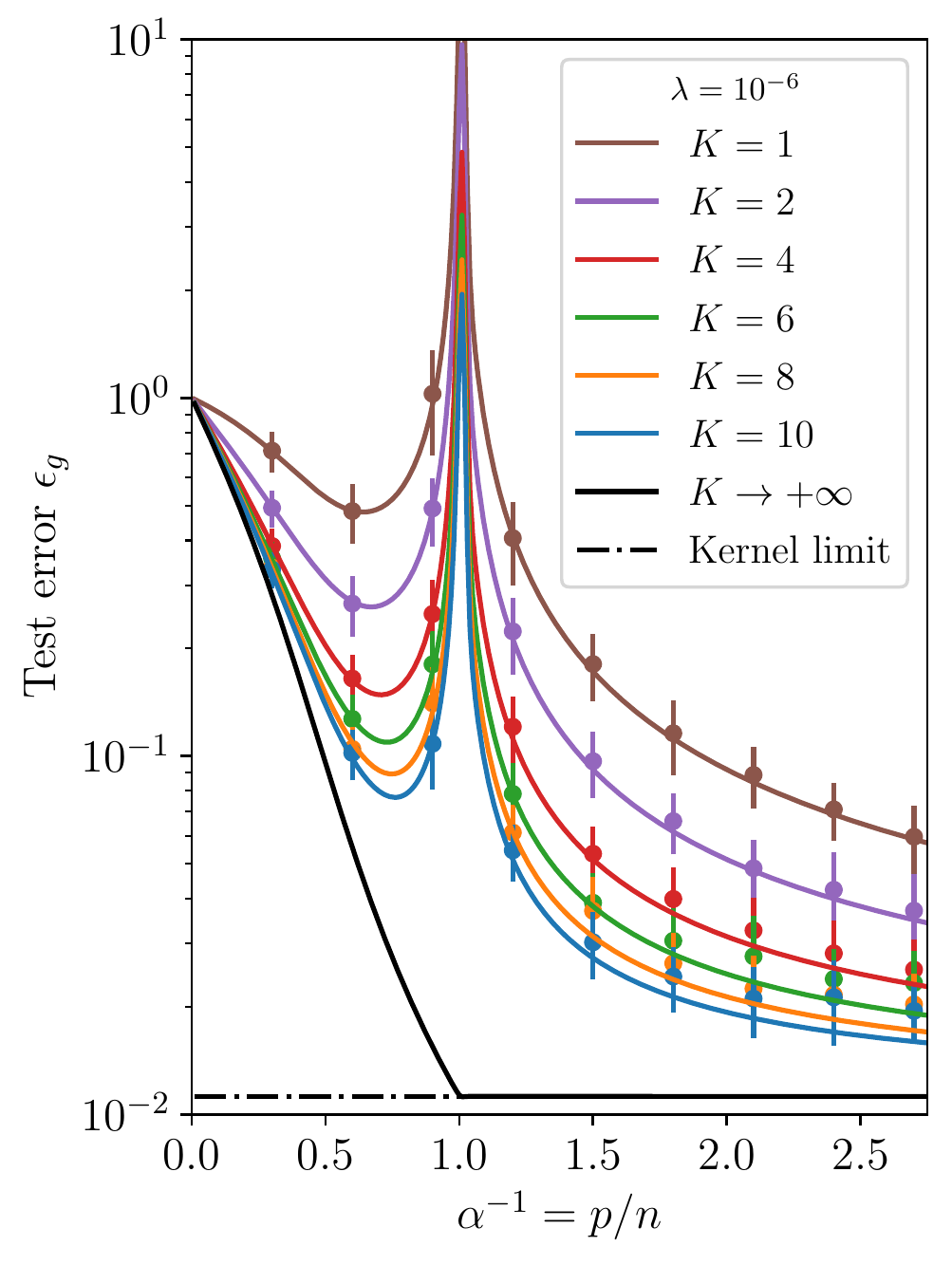}   \qquad \includegraphics[height=0.5\columnwidth]{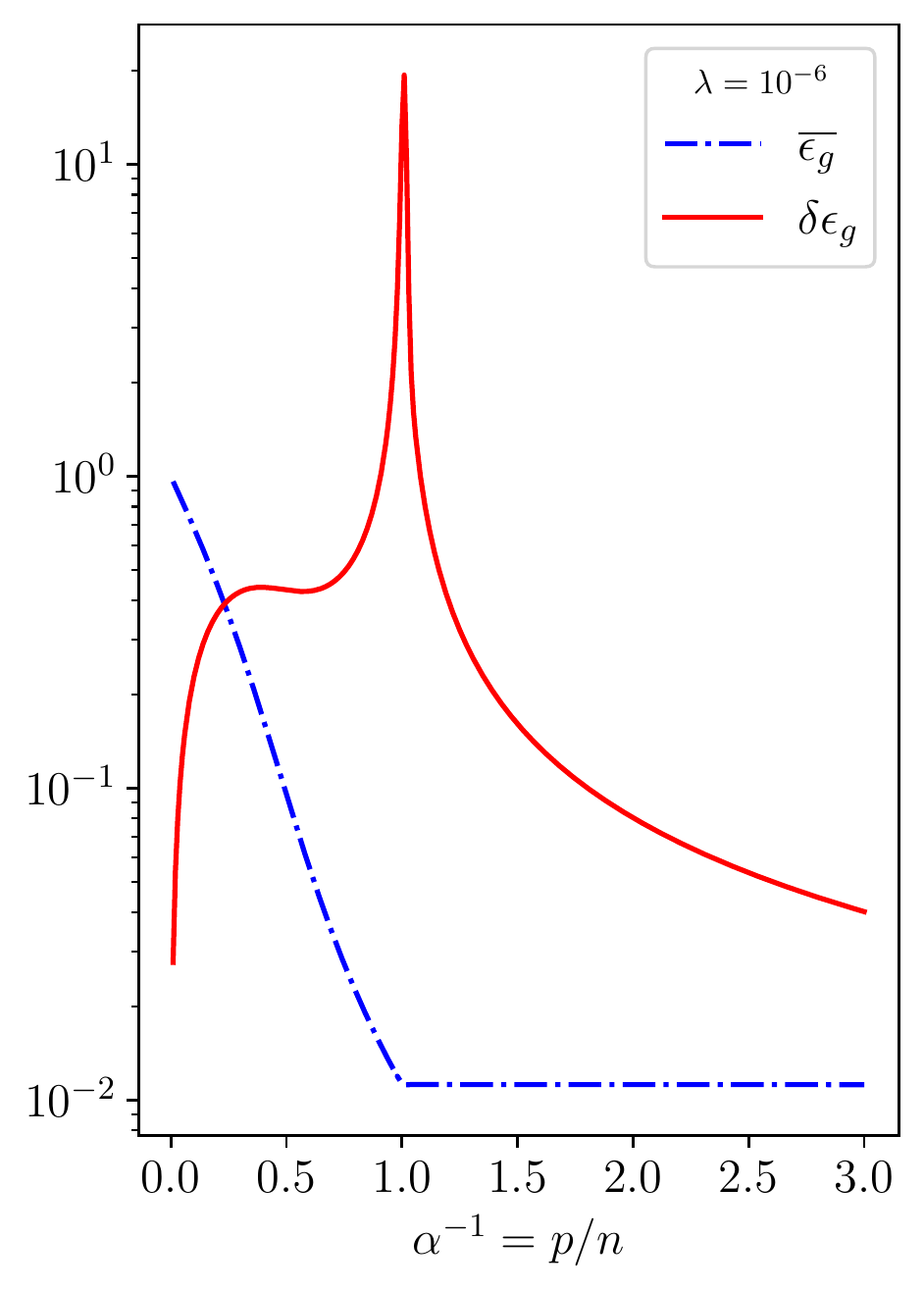}      
\end{center}
    \caption{\textit{Left.} Test error for ridge regression with $\lambda=10^{-6}$ and different values of $K$ as function of $\sfrac{p}{n}=\sfrac{1}{\alpha}$ with $\sfrac{n}{d}=2$ and $\rho=1$. Dots represent the average of the outcomes of $50$ numerical experiments in which the parameters of the neurons are estimated using $\min(d,p)=200$. Here we adopted $\phi(x)=\mathrm{erf}(x)$. \textit{Right.} Decomposition of $\epsilon_g=\overline{\epsilon_g}+\delta\epsilon_g$ in the $K=1$ case}
    \label{fig:square}
\end{figure} \label{sec:ridge}
If we assume $f_0(x)=x$, $\hat f(\bv)=\frac{1}{K}\sum_k v_k$, and a quadratic loss of the type
$\ell(y,x)=\frac{1}{2}(y-x)^2$, it is possible to write down simple recursive equations for $m$, $q_0$ and $q_1$ (see Appendix \ref{app:sec:ridge}). Taking $\Delta(y,\hat y)=(y-\hat y)^2$, the generalisation error is easily computed as
\begin{equation}
    \epsilon_g=\rho+\frac{q_0-q_1}{K}+q_1-2m\xrightarrow{K\to+\infty}\rho+q_1-2m\ \equiv\ \overline{\epsilon_g}.
\end{equation}
Note that in this case the $\lambda\to 0^{+}$ limit gives the minimum $\ell_2$-norm interpolator. In Fig.~\ref{fig:square} we compare our theoretical prediction with numerical results for $\lambda=10^{-6}$ and various values of $K$. It is evident that the divergence of the generalisation error at $\alpha=1$ is only due to the divergence of $q_0$, whereas the contribution $\overline{\epsilon_g}$, which is independent on $q_0$, is smooth everywhere. Alongside with the interpolation divergence, $\delta\epsilon_g=q_0-q_1$ has an additional bump at $\sfrac{p}{n}=\sfrac{d}{n}$, which corresponds to the ``linear peak'' discussed by \citet{dAscoli_2021}.

In the plot we present also the so-called kernel limit, corresponding to the limit $\sfrac{n}{p}=\alpha\to 0$ at fixed $\sfrac{n}{d}$. An explicit manipulation (see Appendix \ref{app:sec:ridge}) shows that $q_1=q_0\equiv q$ in this limit. This implies that in the kernel limit $\epsilon_g^{\rm k}$ does not depend on $K$, being equal to $\epsilon_g^{\rm k}\equiv \rho+q-2m$. The generalisation error obtained in the kernel limit coincides with $\overline{\epsilon_g}$ for $p>n$: this is expected as in $\overline{\epsilon_g}$ the fluctuations amongst learners are averaged out, effectively recovering the cost obtained in the case of an infinite number of parameters.

\subsection{Binary classification}
Suppose now that we are considering a classification task, such that $\mathcal Y=\{-1,1\}$. For this task we consider $f_0(x)=\mathrm{sign}(x)$. A popular choice of loss in this classification task is the logistic loss,
\begin{equation}
\ell(y,x)=\ln(1+\e^{-y x}),
\end{equation}
although other choices, e.g, hinge loss, can be considered. Since both the logistic and hinge losses depend only on the \emph{margin} $y\bw^\intercal\bu$, the empirical risk minimiser for $\lambda\to 0^{+}$ in both cases give the max-margin interpolator \citep{NIPS2003_0fe47339}. We write down the explicit saddle-point equations associated to the logistic and hinge loss in Appendix \ref{app:binary}, but we will focus our attention on the logistic case for the sake of brevity. For this choice of the loss, we obtained the values of $m$, $q_0$ and $q_1$ showed in Fig.~\ref{fig:logOT}. Using these values, a number of relevant questions can be addressed.

\begin{figure}
\begin{center}
    \includegraphics[height=0.5\columnwidth]{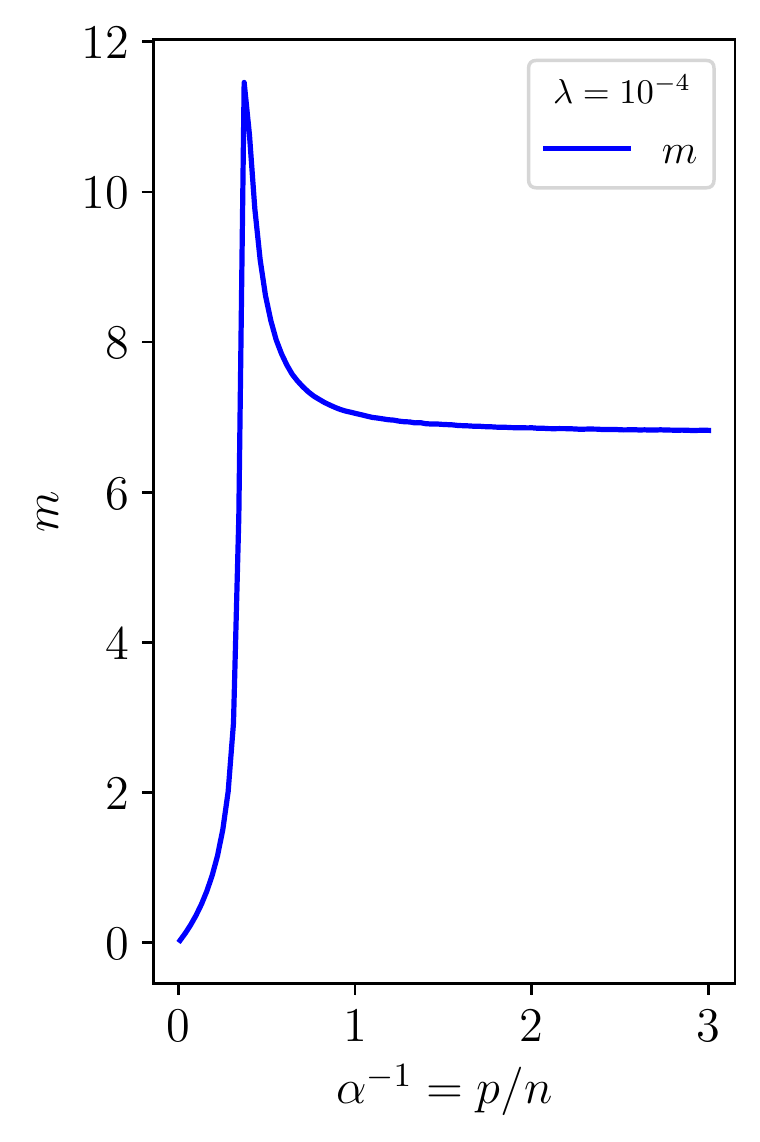}\qquad
    \includegraphics[height=0.5\columnwidth]{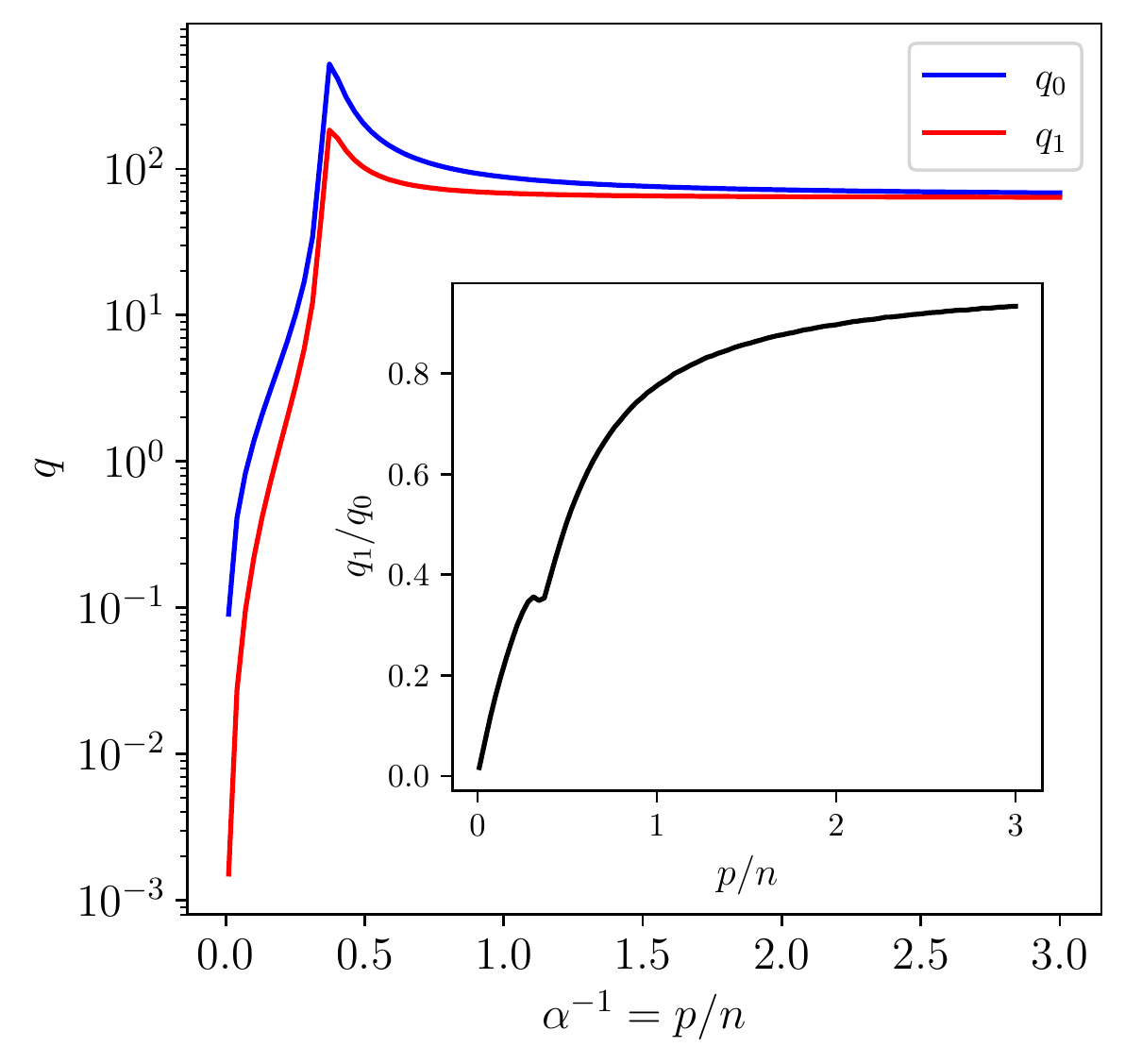}    
\end{center}
    \caption{Analytical estimation of the covariance parameters characterising the correlation with the oracle $m$ (left), the norm of the predictor in feature space $q_0$ and the correlation between learners $q_1$ (right) (see eq.~\eqref{eq:def:covariance} for the definition) in a classification task using logistic loss with ridge penalty with $\lambda=10^{-4}$ at fixed $\sfrac{n}{d}=2$ as function of $\sfrac{p}{n}$. In the inset, ratio $\sfrac{q_1}{q_0}$, quantifying the correlation between two learners. In all parameters the interpolation kink is clearly visible.}
    \label{fig:logOT}
\end{figure}

\subsection*{Alignment of learners}

Assuming that the predictor of the learner $k$ is $\hat y_k(\bx)=\sign(\hat\bw^\intercal_k\bu_k(\bx))$, in Fig.~\ref{fig:log1} (center) we estimate the probability that two learners give opposite classification. This is analytically given by
\begin{equation}
\mathbb P[\hat y_1(\bx)\neq\hat y_2(\bx)]=   \mathbb P[\mu_1\mu_2<0]=\frac{1}{\pi}\arccos\left(\frac{q_1}{q_0}\right).
\end{equation}
Note that by definition the ratio $q_{1}/q_{0}$ is a cosine similarity between two learners in the norm induced by the feature space. Therefore, this provides an interesting interpretation of these sufficient statistics in terms of the probability of disagreement. In particular, as illustrated in Fig.~\ref{fig:log1} (center) overparametrisation promotes agreement between the learners, therefore suppressing uncertainty. More generally, ensembling can be used as a technique for uncertainty estimation \citep{lakshminarayanan_deep_ensembles_2017}. In the context of logistic regression, the pre-activation to the sign function is often interpreted as a \emph{confidence score}. Indeed, introducing the logistic function $\varphi_k(\bx)=(1+\mathrm{exp}(-p^{-1/2}\hat\bw_k^\intercal\bu_k(\bx)))^{-1}$, it expresses the confidence of the $k$th classifier in associating $\hat y=1$ to the input $\bx$. Therefore, it is reasonable to ask how reliable is the logistic score as a confidence measure. For instance, what is the variance of the confidence among different learners? This can be quantified by the joint probability density $\rho(\varphi_1, \varphi_2) \coloneqq\mathbb E_\bx[\delta(\varphi_1-\varphi_1(\bx))\delta(\varphi_2-\varphi_2(\bx))]$, which can be readily computed using our Theorem \ref{th:TeoSimplified}. Fig.~\ref{fig:log1} (left) shows one example at fixed $\sfrac{p}{n}$ and vanishing $\lambda$. 

\subsection*{Ensemble predictors} In the previous two points, we discussed how ensembling can be used as a tool to quantify fluctuations. However, ensembling methods are also used in practical settings in order to mitigate fluctuations, e.g., \citep{Breiman1996}. An important question in this context is: given an ensemble of predictors $\{\hat{\vec{w}}_{k}\}_{k\in[K]}$, what is the best way of combining them to produce a point estimate? In our setting, this amounts to choosing the function $\hat{f}:\mathbb{R}^{K}\to\mathcal{Y}$. Let us consider two popular choices for the estimator $\hat f$ in eq.~\eqref{eq:intro:estimator} used in practice:
\begin{subequations}\label{estlog}
\begin{align}
\text{\bf (a)}&\qquad\hat f(\bv)=\sign\left(\sum_kv_k\right),\\
\text{\bf (b)}&\qquad\hat f(\bv)=\sign\left(\sum_k\sign(v_k)\right).
\end{align}
\end{subequations}
In a sense, \textbf{(a)} provides an estimator based on the average of the output fields, whereas \textbf{(b)}, which corresponds to a majority rule if $K$ is odd \citep{Hansen1990}, is a function of the average of the estimators of the single learners. For both choices of the estimator we use $\Delta(y,\hat y)=\delta_{\hat{y}, y}$ to measure the test error. In Fig.~\ref{fig:log1} (right) we compare the test error obtained using \textbf{(a)} and $\textbf{(b)}$ for $K=3$ with vanishing regularisation $\lambda=10^{-4}$. It is observed that the estimator \textbf{(a)} has better performances than the estimator \textbf{(b)}. As previously discussed, in this case logistic regression is equivalent to max-margin estimation, and in this case the error \textbf{(a)} can be intuitively understood in terms of a robust max-margin estimation obtained by averaging the margins associated to different draws of the random features. In the case \textbf{(a)} it is easy to show that the generalisation error takes the form
\begin{equation}
\epsilon_g=
\frac{1}{\pi}\arccos\bigg(\frac{\sqrt{K}m}{\sqrt{\rho(q_0-q_1+Kq_1)}}\bigg)\xrightarrow{K\to\infty}\frac{1}{\pi}\arccos\left(\frac{m}{\sqrt{\rho q_1}}\right)\equiv\overline{\epsilon_g}.
\end{equation}
This formula is in agreement with numerical experiments, see Fig.~\ref{fig:log2} (left). Unfortunately, we did not find a similar closed-form expression in case \textbf{(b)}. However, we can observe that in the $K\to+\infty$ limit the generalisation error in case \textbf{(a)} coincides with the generalisation error in case \textbf{(b)}, see Fig.~\ref{fig:log2} (right). By comparing with the results in Fig.~\ref{fig:log1} (center), it is evident that the benefit of ensembling in reducing the test error correlates with the tendency of learners to disagree, i.e., for small values of $\sfrac{p}{n}$, as stressed by \citet{NIPS1994_b8c37e33}. Finally, we observe a constant value of $\overline{\epsilon_g}$ beyond the interpolation threshold, compatibly with the numerical results of \citet{geiger2020scaling}.

\begin{figure}
\includegraphics[height=0.3\columnwidth]{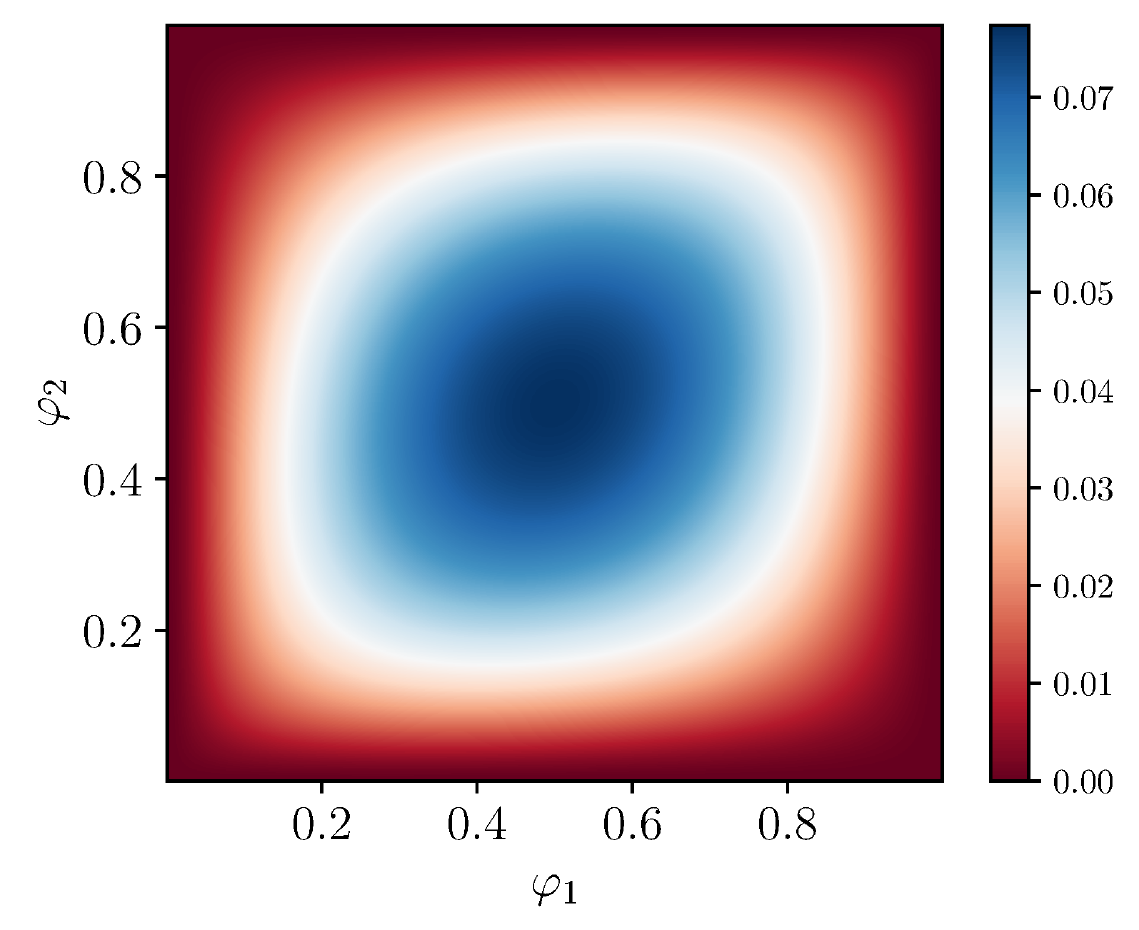}
    \includegraphics[height=0.3\columnwidth]{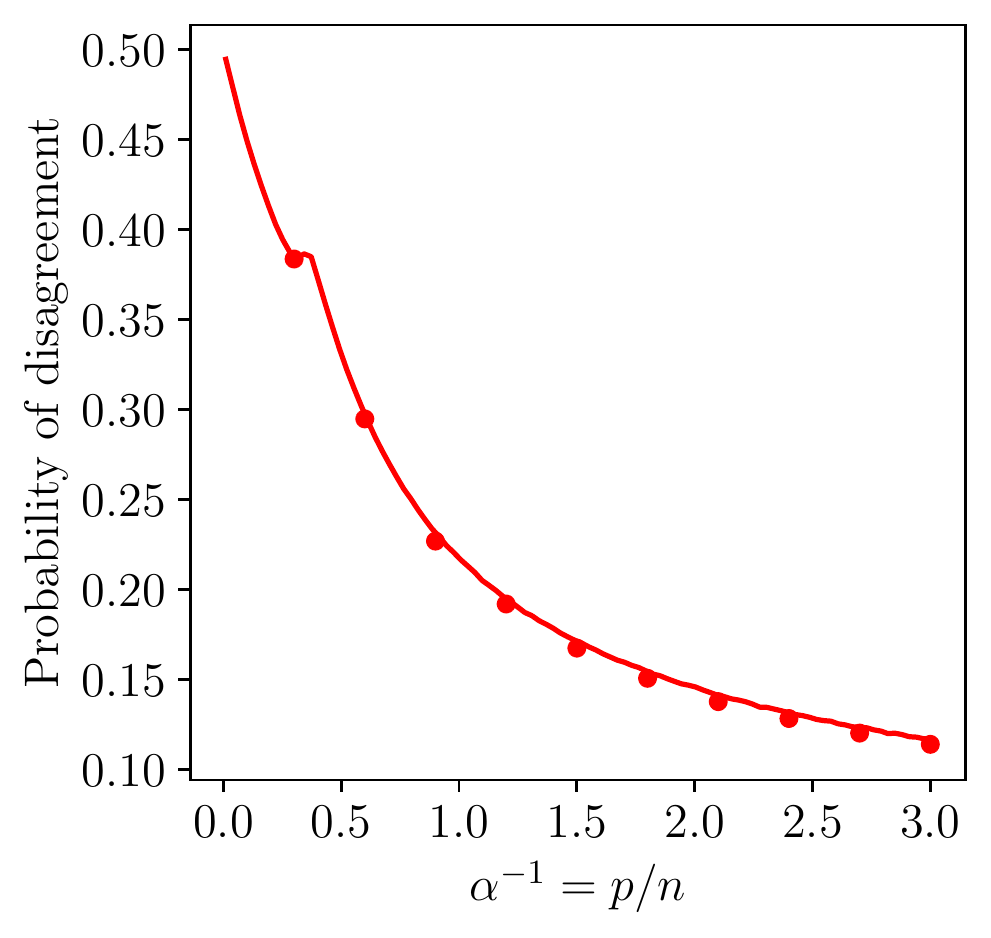}
    \includegraphics[height=0.3\columnwidth]{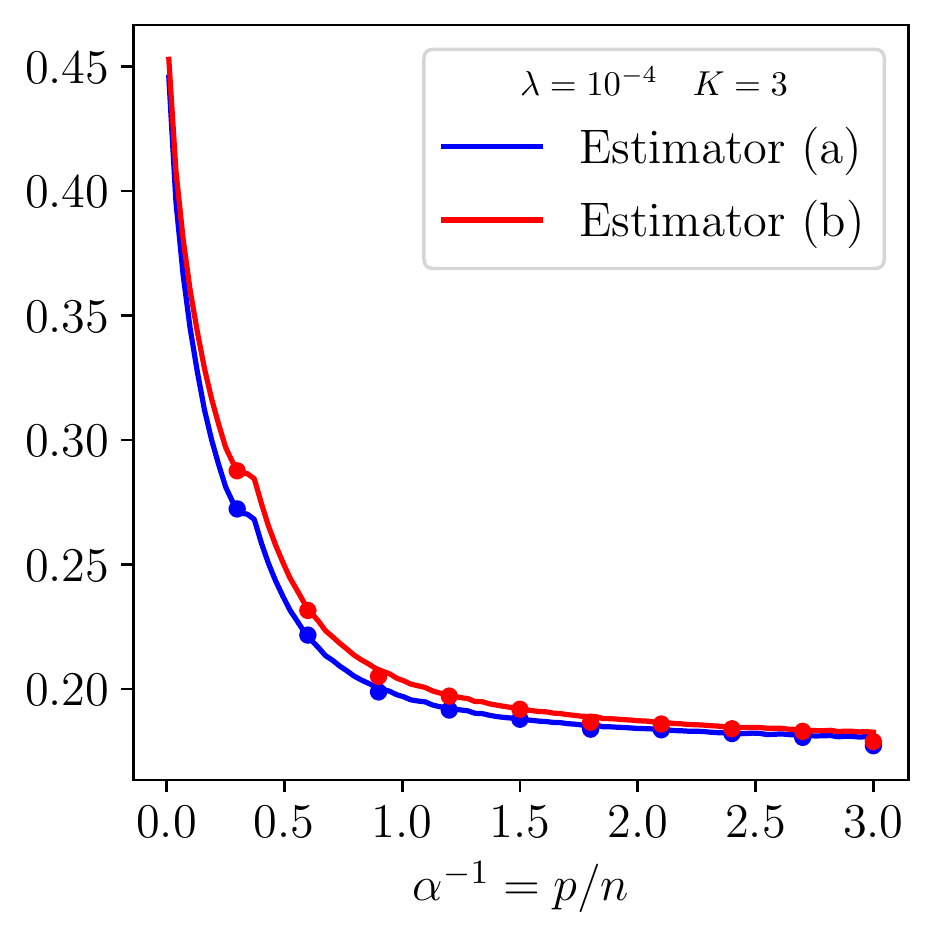}
    \caption{\textit{Left.} Joint probability density of the confidence score $\varphi_i(\bx)=(1+\mathrm{exp}(-p^{-1/2}\hat\bw_i^\intercal\bu_i(\bx)))^{-1}$ of two learners for $\sfrac{p}{n}\simeq 0.13$. \textit{Center.} Probability that two learners give discordant predictions using logistic regression as function of $\sfrac{p}{n}=\sfrac{1}{\alpha}$ with $\sfrac{n}{d}=2$, $\rho=1$, and $\lambda=10^{-4}$. \textit{Right.} Test error for logistic regression using the estimators in eq.~\eqref{estlog} and $K=3$, with the same parameters. We adopted $\phi(x)=\mathrm{erf}(x)$. We observe that the test error obtained using \textbf{(a)} is always smaller than the one obtained using \textbf{(b)}. (\textit{Center and right}) Dots represent the average of the outcomes of $10^3$ numerical experiments.}
    \label{fig:log1}
\end{figure}

\section{The case of general loss and regularisation}\label{sec:general}
In this Section we generalise our results in Sec.~\ref{sec:mainres} relaxing the hypothesis on the loss, on the regularisation and on the properties of the feature maps.  In the general setting we are going to consider, we denote $P^0_y(y|x)$ the probabilistic law by which $y$ is generated. For example, in Sec.~\ref{sec:mainres}, $P^0_y(y|x)=\delta(y-f_0(x))$. In the treatment given here, we allow for more general cases (e.g., the presence of noise in the label generation). We make no assumptions on the generative networks $\bu_k$, so that the information about the first layer is contained in the following tensors,
\begin{align}
\bsOmega&\coloneqq\mathbb E_\bx[\bU(\bx)\otimes\bU(\bx)]\in\R^{p\times p}\otimes\R^{K\times K},\\
\bhPhi&\coloneqq \mathbb E_\bx[\bU(\bx) \bx^\intercal\btheta]\in\R^{p\times K},\\
\bsTheta&=\bhPhi\otimes\bhPhi\in\R^{p\times p}\otimes\R^{K\times K}.
\end{align}
In the equations above, $\bU(\bx)\in\R^{p\times K}$ is the matrix having as concatenated columns $\bu_k(\bx)$. We aim at learning a rule as in eq.~\eqref{eq:intro:estimator}, adopting a general convex loss $\hat\ell\colon \mathcal Y\times\R^K\to \R$, so that the weights are estimated as
\begin{equation}
\label{eq:gen_student}
    \hat\bW = \underset{\bW\in\R^{p\times K}}{\arg\min}\left[\frac{1}{n}\sum_{\mu=1}^n\hat\ell\left(y^\mu,\frac{\mathrm{diag}(\bW^\intercal\bU^\mu)}{\sqrt p}\right)+\lambda r(\bW) \right]
\end{equation}
where $r\colon \R^{p\times K}\to\R$ is a convex regularisation, $\bU^\mu\equiv\bU(\bx^\mu)$ and $\hat\bW\in\R^{p\times K}$ matrix of the concatenated columns $\{\hat\bw_k\}$. Here, since the optimization problem defining the estimator may be non strictly convex, the solution may not be unique. We then denote with $\hat{\mathbf{W}}$ the unique least $\ell_{2}$ norm solution of Eq.\eqref{eq:gen_student}.

In the most general case, the statistical properties of $\hat\bW$ are captured by a finite set of finite-dimensional order parameters, namely $\bV,\bhV,\bQ,\bhQ\in\R^{K\times K}$ and $\bM,\bhM\in\R^K$. These order parameters satisfy a set of fixed-point equations. To avoid a proliferation of indices in our formulas, let us introduce some notation. Let $\bsA=(A_{kk'}^{ij})_{k,k'\in[K]}^{i,j\in[p]}\in\R^{p\times p}\otimes\R^{K\times K}$ be a tensor, and $\bX=(X_k^i)_{k\in[K]}^{i\in[p]}$, $\bY=(Y_k^i)_{k\in[K]}^{i\in[p]}$, $\bX,\bY\in\R^{p\times K}$ two matrices. We will denote
\begin{subequations}
\begin{align}
\llangle \bsA\rrangle&\coloneqq(\sum_i A_{kk'}^{ii})_{kk'}\in\R^{K\times K},\\
\llangle \bX|\bsA|\bY\rrangle&\coloneqq (\sum_{ij}X_k^iA_{kk'}^{ij}Y^j_{k'})_{kk'}\in\R^{K\times K},\\
\llangle \bX|\bY\rrangle&\coloneqq(\sum_{ij}X_k^iY^i_{k})_{k}\in\R^{K},\\
\langle \bX|\bsA|\bY\rangle&\coloneqq \sum_{ijk}X_k^iA_{kk}^{ij}Y^j_{k}\in\R\\
\langle \bX|\bY\rangle&\coloneqq\sum_{ik}X_k^iY^i_{k}\in\R.
\end{align}
Given a second tensor $\bsB\in\R^{p\times p}\otimes\R^{K\times K}$, we write
\begin{align}
\bsA\bsB&\coloneqq(\sum_{i'\kappa}A^{ii'}_{k\kappa}B_{\kappa k'}^{i'j})_{kk'}^{ij}\in\R^{p\times p}\otimes\R^{K\times K},\\
\bsA\circ \bsB&\coloneqq(\sum_{i'}A^{ii'}_{kk'}B_{k'k}^{i'j})_{kk'}^{ij}\in\R^{p\times p}\otimes\R^{K\times K},\\
\bsA\odot\bsB&\coloneqq(A_{kk'}^{ij}B_{kk'}^{ij})_{kk'}^{ij}\in \R^{p\times p}\otimes\R^{K\times K}.
\end{align}
\end{subequations}
We can now state our general result.

\begin{theorem}
\label{th:TheGen}
Let us consider the random quantities $\bxi\in\R^K$ and $\bXi\in\R^{K\times K}$ with entries distributed as $\mathcal N(0,1)$. Assume that in the high-dimensional limit where $d, p, n\to+\infty$ with $\alpha\coloneqq \sfrac{n}{p}$ and $\gamma\coloneqq\sfrac{d}{p}$ kept $\Theta(1)$ constants. Then in this limit, for any pseudo-Lispchitz functions of order 2 $\varphi\colon \R\times \R^K\to \R$ and $\tilde\varphi\colon\R^{K\times p}\to\R$, the estimator $\hat\bW$ verifies
\begin{equation}
\begin{split}
&\mathbb E_{(y,\bx)}\left[\varphi\left( y,\frac{\llangle\hat\bW|\bU\rrangle}{\sqrt p}\right)\right]\xrightarrow{\rm P}\int_{\mathcal Y} \dd y\,\mathbb E_{(\nu,\bmu)}\left[ P^0_y(y|\nu)\varphi\left(y,\bmu\right)\right],\\ &\frac{1}{n}\sum\limits_{\mathclap{\mu=1}}^n\varphi\left( y^\mu,\frac{\llangle\hat\bW|\bU^\mu\rrangle}{\sqrt p}\right)\xrightarrow{\rm P} \int_{\mathcal Y} \dd y\,\mathbb{E}_{\bxi}\left[\mathcal Z^0\left(y,\omega_0,\sigma_0\right)\varphi(y,\bh)\right],\\
&\tilde\varphi(\hat\bW)\xrightarrow{\rm P}\mathbb{E}_{\bXi}\left[\tilde\varphi(\bG)\right],
\end{split}
\end{equation}
where $\bU\equiv\bU(\bx)$, $(\nu,\bmu)\in\R^{1+K}$ are jointly Gaussian random variables with zero mean and covariance matrix
\begin{equation}
(\nu,\bmu) \sim \mathcal{N}\left(\mathbf{0}_{1+K}, \begin{psmallmatrix}
\rho&\bM^\intercal\\\bM&\bQ
\end{psmallmatrix}\right),
\end{equation}
and we have introduced the proximals for the loss and the regularisation:
\begin{equation}\begin{split}
\bh&\coloneqq\arg\min_{\bu}\left[\frac{(\bu-\bomega)\bV^{-1}(\bu-\bomega)}{2}+{\hat\ell}(y,\bu)\right],\\
\bG&\coloneqq\arg\min_\bU\left[\frac{\langle \bU|(\bUno_{p,p}\otimes\bhV)\odot\bsOmega|\bU\rangle}{2}-\langle\bfB|\bU\rangle+\lambda r(\bU)\right],
\end{split}\end{equation}
with $\bomega\coloneqq\bQ^{1/2}\bxi$ and $\bfB\coloneqq(\bUno_p\otimes\bhM^\intercal)\odot\bhPhi+((\bUno_{p,p}\otimes\bhQ)\odot\bsOmega)^{\frac{1}{2}}\bXi$. We have also introduced the auxiliary function
\begin{equation}
\mathcal Z^0(y,\mu,\sigma)\coloneqq\int\frac{P^0_y(y|x)\dd x}{\sqrt{2\pi\sigma}} \e^{-\frac{(x-\mu)^2}{2\sigma}}.
\end{equation}
and the scalar quantities $\omega_0\coloneqq\bM^\intercal\bQ^{-1/2}\bxi$ and $\sigma_0\coloneqq\rho-\bM^\intercal\bQ^{-1}\bM$. The order parameters satisfy the saddle-point equations
\begin{equation} \label{sp1gen}
\begin{split}
\bhV&=-\alpha\int_{\mathcal Y}\dd y\,\mathbb E_{\bxi}\left[\mathcal Z^0(y,\omega_0,\sigma_0)\,\partial_\bomega\bff\right],\\
\bhQ&=\alpha\int_{\mathcal Y}\dd y\,\mathbb E_{\bxi}\left[\mathcal Z^0(y,\omega_0,\sigma_0)\,\bff\bff^\intercal\right],\\
\bhM&=\frac{\alpha}{\sqrt \gamma}\int_{\mathcal Y}\dd y\,\mathbb E_{\bxi}\left[\partial_{\mu}\mathcal Z^0(y,\omega_0,\sigma_0)\bff \right],
\end{split}
\end{equation}
and
\begin{equation} \label{sp2gen}
\begin{split}
\bV&=\frac{1}{p}\mathbb E_\bXi\llangle\bsOmega(\bG\otimes\bXi)\left((\bUno_{p,p}\otimes\bhQ)\odot\bsOmega\right)^{-1/2}\rrangle^\intercal,\\
\bQ&=\frac{1}{p}\mathbb E_\bXi\llangle \bG|\bsOmega|\bG\rrangle,\\
\bM&=\frac{1}{\sqrt\gamma p}\mathbb E_\bXi\llangle\bhPhi|\bG\rrangle.
\end{split}
\end{equation}
In the equation above we have introduced the short-hand notation $\bff\coloneqq\bV^{-1}(\bh-\bomega)$.
\end{theorem}
Eqs.~\eqref{sp1gen} are typically called \textit{channel equations}, because depend on the form of the loss $\hat\ell$. Eqs.~\eqref{sp2gen}, instead, are usually called \textit{prior equations}, because of their dependence on the prior, i.e., $r$. In the following Corollary, we specify their expression for a ridge regularisation, $r(\bW)=\frac{1}{2}\|\bW\|_{\rm F}^2$.
\begin{corollary}[Ridge regularisation] In the hypotheses of Theorem~\ref{th:TheGen}, if $r(\bW)=\frac{1}{2}\|\bW\|_{\rm F}^2$, then the prior equations are
\begin{equation}
\begin{split}
\bV&=\frac{1}{p}\llangle\bsOmega\circ\bsA\rrangle,\\
\bQ&=\frac{1}{p}\llangle\bsOmega\circ\left(\bsA\left((\bUno_{p,p}\otimes\bhM\otimes\bhM^\intercal)\odot\bsTheta+(\bUno_{p,p}\otimes\bhQ)\odot\bsOmega\right)\bsA\right)\rrangle,\\
\bM&=\frac{1}{\sqrt\gamma p}\llangle\bsA\left((\bUno_{p,p}\otimes\bhM\otimes\bUno_K^\intercal)\odot\bsTheta\right)\rrangle.
\end{split}
\end{equation}
In the equation above, we have used the auxiliary tensor $
\bsA\equiv \bsA(\bhV;\lambda,\bsOmega)\coloneqq(\lambda \bI_p\otimes \bI_K+(\bUno_{p,p}\otimes\bhV)\odot\bsOmega)^{-1}\in\R^{p\times p}\otimes\R^{K\times K}$.
\end{corollary}
\subsection{The random feature case and the kernel limit}
Theorem \ref{th:TheGen} is given in a very general setting, and, in particular, no assumptions are made on the features $\bu_k$. We have anticipated in Sec.~\ref{sec:mainres} that, in the case of random features, the structure of the order parameters highly simplifies and the covariance matrix $\bSigma$ is fully specified by only three scalar order parameters for any $K>1$. Here will adapt therefore Theorem \ref{th:TheGen} to the random feature setting in Sec.~\ref{sec:mainres}, using the notation therein. The motivation of this section is to explicitly present the self-consistent equations that are required to produce the results given in the paper.
\begin{corollary}
\label{th:Teo1} 
Assume that in the high-dimensional limit where $d, p, n\to+\infty$ with $\alpha\coloneqq \sfrac{n}{p}$ and $\gamma\coloneqq\sfrac{d}{p}$ kept $\Theta(1)$ constants, the Wishart matrix $\bF\bF^\intercal$ has a well-defined asymptotic spectral distribution. Then in this limit, for any pseudo-Lispchitz function of finite order $\varphi\colon \R\times \R^K\to \R$, the estimator $\hat\bW$ verifies
\begin{equation}
\mathbb E_{(\bx,y)}\left[\varphi\left(y,\frac{\llangle\hat\bW| \bU\rrangle}{\sqrt p}\right)\right]\xrightarrow{\rm P}\mathbb E_{(\nu,\bmu)}\left[\varphi\left(f_0(\nu),\bmu\right)\right],
\end{equation}
where $(\nu,\bmu)\in\R^{K+1}$ is a jointly Gaussian vector with covariance
\begin{equation}
(\nu,\bmu)\sim\mathcal{N}\left(\mathbf{0}_{K+1}, \begin{psmallmatrix}
\rho&m\bUno_K^\intercal\\m\bUno_K&\bQ
\end{psmallmatrix}\right),
\end{equation}
and $\bQ\coloneqq(q_0-q_1)\bI_K+q_1\bUno_{K,K}$. The collection of parameters $(q_0,q_1,m)$ is obtained solving a set of fixed point equations involving the auxiliary variables $(\hat q_0,\hat q_1,\hat m, v,\hat v)$, namely:
\begin{subequations}\label{eq:fp-K1}
\begin{align}
\hat v&
=-\alpha\int_{\mathcal Y}\dd y\,\mathbb E_\omega\left[\mathcal Z^0\!\left(y,\frac{m\omega}{q_0},\rho-\frac{m^2}{q_0}\right)\partial_{\omega}f\right],\\
\hat m&
=\frac{\alpha}{\sqrt \gamma}\int_{\mathcal Y}\dd y\,\mathbb E_\omega\left[\partial_{\mu}\mathcal Z^0\!\left(y,\frac{m\omega}{q_0},\rho-\frac{m^2}{q_0}\right) f \right],\\
\hat q_0&
=\alpha\int_{\mathcal Y}\dd y\,\mathbb E_\omega\left[\mathcal Z^0\!\left(y,\frac{m\omega}{q_0},\rho-\frac{m^2}{q_0}\right)f^2\right],\\
\hat q_1&
=\alpha\int_{{\mathcal Y}}\dd y\mathbb E_{\omega,\omega'}\left[\mathcal Z^0\left(y,m\frac{\omega+\omega'}{q_0+q_1},\rho-\frac{2m^2}{q_0+q_1}\right) ff'\right],\\
v&
=\int\frac{s\varrho(s)\dd s}{\lambda+s\hat v},\\
m&
=\frac{\hat m}{\sqrt{\gamma}}\int \frac{s-\kappa_*^2}{\lambda+\hat vs}\varrho(s)\dd s,\\
q_0&
=\int\frac{(\hat q_0+\hat m^2)s^2-\hat m^2\kappa_*^2s}{\left(\lambda+
\hat v s\right)^2}\varrho(s)\dd s,\\
q_1&
=\left(1+\frac{\hat q_1}{\hat m^2}\right)m^2.
\end{align}
\end{subequations}
where $\omega$ and $\omega'$ are two correlated Gaussian random variables of zero mean and $\mathbb E[\omega^2]=\mathbb E[\omega'^{2}]=q_0$, $\mathbb E[\omega\omega']=q_1$. Moreover, we have introduced the proximals
\begin{equation}
f=\frac{\mathrm{Prox}_{v\ell(y,\bullet)}\left(\omega\right)-\omega}{v},\quad f'=\frac{\mathrm{Prox}_{v\ell(y,\bullet)}\left(\omega'\right)-\omega'}{v},\end{equation}
with
\begin{equation}
\mathrm{Prox}_{v\ell(y,\bullet)}(\omega)\coloneqq\arg\min_x\left[\frac{(x-\omega)^2}{2v}+\ell(y,x)\right].  
\end{equation}
 Finally, $\varrho(s)$ is the asymptotic spectral density of the features covariance matrix $\bOmega \equiv \mathrm{Var}(\bu) =\kappa_{0}^{2}\bUno_{p,p}+\frac{\kappa_1^2}{d}\bF\bF^\intercal+\kappa_*^2\bI_p$ and the coefficients are given by $\kappa_0\coloneqq\mathbb E_\zeta[\phi(\zeta)]$, $\kappa_1\coloneqq\mathbb E_\zeta[\zeta\phi(\zeta)]$, $\kappa_*\coloneqq\mathbb E_\zeta[\phi^2(\zeta)]-\kappa^2_0-\kappa^2_1$ with $\zeta\sim\mathcal N(0,1)$. 
\end{corollary}
The previous corollary recovers the results of \citet{pmlr-v119-gerace20a}, \citet{dhifallah2020precise}, and \citet{loureiro2021learning} when restricted to the $K=1$ case by marginalisation.

\subsection*{The kernel limit}
The so-called kernel limit is obtained by taking the limit of infinite number of parameters so that $\gamma\to 0$ (i.e., $p\gg d$ and $p\gg n$), but with a fixed ratio $\sfrac{\alpha}{\gamma} = \sfrac{n}{d}$. To balance the loss term and the regularisation it is convenient to rescale $\lambda\mapsto \alpha\lambda$. We can simplify the saddle-point equation in this special limit introducing $\hat q_0\mapsto\alpha {\hat q_0}$, $\hat q_1\mapsto \alpha {\hat q_1}$, $\hat m\mapsto\sqrt \alpha {\hat m}$, $\hat v\mapsto\alpha {\hat v}$. The channel equations keep a simple form,
\begin{subequations}
\begin{align}\label{eq:separableKernel1}
\hat v&=-\int_{\mathcal Y}\dd y\,\mathbb E_{\zeta}\left[\mathcal Z^0\!\left(y,\omega_0,\sigma_0\right)\partial_{\omega}f\right],\\
\hat m&=\sqrt{\delta}\int_{\mathcal Y}\dd y\,\mathbb E_{\zeta}\left[ f\partial_{\mu}\mathcal Z^0\!\left(y,\omega_0,\sigma_0\right) \right],\\
\hat q_0&=\int_{\mathcal Y}\dd y\,\mathbb E_{\zeta}\left[\mathcal Z^0\!\left(y,\omega_0,\sigma_0\right)f^2\right],\\
\hat q_1&=\int_{\mathcal Y}\dd y\,\mathbb E_{\omega,\omega'}\left[\mathcal Z^0\left(y,m\frac{\omega+\omega'}{q_0+q_1},\rho-\frac{2m^2}{q_0+q_1}\right) f f'\right].
\end{align}
\end{subequations}
The $p\to+\infty$ limit in the prior equations depends on the spectral density $\varrho(s)$. For example, if $\bF$ has random Gaussian entries with zero mean and unit variance, then $\varrho(s)$ is a shifted Marchenko--Pastur distribution,
\begin{equation}
\varrho(s)=\nu(s-\kappa_*^2;\alpha^{-1}\delta,\kappa_1),\quad\end{equation}
where, if $[x]_+=x\theta(x)$,
\begin{equation}\nu(x;b,a)=\frac{\sqrt{\left[\left(a_+-x\right)\left(x-a_-\right)\right]_+}}{2ab^2\pi^2x}+\left[1-\frac{1}{a}\right]_+\delta(x),
\end{equation}
with $a_\pm\coloneqq b^2(1\pm\sqrt{a})^2$. By means of a series of algebraic manipulation, we obtain in the end at the first order in $\alpha$
\begin{equation}
\begin{split}
    v &= \frac{   \lambda ( \kappa_1^2 +  \kappa_*^2) + \delta^2 \kappa_1^2\kappa_*^2 {{\hat {v}}}
    }{\lambda\left(\lambda + \delta \kappa_1^2 {{\hat {v}}}\right) },\\
    m &= \frac{\sqrt{\delta} \kappa_1^2\hat m}{\lambda + \delta \kappa_1^2 {{\hat {v}}} },
\end{split}\quad
\begin{split}
    q_0 &= \frac{\delta \kappa_1^4\left( {\hat{q}}_0 + \delta {{\hat m}}^2  \right)}{\left(\lambda + \delta \kappa_1^2 {{\hat {v}}}\right)^2  },\\
    q_1 &= \frac{\delta \kappa_1^4\left( {\hat{q}}_1 + \delta {{\hat m}}^2  \right)}{\left(\lambda + \delta \kappa_1^2 {{\hat {v}}}  \right)^2},
\end{split}
\end{equation}
which complete our set of equations for the kernel limit.

\section*{Acknowledgements}
We thank Ali Bereyhi, Giulio Biroli, St\'ephane d'Ascoli, Justin Ko for discussions, and Francesca Mignacco for sharing code with us. We acknowledge funding from the French National Research Agency grants ANR-17-CE23-0023-01 PAIL and ANR-19P3IA-0001 PRAIRIE.

\bibliography{refs}

\newpage
\appendix
\onecolumn
\section{The replica approach}
\subsection{Notation}
We introduce here some notation that will help us to keep the expressions in this appendix more compact. Given two tensors $\bsA=(A^{ij}_{kk'})^{ij}_{kk'}\in\R^{p\times p}\otimes\R^{K\times K}$ and $\bsB=(B^{ij}_{kk'})^{ij}_{kk'}\in\R^{p\times p}\otimes\R^{K\times K}$, $i,j\in[p]$, $k,k'\in[K]$, then
\begin{align}
\hat\bsC=\bsA\bsB&\Leftrightarrow \hat C^{ij}_{kk'}\coloneqq \sum_{r,\kappa}A^{ir}_{k\kappa}B^{rj}_{\kappa k'}\in\R^{p\times p}\otimes\R^{K\times K}\\
\bsC=\bsA\odot\bsB&\Leftrightarrow C^{ij}_{kk'}\coloneqq A^{ij}_{kk'}B^{ij}_{kk'}\in\R^{p\times p}\otimes\R^{K\times K}\\
\tilde\bsC=\bsA\circ\bsB&\Leftrightarrow \tilde C^{ij}_{kk'}\coloneqq \sum_{r}A^{ir}_{kk'}B^{rj}_{k'k}\in\R^{p\times p}\otimes\R^{K\times K}.
\end{align}
Also, if $\bX=(X^i_k)^{i}_k\in\R^{p\times K}$ and $\bY=(Y^i_k)^{i}_k\in\R^{p\times K}$, we write
\begin{align}
\bX\otimes\bY&=(X^{i}_kY^{j}_{k'})^{ij}_{kk'}\in\R^{p\times p}\otimes \R^{K\times K}\\
\bsA|\bX\rangle&=\left(\textstyle\sum_{j\kappa}A^{ij}_{k\kappa}X^{j}_{\kappa}\right)_k^i\in\R^{p\times K}\\
\langle\bX|\bY\rangle&=\textstyle\sum_{ik}X^i_kY^i_k\in\R\\
\langle \bX|\bsA| \bY\rangle&=\textstyle\sum_{ijkk'}X^{i}_kA^{ij}_{kk'}Y^j_{k'}\in\R,\\
\llangle \bX| \bY\rrangle&=\left(\textstyle\sum_{ij}X^{i}_kY^j_{k}\right)_k\in\R^K\\
\llangle \bX|\bsA| \bY\rrangle&=\left(\textstyle\sum_{ij}X_k^iA^{ij}_{kk'}Y_{k'}^j\right)_{kk'}\in\R^{K\times K},\\
\llangle \bsA\rrangle&=\left(\textstyle\sum_iA^{ii}_{kk'}\right)_{kk'}\in\R^{K\times K}.
\end{align}
In other words, the double brakets $\llangle\bullet\rrangle$ express the contraction of the upper indices only. This means for example that $\llangle \bX| \bY\rrangle=\diag(\bX^\intercal\bY)\in\R^K$. Finally, if $\bu,\bv\in\R^K$ and $\bA\in\R^{K\times K}$, $\langle \bu|\bA|\bv\rangle\coloneqq\bu^\intercal\bA\bv=\sum_{kk'}u_kA_{kk'}v_{k'}\in\R$. We will adopt the same notation in the simple $K=1$ case.
\subsection{The replica computation}
\label{app:replicas}
The replica computation relies on the treatment of a Gibbs measure over the weights $\bW$ which concentrates on the weights $\hat\bW$ that minimize a certain loss ${\hat\ell}$ when a fictitious `inverse temperature' parameter is sent to infinity. Such measure reads
\begin{align}\comprimi
 \mu_\beta(\bW)&\coloneqq\frac{P_\bW(\bW)}{\mathcal Z(\beta)}\prod_{\mu=1}^n\exp\left[-\beta{\hat\ell}\left(y^\mu,\frac{\llangle\bW|\bU^\mu\rrangle}{\sqrt p}\right)\right],\\ \mathcal Z(\beta)&\coloneqq \int\dd\bW\,P_\bW(\bW)\prod_{\mu=1}^n\exp\left[-\beta{\hat\ell}\left(y^\mu,\frac{\llangle\bW|\bU^\mu\rrangle}{\sqrt p}\right)\right],
\end{align}
where $P_\bW(\bW)=\e^{-\beta\lambda r(\bW)}$ is the prior on the weights $\bW=(W^i_k)^{i\in[p]}_{k\in[K]}\in\R^{p\times K}$, possibly containing the regularisation. The dataset $(y^\mu,\bU^\mu)_\mu$ is obtained from a set of $n$ samples $\bx^\mu\sim\mathcal N(\mathbf 0_d,\bI_d)$, $\mu\in[n]$. For each $\mu$, the label $y^\mu$ has distribution $P^0_y\left(y|d^{-1/2}\langle\btheta\big|\bx^\mu\rangle\right)$ for some fixed $\btheta\sim\mathcal N(\mathbf 0_d,\rho\bI_d)$. The array of features $\bU^\mu$, instead, is obtained as function of the vector $\bx^\mu$ via a law $\bU\colon\R^d\to\R^{p\times K}$ such that $\bU^\mu\coloneqq\bU(\bx^\mu)\in\R^{p\times K}$. As we will show below, the tensors
\begin{align}
\bsOmega&\coloneqq\mathbb E_\bx[\bU(\bx)\otimes\bU(\bx)]\in\R^{p\times p}\otimes\R^{K\times K},\\
\bhPhi&\coloneqq \mathbb E_\bx[\bU(\bx)\langle \bx|\btheta\rangle]\in\R^{p\times K},
\end{align}
will incorporate the information about the action of the law $\bU$. We denote for brevity
\begin{equation}
P_y(y|\bu)\propto  \exp\left[-\beta{\hat\ell}(y,\bu)\right],
\end{equation}
and we proceed computing the \textit{free entropy} $\Phi\coloneqq\mathbb E_{(y^\mu,\bx^\mu)_\mu}[\ln \mathcal Z(\beta)]$ using the replica trick, i.e., the fact that $\mathbb E[\ln \mathcal Z(\beta)]=\lim_{s\to 0}\frac{1}{s}\ln\mathbb E[\mathcal Z^s(\beta)]$
\begin{equation}
\mathbb E_{(y^\mu,\bx^\mu)_\mu}[\mathcal Z^s(\beta)]=\prod_{a=1}^s\int\dd\bW^aP_\bW(\bW^a) \left(\mathbb E_{(y,\bx)}\left[ P^0_y\left(y\Big|\frac{\langle\bx|\btheta\rangle}{\sqrt d}\right)\prod_{a=1}^sP_{y}\left(y\Big|\frac{\llangle \bW^a|\bU(\bx)\rrangle}{\sqrt p}\right)\right]\right)^n.
\end{equation}
Denoting by $\bmu^a\equiv(\mu^a_k)_{k\in[K]}$, if we now consider
\begin{multline}
\mathbb E_{\bx}\left[P^0_y\left(y\Big|\frac{\langle\bx|\btheta\rangle}{\sqrt d}\right)\prod_{a=1}^sP_{y}\left(y\Big|\frac{\llangle \bW^a|\bU(\bx)\rrangle}{\sqrt d}\right)\right]\\
=\int \dd\nu P^0_y(y|\nu)\prod_{a=1}^s\left[\int\dd\bmu^a P_{y}(y|\bmu^a)\right]\underbrace{\mathbb E_\bx\left[\delta\left(\nu-\frac{\langle\bx|\btheta\rangle}{\sqrt d}\right)\prod_{a=1}^s\delta\left(\bmu^{a}-\frac{\llangle \bW^a|\bU(\bx)\rrangle}{\sqrt p}\right)\right]}_{P(\nu,\bmu)}.
\end{multline}
We apply now the \textit{Gaussian Equivalence Principle} \citep{goldt2021gaussian}, i.e., we assume that $P(\nu,\bmu)$ is a Gaussian with covariance matrix given by
\begin{equation}\bSigma(\bW)=
 \begin{pmatrix}
 \rho&{\bsM}^\intercal\\\bsM&\bsQ 
 \end{pmatrix},
\end{equation} 
where $\bsM=(\bM^a)_{a\in[s]}\in\R^{sK}$ and $\bsQ=(\bQ^{ab})_{a,b\in[s]}\in\R^{sK\times sK}$. Here, for each $a,b\in[n]$, $\bM^a\in\R^K$ and $\bQ^{ab}\in\R^{K\times K}$ and are defined as
\begin{align}
\rho&\coloneqq\mathbb E[\nu^2]=\frac{\|\btheta\|^2_2}{d},\\
\bM^a&\coloneqq\mathbb E[\bmu^{a}\nu]=\frac{\llangle\bW^a|\bhPhi\rrangle}{\sqrt{pd}},\\
\bQ^{ab}&\coloneqq\mathbb E[\bmu^{a}\bmu^{b}{}^\intercal]=\frac{\llangle\bW^a|\bsOmega|\bW^b\rrangle}{p}.
\end{align}
In the end
\begin{multline}\comprimi
\mathbb E_{(y^\mu,\bx^\mu)_\mu}[\mathcal Z^s(\beta)]=\mathbb E_{\btheta}\left[\prod_{a=1}^s\int\dd\bW^aP_\bW(\bW^a) \left(\int\dd y\int \dd\nu P^0_y(y|\nu)\prod_{a=1}^s\left[\int \dd\bmu^{a} P_{y}(y|\bmu^{a})\right]P(\nu,\bmu)\right)^n\right]\\
=\left(\prod_{a=1}^s\iint\frac{\dd \bsM^a\dd\hat \bsM^a}{(2\pi)^K}\right)\left(\prod_{a=1}^s\iint\frac{\dd \bsQ^{ab}\dd\hat\bsQ^{ab}}{(2\pi)^{K^2}}\right)\int\frac{\dd\rho\dd\hat\rho}{2\pi}\exp\left(p\Phi^{(s)}(\rho,\bsM,\bsQ,\hat\rho,\hat\bsM,\hat\bsQ)\right).
\end{multline}
Absorbing the factor $-i$ in the integrals, and denoting by $\sfrac{n}{p}=\alpha$ and $\sfrac{d}{p}=\gamma$,
\begin{equation}
\Phi^{(s)}(\rho,\bsM,\bsQ,\hat\rho,\hat\bsM,\hat\bsQ)=-\gamma\rho\hat\rho-\sqrt{\gamma}\sum_{a=1}^s\langle\bhM^a{}|\bM^a\rangle-\sum_{a\leq b}\langle\bhQ^{ab}|\bQ^{ab}\rangle+\alpha\Psi^{(s)}_y(\rho,\bsM,\bsQ)+\Psi_\bw^{(s)}(\hat\rho,\hat\bsM,\hat\bsQ).
\end{equation}
Here we have introduced
\begin{equation}\comprimi\begin{split}
\Psi^{(s)}_y(\rho,\bsM,\bsQ)&\coloneqq\ln\left[\int\dd y\int \dd\nu P^0_y(y|\nu)\prod_{a=1}^s\left[\int \dd\bmu^{a} P_{y}(y|\bmu^{a})\right]P(\nu,\bmu)\right]\\
\Psi^{(s)}_\bw(\hat\rho,\hat\bsM,\hat\bsQ)&\coloneqq\frac{1}{p}\ln\left[\e^{\hat\rho\|\btheta\|^2_2} \left(\prod_{a=1}^s\int P(\bW^a)\e^{\langle (\bUno_p\otimes\bhM^a{}^\intercal)\odot\bhPhi) |\bW^a\rangle}\right)\exp\left(\sum_{a\leq b} \langle \bW^a|(\bUno_{p,p}\otimes\bhQ^{ab})\odot\bsOmega|\bW^b\rangle\right)\right]
\end{split}\end{equation}
so that in the high dimensional limit the desired average is the extremum of the functional ${1}{s}\Phi^{(s)}$ in the $s\to 0$ limit,
\begin{equation}
\mathbb E_{(y^\mu,\bx^\mu)_\mu}[\ln \mathcal Z(\beta)]=\lim_{s\to 0}\frac{1}{s}\mathrm{Ext}\,\Phi^{(s)}(\rho,\bsM,\bsQ,\hat\rho,\hat\bsM,\hat\bsQ).
\end{equation}

\paragraph{Replica symmetric ansatz --} In order to take the limit, let us assume as usual a \textit{replica symmetric} (RS) ansatz:
\begin{equation}
\begin{split}
\bM^a&\equiv \bM\quad a\in[s],\\
\bQ^{ab}&\equiv\begin{cases}\bR &\text{if }a=b,\\\bQ&\text{if }a\neq b,\end{cases}
\end{split}\qquad
\begin{split}
\bhM^a&\equiv \bhM\quad a\in[s],\\
\bhQ^{ab}&\equiv\begin{cases}-\frac{1}{2}\bhR &\text{if }a=b,\\\bhQ&\text{if }a\neq b.\end{cases}
\end{split}
\end{equation}
Observe that $\lim_{s\to 0}\Phi^{(s)}=0$ by construction, meaning that $\hat\rho=0$ fixing $\rho=\frac{1}{d}\mathbb E_\btheta[\|\btheta\|_2^2]$. Before proceeding further, we note that the matrix $\bsQ$ in the RS ansatz can be written as $\bsQ=\bI_s\otimes(\bR-\bQ)+\bUno_{s,s}\otimes \bQ$, where $\bUno_{s,s}$ is the $s\times s$ matrix of $1$. Similarly, $\bsM=\bUno_s\otimes \bM$, where $\bUno_s$ is the column vector of $s$ elements equal to $1$. 
Following similar steps to the ones detailed, e.g., in \citep{loureiro2021learning}, we obtain
\begin{multline}
\Psi_y(\rho,\bM,\bQ,\bR)\coloneqq \lim_{s\to 0}\frac{1}{s}\Psi^{(s)}_y(\rho,\bsM,\bsQ)\\
=\int_{\mathcal Y}\dd y\,\mathbb E_{\bxi}\left[\mathcal Z^0\!\left(y,\langle\bM|\bQ^{-1/2}|\bxi\rangle,\rho-\langle\bM|\bQ^{-1}|\bM\rangle\right)\ln\mathcal Z\left(y,\sqrt\bQ\bxi,\bV\right)\right],
\end{multline}
where $\bV\coloneqq\bR-\bQ$, $\bxi\sim\mathcal N(\mathbf 0_K,\bI_K)$ and we have introduced
\begin{align}
\mathcal Z(y,\bmu,\bSigma)&\coloneqq\int\frac{P_{y}(y|\bmu)\dd\bmu}{\sqrt{\det(2\pi\bSigma)}} \e^{-\frac{1}{2}\langle\bu-\bmu|\bSigma^{-1}|\bu-\bmu\rangle},\\
\mathcal Z^0(y,\mu,\sigma)&\coloneqq \int\frac{P_{y}^0(y|x)\dd x}{\sqrt{2\pi\sigma}} \e^{-\frac{(x-\mu)^2}{2\sigma}}.
\end{align}
Similarly, defining $\bhV=\hat\bR+\bhQ$, we can write down the prior channel. We can write then
{\small\begin{multline}
\Psi_\bw(\bhM,\bhQ,\hat\bR)\coloneqq\lim_{s\to 0}\frac{1}{s}\Psi^{(s)}_\bw(0,\hat\bsM,\hat\bsQ)\\
=\lim_{s\to 0}\frac{1}{sp}\ln\left[\left(\prod_{a=1}^s\int\dd\bW^a P_\bW(\bW^a)\e^{\langle( \bUno_p\otimes\bhM^\intercal)\odot\bhPhi|\bW^a\rangle-\frac{1}{2}{\langle\bW^a}|(\bUno_{p,p}\otimes\bhR)\odot\bsOmega|\bW^a\rangle}\right)\prod_{a<b}\e^{\langle{\bW^a}|(\bUno_{p,p}\otimes\bhQ)\odot\bsOmega)^{1/2}|\bW^b\rangle}\right]\\
=\frac{1}{p}\mathbb E_{\bXi}\left[\ln\left(\int\dd\bW \e^{-\lambda\beta r(\bW)+\langle(\bUno_p\otimes\bhM^\intercal)\odot\bhPhi|\bW\rangle+\langle\bXi|(\bUno_{p,p}\otimes\bhQ)\odot\bsOmega)^{1/2}|\bW\rangle-\frac{1}{2}\langle\bW|(\bUno_{p,p}\otimes\bhV)\odot\bsOmega|\bW\rangle}\right)\right]
\end{multline}}
where we have performed a Hubbard-Stratonovich transformation and $\bXi\equiv (\Xi^i_k)_{k}^{i}\in\R^{p\times K}$ has $\Xi_k^i\sim\mathcal N(0,1)$ for all $i,k$. The free entropy is then
\begin{equation}
\Phi\coloneqq\lim_{s\to 0}\frac{1}{s}\Phi^{(s)}=-\sqrt\gamma\langle\bhM|\bM\rangle+\frac{\langle \bhV|\bV\rangle+\langle \bhV|\bQ\rangle-\langle \bhQ|\bV\rangle}{2}+\alpha\Psi_y(\rho,\bM,\bQ,\bV)+\Psi_\bw(\bhM,\bhQ,\bhV).
\end{equation}
We are interested in the extremum of this quantity, and therefore we have to find the order parameters that maximise it by means of a set of saddle-point equations. Defining for brevity
\begin{equation}
\bomega=\bQ^{1/2}\bxi ,\qquad\omega_0=\langle\bM|\bQ^{-1}|\bomega\rangle,\qquad \sigma_0=\rho-\langle\bM|\bQ^{-1}|\bM\rangle,
\end{equation}
a first set of saddle-point equation is
\begin{subequations}\label{eq:sub2}
\begin{align}
\bhV&=-\alpha\int_{\mathcal Y}\dd y\,\mathbb E_{\bxi}\left[\mathcal Z^0(y,\omega_0,\sigma_0)\partial_\bomega\bff\right],\\
\bhQ&={\alpha}\int_{\mathcal Y}\dd y\,\mathbb E_{\bxi}\left[\mathcal Z^0(y,\omega_0,\sigma_0)\bff\bff^\intercal\right],\\
\bhM&=\frac{\alpha}{\sqrt\gamma}\int_{\mathcal Y}\dd y\,\mathbb E_{\bxi}\left[\mathcal Z^0(y,\omega_0,\sigma_0)f^0\bff \right]
\end{align}
\end{subequations}
where
\begin{align}
f^0&\equiv\partial_{\omega_0}\ln\mathcal Z^0(y,\omega_0,\sigma_0)\\
\bff&\equiv\partial_\bomega\ln\mathcal Z(y,\bomega,\bV).
\end{align}

\paragraph{Zero temperature state evolution equations --} To obtain a nontrivial $\beta\to+\infty$ limit we rescale $\bhV\mapsto \beta\bhV$, $\bV\mapsto \beta^{-1} \bV$, $\bhQ\mapsto\beta^{2}\bhQ$, $\bhM\mapsto \beta\bhM$. After this change of variable, Eqs.~\eqref{sp1o} remain formally identical. To complete the set of saddle-point equations, let us observe that, defining
\begin{equation}
\mathcal L(y,\bu)=\frac{1}{2}\langle\bu-\bomega|\bV^{-1}|\bu-\bomega\rangle+{\hat\ell}(y,\bu)
\end{equation}
then after the rescaling
\begin{equation}
    \ln\mathcal Z(y,\bomega,\beta^{-1}\bV)=\ln\int\frac{\e^{-\beta\mathcal L(y,\bu)}\dd\bu}{\sqrt{\det(2\pi\bV)}}
    \xrightarrow{\beta\gg 1} -\beta\mathcal L(y,\bh)\quad \text{with  }\bh=\arg\min_{\bu}\mathcal L(y,\bu).
\end{equation}
In this way the remaining saddle-point equations keep the form \eqref{eq:sub2} but with
\begin{equation}
\bff\coloneqq \bV^{-1}(\bh-\bomega).
\end{equation}
In the $\beta\to+\infty$ limit, we can write also
\begin{multline}
\Psi_\bw(\bhM,\bhQ,\hat\bR)
=\frac{1}{p}\mathbb E_{\bXi}\left[\ln\left(\int\dd\bW \e^{-\lambda\beta r(\bW)+\beta\langle\bfB|\bW\rangle-\frac{\beta}{2}\langle\bW|(\bUno_{p,p}\otimes\bhV)\odot\bsOmega|\bW\rangle}\right)\right]\\
=-\frac{\beta}{p}\mathbb E_{\bXi}\left[\frac{\langle\bG|(\bUno_{p,p}\otimes\bhV)\odot\bsOmega|\bG\rangle}{2}-\langle \bfB|\bG\rangle+\lambda r(\bG)\right]
\end{multline}
where
\begin{equation}
\bfB\coloneqq(\bUno_p\otimes\bhM^\intercal)\odot\bhPhi+\left((\bUno_{p,p}\otimes\bhQ)\odot\bsOmega\right)^{1/2}\bXi
\end{equation}
and
\begin{equation}
\bG\coloneqq\arg\min_\bU\left[\frac{\langle \bU|(\bUno_{p,p}\otimes\bhV)\odot\bsOmega|\bU\rangle}{2}-\langle\bfB|\bU\rangle+\lambda r(\bU)\right].
\end{equation}
As a result, the remaining saddle point equations are
\begin{subequations}
\begin{align}
\bV&=\frac{1}{p}\mathbb E_\bXi\llangle\bsOmega(\bG\otimes\bXi)\left((\bUno_{p,p}\otimes\bhQ)\odot\bsOmega\right)^{-1/2}\rrangle^\intercal\\
\bQ&=\frac{1}{p}\mathbb E_\bXi\llangle \bG|\bsOmega|\bG\rrangle,\\
\bM&=\frac{1}{\sqrt\gamma p}\mathbb E_\bXi\llangle\bhPhi|\bG\rrangle.\end{align}
\end{subequations}
\paragraph{The case of $\ell_2$ regularisation --} If we assume an ${\hat\ell}_2$ regularisation $r(\bW)=\frac{1}{2}\|\bW\|^2_{\rm F}$, $\Psi_\bw$ is a Gaussian integral that can be explicitly computed before the rescaling in $\beta$. Denoting
\begin{equation}
\bsA\coloneqq\left[\lambda\bI_K\otimes\bI_p+(\bUno_{p,p}\otimes\bhV)\odot\bsOmega\right]^{-1}\quad \text{and}\quad\bsTheta\coloneqq\bhPhi\otimes\bhPhi\in\R^{p\times p}\otimes\R^{K\times K}
\end{equation}
we obtain the following saddle-point equations for $\bV$, $\bQ$ and $\bM$,
\begin{subequations}\label{sp1o}
\begin{align}
\bV&=\frac{\llangle\bsA\circ\bsOmega\rrangle}{p},\\
\bQ&=\frac{1}{p}\llangle\bsOmega\circ\left(\bsA\left((\bUno_{p,p}\otimes\bhM\otimes\bhM^\intercal)\odot\bsTheta+(\bUno_{p,p}\otimes\bhQ)\odot\bsOmega\right)\bsA\right)\rrangle\\
\bM&=\frac{1}{\sqrt\gamma p}\llangle\bsA\left((\bUno_{p,p}\otimes\bhM\otimes\bUno_K^\intercal)\odot\bsTheta\right)\rrangle.\end{align}
\end{subequations}

\paragraph{Training loss and generalisation error --} The order parameters introduced to solve the problem allow us to reach our ultimate goal of computing the average errors of the learning process. We have
\begin{multline}\comprimi
\epsilon_{\hat\ell}\equiv \frac{1}{n}\sum\limits_{\nu=1}^{n}{\hat\ell}\left(y^{\nu}, \frac{\llangle \hat\bW|\bU^\nu\rrangle}{\sqrt p}\right)
\xrightarrow{n\to+\infty}\\
-\partial_\beta\Psi_y=\int\dd y\,\mathbb E_{\bxi}\left[\frac{\mathcal Z^0\!\left(y,\omega_0,\sigma_0\right)}{\mathcal Z\left(y,\bomega,\bV\right)}\int\frac{{\hat\ell}(y,\bu) \e^{-\frac{1}{2}\langle\bomega-\bmu|\bV^{-1}|\bomega-\bu\rangle-\beta{\hat\ell}(y,\bu)}\dd\bu}{\sqrt{\det(2\pi\bV)}}\right]\\
\xrightarrow[\bV\mapsto \beta^{-1}\bV]{\beta\to+\infty}\int\dd y\,\mathbb E_{\bxi}\left[\mathcal Z^0\!\left(y,\omega_0,\sigma_0\right){\hat\ell}(y,\bh)\right],
\end{multline}
where $\bh$ is the proximal introduced above and all overlaps have to be intended computed at the fixed point.

We can also study the generalisation error
\begin{multline}\comprimi
\epsilon_g=\mathbb E_{(y,\bU)}\left[\Delta\left(y,\hat y\left(\frac{\llangle\hat\bW|\bU\rrangle}{\sqrt p}\right)\right)\right]\\
=\int \dd\bmu \int \dd\nu\int\dd y\,\Delta\left(y,\hat y(\bmu)\right)P_y^0(y|\nu)\mathbb E_{\bx}\left[\delta\left(\bmu-\frac{\llangle\hat\bW|\bU\rrangle}{\sqrt p}\right)\delta\left(\nu-\frac{\langle\bx|\btheta\rangle}{\sqrt d}\right)\right]\\
=\int\mathbb E_{(\nu,\bmu)}\left[\Delta\left(y,\hat y(\bmu)\right)P_y^0(y|\nu)\right]\dd y
\end{multline}
where $(\nu,\bmu)$ are jointly Gaussian with covariance
\begin{equation}
\bSigma=\begin{pmatrix}
\rho&\bM^\intercal\\
\bM&\bQ
\end{pmatrix}    
\end{equation}
In particular, if $P_y^0(y|\nu)=\delta(y-f_0(\nu))$, then $\epsilon_g=\mathbb E_{(\nu,\bmu)}\left[\Delta\left(f_0(\nu),\hat y(\bmu)\right)\right]$, which corresponds to \eqref{generror}.

\subsection{Separable loss with ridge regularisation}\label{app:replicasSep}
Let us focus now on the case of separable losses, i.e., losses in the form ${\hat\ell}(y,\bu)=\sum_k\ell(y,u_k)$, which is a crucial special case in the analysis of our contribution. We will assume a ridge regularisation $r(\bW)=\frac{1}{2}\|\bW\|_{\rm F}^2$. Let us also assume that the $K$ generative networks are statistically equivalent. This implies a specific structure in the tensors $\bsTheta$ and $\bsOmega$,
\begin{align}
\bsOmega_{kk'}=\bsOmega_{k'k}^\intercal&\stackrel{\rm d}{=}
\begin{cases}
\bOmega&\text{ for }k=k',\\
\hat\bOmega&\text{ for }k<k',\\
\end{cases}\\
\bsTheta_{kk'}=\bsTheta_{k'k}^\intercal&\stackrel{\rm d}{=}
\begin{cases}
\bTheta&\text{ for }k=k',\\
\hat\bTheta&\text{ for }k<k'.\\
\end{cases}
\end{align}
Here by $\stackrel{\rm d}{=}$ we mean that the equalities hold in distribution. Observe $\bsOmega_{kk}$ and $\bsOmega_{kk'}$ are not uncorrelated quantities. For reasons of symmetry reasons, we impose therefore the ansatz
\begin{equation}
\begin{split}
\bV&=v\bI_K,\\
\bM&=m\bUno_K,\\
\bQ&=(q_0-q_1)\bI_K+q_1\bUno_{K,K},
\end{split}\qquad
\begin{split}
\bhV&=\hat v\bI_K,\\
\bhM&=\hat m\bUno_K,\\
\bhQ&=(\hat q_0-\hat q_1)\bI_K+\hat q_1\bUno_{K,K}.
\end{split}
\end{equation}
It is easily seen that
\begin{align} \bQ^{1/2}&=\sqrt{q_0-q_1}\bI_K+\frac{\sqrt{q_0+(K-1)q_1}-\sqrt{q_0-q_1}}{K}\bUno_{K,K},\\ 
\bsA_{kk'}&=(\lambda \bI_p+\bsOmega_{kk})^{-1}\delta_{kk'}.
\end{align}
Plugging this ansatz in our equations, and introducing $\eta=\sfrac{q_1}{q_0}$, we obtain
\begin{equation}\label{sp2ansaz}
\begin{split}
\hat v&=-\alpha\int_{\mathcal Y}\dd y\,\mathbb E_{\zeta}\left[\mathcal Z^0\!\left(y,\omega_0,\sigma_0\right)\partial_{\omega}f\right],\\
\hat m&=\frac{\alpha}{\sqrt \gamma}\int_{\mathcal Y}\dd y\,\mathbb E_{\zeta}\left[f\partial_{\mu}\mathcal Z^0\!\left(y,\omega_0,\sigma_0\right) \right],\\
\hat q_0&=\alpha\int_{\mathcal Y}\dd y\,\mathbb E_{\zeta}\left[\mathcal Z^0\!\left(y,\omega_0,\sigma_0\right)f^2\right],\\
\hat q_1&=\alpha\mathbb E_{y,\zeta,\zeta'}\left[\mathcal Z^0\!\left(y,\frac{m}{
\sqrt{q_0}}\frac{\zeta+\zeta'}{1+\eta},\rho-\frac{2m^2}{q_0+q_1}\right)ff'\right],
\end{split}\quad \text{with }\omega_0\coloneqq\frac{m\zeta}{\sqrt{q_0}},\quad \sigma_0\coloneqq\rho-\frac{m^2}{q_0}.
\end{equation}
The new variables $\zeta_1$ and $\zeta_2$ are obtained by a linear transformation from the old ones. In particular, they are distributed as two components of a vector 
\begin{equation}
\bzeta=\left(\sqrt{1-\eta}\bI_K+\frac{\sqrt{1+(K-1)\eta}-\sqrt{1-\eta}}{K}\bUno_{K,K}\right)\bxi,
\end{equation}
where $\bxi\sim\mathcal N(\mathbf 0_K,\bI_K)$. It follows that $\zeta,\zeta'\sim\mathcal N(0,1)$ but they are correlated as
\begin{equation}
\mathbb E[\zeta\zeta']=\eta.
\end{equation}
Moreover, we have introduced the proximal
\begin{equation}
f=\frac{h-\omega}{v}\qquad\text{where}\qquad h=\arg\min_x\left[\frac{(x-\sqrt{q_0}\zeta)^2}{2v}+\ell(y,x)\right]
\end{equation}
and the corresponding $f'$ obtained using $\zeta'$. The remaining equations read
\begin{subequations}\label{sp1ansaz}
\begin{align}
v&=\frac{\tr[(\lambda \bI_p+\hat v\bOmega)^{-1}\bOmega]}{p},\\
m&=\frac{\hat m}{\sqrt{\gamma}}\frac{\tr[\bTheta(\lambda \bI_p+\hat v\bOmega)^{-1}]}{p}\\
q_0&=\frac{\tr[(\lambda \bI_p+\hat v\bOmega)^{-1}\left(\hat m^2\bTheta+\hat q_0\bOmega\right)(\lambda \bI_p+\hat v\bOmega)^{-1}\bOmega]}{p}\\
q_1&=\frac{\tr[(\lambda \bI_p+\hat v\bOmega)^{-1}\left(\hat m^2\hat\bTheta+\hat q_1\hat\bOmega\right)(\lambda \bI_p+\hat v\bOmega')^{-1}\hat\bOmega{}^\intercal]}{p}.\end{align}
\end{subequations}
\paragraph{The random-features model for the generative networks --} To further simplify these expressions suppose now that our generative networks are such that
\begin{equation}
\bu_k(\bx)=\phi\left(\frac{\bF_k\bx}{\sqrt d}\right),\quad k\in[K],
\end{equation}
where $\bF_k\in\R^{p\times d}$ are (fixed) random matrices extracted from some given distribution and $\phi$ is a nonlinearity acting elementwise. As anticipated in the main text, we can use the fact that each generative network is equivalent to the following Gaussian model \citep{mei2020generalization}
\begin{equation}
\bu_k(\bx)\mapsto \kappa_0\bUno_p+ \kappa_1\frac{\bF_k\bx}{\sqrt d}+\kappa_*\bz_k.
\end{equation}
for some coefficients $\kappa_0$, $\kappa_1$ and $\kappa_*$ depending on $\phi$ (see Therorem \ref{th:Teo1}), and $\bz^\mu\sim\mathcal N(\boldsymbol 0_p,\bI_p)$. Assuming now, for the sake of simplicity, to the $\kappa_0=0$ case and that $\mathbb E_\btheta[\btheta\btheta^\intercal]=\bI_d$, then
\begin{equation}
\begin{split}
\bOmega&
\stackrel{\rm d}{=}\frac{\kappa_1^2}{d}\bF^\intercal\bF+\kappa_*^2\bI_p,\\
\bTheta&\stackrel{\rm d}{=}\frac{\kappa_1^2}{d}\bF^\intercal\bF,
\end{split}\qquad
\begin{split}
\hat\bTheta&\stackrel{\rm d}{=}\hat\bOmega\stackrel{\rm d}{=}\frac{\kappa_1^2}{d}\bF^\intercal\bF',\quad \bF\stackrel{\rm d}{=}\bF'\stackrel{\rm d}{=}\bF_k\ \forall k\in[K].
\end{split}\end{equation}
Once the spectral density $\varrho(s)$ of $\bOmega$ is introduced, it is immediate to see that the equations for $q_0$, $m$ and $v$ take the forms given in the main text. The equation for $q_1$ requires an additional step. If we introduce the symmetric random matrix
\begin{equation}
\hat \bF\coloneqq \frac{\kappa_1^2}{d}\bF\left((\lambda+\hat v\kappa_*^2)\bI_p+\frac{\hat v\kappa_1^2}{d}\bF^\intercal\bF\right)^{-1}\bF^\intercal\in \R^{p\times p}
\end{equation}
then we can rewrite the equation as
\begin{equation}
q_1=\left(\hat m^2+\hat q_1\right)\frac{\tr [\hat\bF\hat\bF']}{p}=\frac{\hat m^2+\hat q_1}{\gamma}\left(\frac{\tr \hat\bF}{p}\right)^2=\left(1+\frac{\hat q_1}{\hat m^2}\right)m^2,
\end{equation}
where in the second equality we used the fact that $\hat\bF$ and $\hat\bF'$ are asymptotically free.
\paragraph{Ridge regression --} Let us consider the simple case of ridge regression with $f_0(x)=x$. We will give here the channel equations that are obtained straightforwardly as
\begin{equation}
\begin{split}
 \hat v &= \frac{\alpha}{1+v}\\
\hat{m} &= \frac{1}{1+v}\frac{\alpha}{\sqrt\gamma},
\end{split}\qquad
\begin{split}
\hat q_0&=\alpha\frac{\rho-2m+q_0}{(1+v)^2},\\
\hat q_1&=\alpha\frac{\rho-2m+q_1}{(1+v)^2}.
\end{split}
\label{eq:channel_redige}
\end{equation}
The kernel limit is also obtained straightforwardly taking the $\alpha\to 0$ limit and rescaling $\hat q_0\mapsto\alpha {\hat q_0}$, $\hat q_1\mapsto \alpha {\hat q_1}$, $\hat m\mapsto\sqrt \alpha {\hat m}$, $\hat v\mapsto\alpha {\hat v}$ \label{app:sec:ridge}.
\begin{align}
    \begin{split}
    v &=\frac{(1-\delta) \kappa_1^2+\sqrt{(1-\delta)^2 \kappa_1^4+2 (\kappa_*^2+\lambda )(1+\delta) \kappa_1^2 +(\kappa_*^2+\lambda)^2}+\kappa_*^2-\lambda }{2 \lambda }\\
    m&=\frac{1}{1+\lambda\frac{v+1}{\delta  \kappa_1^2}}\\
    q_0 &= q_1 = \frac{\delta -2 m+\rho}{\left(1  +2 \lambda\frac{v+1}{\delta\kappa_1^2}\right)^2-1}\equiv q.
    \end{split}
\end{align}
\paragraph{Binary classification problem}\label{app:binary}
We consider now the case $f_0(x)=\mathrm{sign}(x)$, corresponding to a binary classification problem, and we write down the channel equations for this problem in the case of logistic and hinge loss. In this case we have that
\begin{equation}
\begin{split}
    \mathcal Z^0(y,\omega_0,\sigma_0)&=\frac{\delta(y-1)+\delta(y+1)}{2}\left(1+\mathrm{erf}\left(\frac{y\omega_0}{\sqrt{2\sigma_0}}\right)\right),\\
     \partial_{\mu} \mathcal Z^0(y,\omega_0,\sigma_0)&=\left(\delta(y-1)-\delta(y+1)\right)\frac{\e^{-\frac{\omega_0^2}{2\sigma_0}}}{\sqrt{2\pi\sigma_0}}.
\end{split}
\end{equation}

If we pick a logistic loss in the form ${\hat\ell}(y,\bmu)=\sum_k\ln(1+\e^{-y\mu_k})$, then the proximal $h$ solves the equation
\begin{equation} h=\omega+\frac{yv}{1+\e^{yh}},
\end{equation}
in such a way that $f=\frac{\eta-\omega}{v}$ satisfies
\begin{equation}
\partial_\omega f=-\left(v+2\cosh\left(y\frac{vf+\omega}{2}\right)\right)^{-1}.
\end{equation}
If we use instead a hing loss ${\hat\ell}(y,\bmu)=\sum_k\max(0,1-y\mu_k)$, the proximal is such that
\begin{equation}
f=\begin{cases}
y&\text{if }1-v>\omega y,\\
\frac{y-\omega}{v}&\text{if }1-v<\omega y<1,\\
0&\text{otherwise},
\end{cases}
\qquad
\partial_\omega f=
\begin{cases}
-\frac{1}{v}&\text{if }1-v<\omega y<1,\\
0&\text{otherwise}.
\end{cases}
\end{equation}
In the case of the hinge loss, the simple form of the proximal allows for a more explicit expression of the channel equations. Introducing
\begin{equation}
\hat\sigma=\rho-\frac{2m^2}{q_0+q_1},\qquad \mathcal N(\zeta,\zeta';\eta)=\frac{\exp\left(-\frac{\zeta^2+{\zeta'}^2-2\eta\zeta\zeta'}{2(1-\eta^2)}\right)}{2\pi\sqrt{1-\eta^2}}    
\end{equation}
we obtain
\begin{equation}\comprimi
\begin{split}
\hat v=&\frac{\alpha}{v}\int_{\frac{1-v}{\sqrt{q_0}}}^{\sfrac{1}{\sqrt{q_0}}}\frac{\e^{-\frac{\zeta^2}{2}}\dd\zeta}{\sqrt{2\pi}}\left(1+\erf\left(\frac{m\zeta}{\sqrt{2q_0\sigma}}\right)\right),\\
\hat m=&\frac{\alpha}{\sqrt {\rho\gamma}v}\left[\frac{v+\erf\left(\sqrt{\frac{\rho}{2q_0\sigma}}\right)-(1-v)\erf\left((1-v)\sqrt{\frac{\rho}{2q_0\sigma}}\right)}{\sqrt{2\pi}}+\sqrt{\frac{q_0\sigma}{\rho}}\frac{\exp\left(-\frac{\rho}{2q_0\sigma}\right)-\exp\left(\frac{\rho(1-v)^2}{2q_0\sigma}\right)}{2\pi}\right]\\
\hat q_0=&\alpha\left[\int^{\frac{1-v}{\sqrt{q_0}}}_{-\infty}\frac{\e^{-\frac{\zeta^2}{2}}\dd\zeta}{\sqrt{2\pi}}\left(1+\erf\left(\frac{m\zeta}{\sqrt{2q_0\sigma}}\right)\right)+\int_{\frac{1-v}{\sqrt{q_0}}}^{\sfrac{1}{\sqrt{q_0}}}\frac{\e^{-\frac{\zeta^2}{2}}\dd\zeta}{\sqrt{2\pi}}\left(1+\erf\left(\frac{m\zeta}{\sqrt{2q_0\sigma}}\right)\right)\left(\frac{1-\sqrt{q_0}\zeta}{v}\right)^2\right],\\
\hat q_1=&\alpha\iint^{\frac{1-v}{\sqrt{q_0}}}_{-\infty}\mathcal N(\zeta,\zeta';\sfrac{q_1}{q_0})\dd\zeta\dd\zeta'\left(1+\erf\left(\frac{m}{\sqrt{2q_0\hat\sigma}}\frac{\zeta+\zeta'}{1+\sfrac{q_1}{q_0}}\right)\right)\\
&+2\alpha\int_{\frac{1-v}{\sqrt{q_0}}}^{\sfrac{1}{\sqrt{q_0}}}\dd\zeta\int_{-\infty}^{\frac{1-v}{\sqrt{q_0}}}\dd\zeta'\mathcal N(\zeta,\zeta';\sfrac{q_1}{q_0})\left(1+\erf\left(\frac{m}{\sqrt{2q_0\hat\sigma}}\frac{\zeta+\zeta'}{1+\sfrac{q_1}{q_0}}\right)\right)\left(\frac{1-\sqrt{q_0}\zeta}{v}\right)\\
&+\alpha\iint_{\frac{1-v}{\sqrt{q_0}}}^{\sfrac{1}{\sqrt{q_0}}}\dd\zeta\dd\zeta'\mathcal N(\zeta,\zeta';\sfrac{q_1}{q_0})\left(1+\erf\left(\frac{m}{\sqrt{2q_0\hat\sigma}}\frac{\zeta+\zeta'}{1+\sfrac{q_1}{q_0}}\right)\right)\left(\frac{1-\sqrt{q_0}\zeta}{v}\right)\left(\frac{1-\sqrt{q_0}\zeta'}{v}\right).
\label{eq:channel_binary_class}
\end{split}
\end{equation}
Let us now make the change of variables $\zeta\mapsto \frac{\sqrt{q_0+q_1}z+\sqrt{q_0-q_1}z'}{\sqrt{2q_0}}$ and $\zeta'\mapsto \frac{\sqrt{q_0+q_1}z-\sqrt{q_0-q_1}z'}{\sqrt{2q_0}}$. This allows us to rewrite the expression for $q_1$ as
{\small\begin{multline}\comprimi
\hat q_1=\alpha\int^{\sqrt{2}(1-v)}_{-\infty}\mathcal N(z;0,q_0+q_1)\left(1+\erf\left(\frac{m z}{\sqrt{\hat\sigma}(q_0+q_1)}\right)\right)\erf\left(\frac{\sqrt 2(1-v)-z}{\sqrt{q_0-q_1}}\right)\\\comprimi
+2\alpha\int_{-\infty}^{\sqrt{2}(1-\sfrac{v}{2})}\mathcal N(z;0,q_0+q_1)\left(1+\erf\left(\frac{m z}{\sqrt{\hat\sigma}(q_0+q_1)}\right)\right)\left(\frac{1}{v}-\frac{1}{v}\left[\frac{z}{2}\erf\left(\frac{x}{\sqrt{q_0-q_1}}\right)+\sqrt{\frac{q_0-q_1}{2\pi}}\e^{-\frac{x^2}{2}}\right]_{|\sqrt 2(1-v)-z|}^{\sqrt 2-z}\right)\dd z\\
+\alpha\iint_{\frac{1-v}{\sqrt{q_0}}}^{\sfrac{1}{\sqrt{q_0}}}\dd\zeta\dd\zeta'\mathcal N(z;0,q_0+q_1)\mathcal N(z':0,1)\left(1+\erf\left(\frac{m z}{\sqrt{\hat\sigma}(q_0+q_1)}\right)\right)\left(\frac{1-\sqrt{q_0}\zeta}{v}\right)\left(\frac{1-\sqrt{q_0}\zeta'}{v}\right).
\end{multline}}

\section{Proof of the main theorem}

In this section we prove Theorem \ref{th:TheGen}, from which all other analytical results in the paper can be deduced. We start by reminding the learning problem defining the ensemble of estimators with a few auxiliary notations, so that this part is self contained. The sketch of proof is one pioneered in \citep{bayati2011lasso, donoho2016high} and is the following: the estimator $\mathbf{W}^{*}$ is expressed as the limit of a carefully chosen sequence, an \emph{approximate message-passing iteration} \citep{bayati2011dynamics,zdeborova2016statistical}, whose iterates can be asymptotically exactly characterized using an auxiliary, closed form iteration, the \emph{state evolution equations}. We then show that converging trajectories of such an AMP iteration can be systematically found.
\subsection{The learning problem}
We start by reminding the definition of the problem. Consider the following generative model
\begin{equation}
\mathbf{y} = f_{0}(\frac{1}{\sqrt{d}}\mathbf{X}_{0}\mathbf{w}_{0}, \boldsymbol{\epsilon}_{0})
\end{equation}
where $\by \in \mathbb{R}^{n},\bX_{0} \sim \mathcal{N}(0,\boldsymbol{\Sigma}_{00}) \in \mathbb{R}^{n \times d}, \mathbf{w}_{0} \in \mathbb{R}^{d}$, $\boldsymbol{\epsilon}_{0} \in \mathbb{R}^{d}$ is a noise vector and $\boldsymbol{\Sigma_{00}}\in \mathbb{R}^{d\times d}$ is a positive definite matrix. The goal is to learn this generative model using an ensemble of predictors $\mathbf{W} = \begin{bmatrix}\mathbf{w}_{1}\vert \mathbf{w}_{2} \vert... \vert \mathbf{w}_{K}\end{bmatrix} \in \mathbb{R}^{p \times K}$ where each predictor $\mathbf{w}_{k} \in \mathbb{R}^{p}, k \in [1,K]$
is learned using a sample dataset $\mathbf{X}_{k} \in \mathbb{R}^{n \times p}$, where, for any $i \in [1,n]$ and $k \in [0,K]$, we have:
\begin{equation}
    \mathbb{E}\left[\mathbf{x}_{i}^{k}(\mathbf{x}_{i}^{k'})^{\top}\right] = \boldsymbol{\Sigma}_{kk'}
\end{equation}
where each sample is Gaussian and we denote :
\begin{equation}
    \boldsymbol{\Sigma} = \begin{bmatrix}\boldsymbol{\Sigma}_{00}&\boldsymbol{\Sigma}_{01}&...&\boldsymbol{\Sigma}_{0K} \\ \boldsymbol{\Sigma}_{10}&\boldsymbol{\Sigma}_{11}&...&\boldsymbol{\Sigma}_{1K} \\ ... \\\boldsymbol{\Sigma}_{K0}&\boldsymbol{\Sigma}_{K1}&...&\boldsymbol{\Sigma}_{KK}\end{bmatrix} \in \mathbb{R}^{(Kp+d) \times (Kp+d)}.
\end{equation}
The predictors interact with each sample dataset in a linear way, i.e. we will consider 
a generalized linear model acting on the ensemble of products $\left\{\mathbf{X}_{k}\mathbf{w}_{k}\right\}_{k=1}^{K}$:
\begin{equation}
    \label{eq:student}
    \mathbf{W}^{*} \in \underset{{\mathbf{W}\in \mathbb{R}^{p \times K}}}{\arg\min} \mathcal{L}\left(\mathbf{y},\left\{\frac{1}{\sqrt{p}}\mathbf{X}_{k}\mathbf{w}_{k}\right\}_{k=1}^{K}\right)+r_{0}(\mathbf{W})
\end{equation}
where $\mathcal{L},r_{0}$ are convex functions.
We wish to determine the asymptotic properties of the estimator $\mathbf{W}^{*}$ in the limit where $n,p,d \to \infty$ with fixed ratios $\alpha = n/p, \gamma = d/p$. We now list the necessary assumptions for our main theorem to hold.
\paragraph{Assumptions --}
\begin{itemize}
    \item the functions $\mathcal{L},r_{0}$ are proper, closed, lower-semicontinuous, convex functions. The loss function $\mathcal{L}$ is pseudo-lipschitz of order 2 in both its arguments and the regularisation $r_{0}$ is pseudo-Lipschitz of order 2. The cost function $\mathcal{L}(\mathbf{X}.)+r(.)$ is coercive.
    \item for any $1\leqslant k \leqslant K$, the matrix $\boldsymbol{\Sigma}_{k} \in \mathbb{R}^{p\times p}$ is symmetric and there exist strictly positive constants $\kappa_{0},\kappa_{1}$ such that 
    $\kappa_{0} \leqslant \lambda_{min}(\boldsymbol{\Sigma}_{k}) \leqslant \lambda_{max}(\boldsymbol{\Sigma}_{k}) \leqslant \kappa_{1}$. We also assume that the matrix $\boldsymbol{\Sigma}$ is positive definite.
    \item their exists a positive constant $C_{f_{0}}$ such that $\norm{f_{0}(\frac{1}{\sqrt{d}}\mathbf{X}_{0}\mathbf{w}_{0}, \boldsymbol{\epsilon}_{0})}_{2}\leqslant C_{f_{0}}\left( \norm{\frac{1}{\sqrt{d}}\mathbf{X}_{0}\mathbf{w}_{0}}_{2}+\norm{\boldsymbol{\epsilon}_{0}}_{2}\right)$
    \item the dimensions $n,p,d$ grow linearly with finite ratios $\alpha = n/p$ and $\gamma = d/p$.
    \item the ground truth vector $\mathbf{w}_{0} \in \mathbb{R}^{d}$ and noise vector $\boldsymbol{\epsilon}_{0} \in \mathbb{R}^{n}$ are sampled from subgaussian probability distributions independent from each other and from all other random quantities of the learning problem.
\end{itemize}
The proof method we will employ involves expressing the estimator $\mathbf{W}^{*}$ as the limit of a carefully chosen sequence. In the case of non-strictly convex problems, the estimator may not be unique, making it unclear what estimator is reached by the sequence (at best we know it belongs to the set of zeroes of the subgradient of the cost function). We thus start with the following problem
\begin{align}
    \label{eq:sc_student}
    &\mathbf{W}^{*} \in \underset{{\mathbf{W}\in \mathbb{R}^{p \times K}}}{\arg\min} \mathcal{L}(\mathbf{y},\left\{\mathbf{X}_{k}\mathbf{w}_{k}\right\}_{k=1}^{K})+r_{\lambda_{2}}(\mathbf{W}) \\
    &\mbox{where, for any $\mathbf{W} \in \mathbb{R}^{p \times K}$,} \thickspace r_{\lambda_{2}}(\mathbf{W}) = r_{0}(\mathbf{W})+\frac{\lambda_{2}}{2}\norm{\mathbf{W}}_{F}^{2}
\end{align}
i.e. we add a ridge regularisation to the initial problem to make it strongly convex. We will relax this additional strong convexity constraint later on.
\subsection{Asymptotics for the strongly convex problem}
We now reformulate the minimization problem Eq.\eqref{eq:sc_student} to make it amenable to an approximate message-passing iteration (AMP). The key feature of this ensembling problem, outside of the convexity 
which will be crucial to control the trajectories of the AMP iteration, is the fact that 
each predictor only interacts linearly with each design sample, along with the correlation structure of the overall dataset. 
We are effectively sampling $n$ vectors of size $(Kp+d)$ from the Gaussian distribution with covariance $\boldsymbol{\Sigma}$, i.e. $
    \begin{bmatrix}\mathbf{x}_{0}\vert\mathbf{x}_{1}\vert...\vert\mathbf{x}_{K}\end{bmatrix} \sim \mathcal{N}(0,\boldsymbol{\Sigma})$.
We then write $\left\{\mathbf{X}_{k}\mathbf{w}_{k}\right\}_{k=0}^{K} = \begin{bmatrix} \mathbf{X}_{0}\mathbf{w}_{0}\vert ... \vert \mathbf{X}_{K}\mathbf{w}_{K}\end{bmatrix} \in \mathbb{R}^{n \times (K+1)}$, such that
\begin{align}
    \begin{bmatrix} \mathbf{X}_{0}\mathbf{w}_{0}\vert ... \vert \mathbf{X}_{K}\mathbf{w}_{K}\end{bmatrix} &= \begin{bmatrix} \mathbf{X}_{0}\vert ... \vert \mathbf{X}_{K}\end{bmatrix}\mathbf{W} = \mathbf{Z}\boldsymbol{\Sigma}^{1/2}\begin{bmatrix}\mathbf{w}_{0}&0 \\0&\tilde{\mathbf{W}}\end{bmatrix} \\
    \mbox{where} \quad \tilde{\mathbf{W}} &= \begin{bmatrix}\mathbf{w}_{1}&0&...&0 \\ 0&\mathbf{w}_{2}&...&0\\ &&... &\\0&0&...&\mathbf{w}_{K}\end{bmatrix} \in \mathbb{R}^{Kp \times K}
\end{align}
and $\mathbf{Z} \in \mathbb{R}^{n \times (Kp+d)}$ is a random matrix with i.i.d. $\mathcal{N}(0,1)$ elements.
Then, any sample $\begin{bmatrix}\mathbf{x}_{0}\vert\mathbf{x}_{1}\vert...\vert\mathbf{x}_{K}\end{bmatrix}$ may be rewritten as
\begin{align}
    \mathbf{x}_{0} &= \Psi^{1/2}\mathbf{a} \quad \mbox{and} \quad \begin{bmatrix}\mathbf{x}_{1}\vert...\vert\mathbf{x}_{K}\end{bmatrix} = \Phi^{\top}\Psi^{-1/2}\mathbf{a}+\left(\Omega-\Phi^{\top}\Psi^{-1}\Phi\right)^{1/2}\mathbf{b} \\
    \mathbf{X}_{0} &= \mathbf{A}\Psi^{1/2} \quad \mbox{and} \quad \begin{bmatrix}\mathbf{X}_{1}\vert...\vert\mathbf{X}_{K}\end{bmatrix} = \mathbf{A}\Psi^{-1/2}\Phi+\mathbf{B}\left(\Omega-\Phi^{\top}\Psi^{-1}\Phi\right)^{1/2}
\end{align}
where $\mathbf{a} \in \mathbb{R}^{d}, \mathbf{b} \in \mathbb{R}^{Kp}$ are vectors with i.i.d. standard normal components, $\mathbf{A} \in \mathbb{R}^{n \times d}, \mathbf{B} \in \mathbb{R}^{n \times Kp}$ are the corresponding design matrices, and the covariance
matrices are given by $\boldsymbol{\Psi} = \boldsymbol{\Sigma}_{00} \in \mathbb{R}^{d \times d}, \boldsymbol{\Phi} = \begin{bmatrix}\boldsymbol{\Sigma}_{11}\vert \boldsymbol{\Sigma}_{12} \vert \boldsymbol{\Sigma}_{13}... \vert \boldsymbol{\Sigma}_{1K}\end{bmatrix} \in \mathbb{R}^{d \times Kp}$ and 
\begin{equation}
\boldsymbol{\Omega} =  \begin{bmatrix}\boldsymbol{\Sigma}_{11}&\boldsymbol{\Sigma}_{12}&...&\boldsymbol{\Sigma}_{1K} \\ \boldsymbol{\Sigma}_{21}&\boldsymbol{\Sigma}_{22}&...&\boldsymbol{\Sigma}_{2K} \\ ... \\\boldsymbol{\Sigma}_{K1}&\boldsymbol{\Sigma}_{K2}&...&\boldsymbol{\Sigma}_{KK}\end{bmatrix} \in \mathbb{R}^{Kp \times Kp}
\end{equation}
The optimization problem may then be written, introducing the appropriate scalings
\begin{align}
    \tilde{\mathbf{W}}^{*} \in \underset{\tilde{\mathbf{W}} \in \mathbb{R}^{Kp \times K}}{\arg\min} \mathcal{L}\left(f_{0}(\frac{1}{\sqrt{d}}\mathbf{A}\tilde{\mathbf{w}}_{0}), \frac{1}{\sqrt{p}}\left(\mathbf{A}\Psi^{-1/2}\Phi+\mathbf{B}\left(\Omega-\Phi^{\top}\Psi^{-1}\Phi\right)^{1/2}\right)\tilde{\mathbf{W}}\right)+r(\tilde{\mathbf{W}})
\end{align}
where we let $\tilde{\mathbf{w}}_{0} = \Psi^{1/2}\mathbf{w}_{0}$, its scaled norm $\rho_{\tilde{\mathbf{w}}_{0}} = \frac{1}{d}\norm{\tilde{\mathbf{w}}_{0}}_{2}^{2}$ and we introduced the function
\begin{align}
r : \mathbb{R}^{Kp \times K} &\to \mathbb{R} \\
\tilde{\mathbf{W}} &\to r_{\lambda_{2}}(\mathbf{W}) 
\end{align}.
In order to isolate the contribution correlated with the teacher, we condition the design matrix $\mathbf{A}$ on the teacher distribution $\mathbf{y}$, we can write
\begin{align}
    \mathbf{A} &= \mathbb{E}\left[\mathbf{A} \vert \mathbf{y}\right]+\mathbf{A}-\mathbb{E}\left[\mathbf{A} \vert \mathbf{y}\right] \\
    &= \mathbb{E}\left[\mathbf{A} \vert \mathbf{A}\tilde{\mathbf{w}}_{0}\right]+\mathbf{A}-\mathbb{E}\left[\mathbf{A} \vert \mathbf{A}\tilde{\mathbf{w}}_{0}\right] \\
    &= \mathbf{A}\mathbf{P}_{\tilde{\mathbf{w}_{0}}}+\tilde{\mathbf{A}}\mathbf{P}^{\perp}_{\tilde{\mathbf{w}_{0}}}
\end{align}
where $\tilde{\mathbf{A}}$ is an independent copy of $\mathbf{A}$, see \citep{bayati2011dynamics} Lemma 11. The cost function then becomes
\begin{align}
\mathcal{L}\left(f_{0}\left(\sqrt{\rho_{\tilde{\mathbf{w}}_{0}}}\mathbf{s}\right), \frac{1}{\sqrt{p}}\left(\mathbf{s}\frac{(\Phi^{\top}\mathbf{w}_{0})^{\top}}{\sqrt{d\rho_{\tilde{\mathbf{w}}_{0}}}}+\tilde{\mathbf{A}}\mathbf{P}^{\perp}_{\tilde{\mathbf{w}_{0}}}\Psi^{-1/2}\Phi+\mathbf{B}\left(\Omega-\Phi^{\top}\Psi^{-1}\Phi\right)^{1/2}\right)\tilde{\mathbf{W}}\right)+r(\tilde{\mathbf{W}})
\end{align}
where $\mathbf{s} = \mathbf{A}\frac{\tilde{\mathbf{w}}_{0}}{\norm{\tilde{\mathbf{w}}_{0}}_{2}} \in \mathbb{R}^{n}$ is an i.i.d. standard normal vector. The term $\tilde{\mathbf{A}}\mathbf{P}^{\perp}_{\tilde{\mathbf{w}_{0}}}\Psi^{-1/2}\Phi+\mathbf{B}\left(\Omega-\Phi^{\top}\Psi^{-1}\Phi\right)^{1/2}$ can then be represented as a $\mathbb{R}^{n \times Kp}$ Gaussian matrix with covariance
\begin{align}
    \Phi^{\top}\Psi^{-1/2}\mathbf{P}_{\tilde{\mathbf{w}}_{0}}^{\perp}\Psi^{-1/2}\Phi+\Omega-\Phi^{\top}\Psi^{-1}\Phi &= \Omega-\Phi^{\top}\Psi^{-1/2}\mathbf{P}_{\tilde{\mathbf{w}}_{0}}\Psi^{-1/2}\Phi \\
    &=\Omega-\Phi^{\top}\Psi^{-1/2}\frac{\tilde{\mathbf{w}}_{0}\tilde{\mathbf{w}}_{0}^{\top}}{\norm{\tilde{\mathbf{w}}_{0}}_{2}^{2}}\Psi^{-1/2}\Phi = \Omega-\frac{\mathbf{c}\mathbf{c}^{\top}}{\norm{\tilde{\mathbf{w}}_{0}}_{2}^{2}}
\end{align}
where we introduced $\mathbf{c} = \Phi^{\top}\mathbf{w}_{0} \in \mathbb{R}^{Kp}$ and $ \rho_{\mathbf{c}} = \frac{1}{p}\norm{\mathbf{c}}_{2}^{2}$, reaching the cost function
\begin{align}
    \mathcal{L}\left(f_{0}\left(\sqrt{\rho_{\tilde{\mathbf{w}}_{0}}}\mathbf{s}\right), \frac{1}{\sqrt{p}}\left(\mathbf{s}\frac{\mathbf{c}^{\top}}{\sqrt{d\rho_{\tilde{\mathbf{w}}_{0}}}}+\mathbf{Z}\left(\Omega-\frac{\mathbf{c}\mathbf{c}^{\top}}{\norm{\tilde{\mathbf{w}}_{0}}_{2}^{2}}\right)^{1/2}\right)\tilde{\mathbf{W}}\right)+r(\tilde{\mathbf{W}})
\end{align} 
Introducing $\mathbf{m} = \frac{1}{\sqrt{dp}}\tilde{\mathbf{W}}^{\top}\mathbf{c} \in \mathbb{R}^{K}, \mathbf{C} = \Omega-\frac{\mathbf{c}\mathbf{c}^{\top}}{\norm{\tilde{\mathbf{w}}_{0}}_{2}^{2}} \in \mathbb{R}^{Kp \times Kp}$, and the Lagrange multiplier $\boldsymbol{\nu}$ associated to $\mathbf{m}$, the optimization problem can equivalently be written
\begin{align}
\label{eq_inter1}
    \inf_{\mathbf{m}\in \mathbb{R}^{K},\tilde{\mathbf{W}}\in \mathbb{R}^{Kp\times K}}\sup_{\boldsymbol{\nu}\in \mathbb{R}^{K}}\mathcal{L}\left(f_{0}\left(\sqrt{\rho_{\tilde{\mathbf{w}}_{0}}}\mathbf{s}\right), \mathbf{s}\frac{\mathbf{m}^{\top}}{\sqrt{\rho_{\tilde{\mathbf{w}}_{0}}}}+\frac{1}{\sqrt{p}}\mathbf{Z}\mathbf{C}^{1/2}\tilde{\mathbf{W}}\right)+r(\tilde{\mathbf{W}})-\boldsymbol{\nu}^{\top}\left(\tilde{\mathbf{W}}^{\top}\mathbf{c}-\sqrt{dp}\mathbf{m}\right)
\end{align} 
We now look for an explicit expression of the matrix square root $\mathbf{C}^{1/2}$
\begin{align}
    \mathbf{C} &= \Omega^{1/2}\left(Id-\frac{\Omega^{-1/2}\mathbf{c}(\Omega^{-1/2}\mathbf{c})^{\top}}{\norm{\tilde{\mathbf{w}}_{0}}_{2}^{2}}\right)\Omega^{1/2} \quad \mbox{let} \quad \tilde{\mathbf{c}} = \Omega^{-1/2}\mathbf{c}\\
    &= \Omega^{1/2}\left(\mathbf{P}_{\tilde{\mathbf{c}}}^{\perp}+\kappa\mathbf{P}_{\tilde{\mathbf{c}}}\right)\Omega^{1/2} \quad \mbox{where} \quad \kappa = 1-\frac{\norm{\tilde{\mathbf{c}}}_{2}^{2}}{\norm{\tilde{\mathbf{w}}_{0}}_{2}^{2}} \\
    &= \Omega^{1/2}\left(\mathbf{P}_{\tilde{\mathbf{c}}}^{\perp}+\sqrt{\kappa}\mathbf{P}_{\tilde{\mathbf{c}}}\right)\left(\mathbf{P}_{\tilde{\mathbf{c}}}^{\perp}+\sqrt{\kappa}\mathbf{P}_{\tilde{\mathbf{c}}}\right)\Omega^{1/2}
\end{align}
where the positivity of $\kappa$ is ensured by the positive-definiteness of $\boldsymbol{\Sigma}$. The problem then becomes 
\begin{align}
    \inf_{\mathbf{m},\tilde{\mathbf{W}}} \sup_{\boldsymbol{\nu}}\mathcal{L}\left(f_{0}\left(\sqrt{\rho_{\tilde{\mathbf{w}}_{0}}}\mathbf{s}\right), \mathbf{s}\frac{\mathbf{m}^{\top}}{\sqrt{\rho_{\tilde{\mathbf{w}}_{0}}}}+\frac{\sqrt{\kappa}}{\sqrt{p}}\mathbf{Z}\mathbf{P}_{\tilde{\mathbf{c}}}\Omega^{1/2}\tilde{\mathbf{W}}+\frac{1}{\sqrt{p}}\tilde{\mathbf{Z}}\mathbf{P}_{\tilde{\mathbf{c}}}^{\perp}\Omega^{1/2}\tilde{\mathbf{W}}\right)&+r(\tilde{\mathbf{W}})\notag\\
    &-\boldsymbol{\nu}^{\top}\left(\tilde{\mathbf{W}}^{\top}\mathbf{c}-\sqrt{dp}\mathbf{m}\right)
\end{align}
where $\tilde{\mathbf{Z}}$ is an independent copy of $\mathbf{Z}$, see \citep{bayati2011dynamics} Lemma 11. Then
\begin{align}
    \frac{\sqrt{\kappa}}{\sqrt{p}}\mathbf{Z}\mathbf{P}_{\tilde{\mathbf{c}}}\Omega^{1/2}\tilde{\mathbf{W}} &= \frac{\sqrt{\kappa}}{\sqrt{p}}\tilde{\mathbf{s}}\frac{\mathbf{c}^{\top}\tilde{\mathbf{W}}}{\norm{\tilde{\mathbf{c}}}_{2}} \\
    &= \sqrt{\kappa}\tilde{\mathbf{s}}\frac{\mathbf{c}^{\top}\tilde{\mathbf{W}}}{p\sqrt{\rho_{\tilde{\mathbf{c}}}}} \\
    &= \tilde{\mathbf{s}}\frac{\sqrt{\gamma\kappa}\mathbf{m}^{\top}}{\sqrt{\rho_{\tilde{\mathbf{c}}}}} 
\end{align}
where $\tilde{\mathbf{s}} = \mathbf{Z}\frac{\tilde{\mathbf{c}}}{\norm{\tilde{\mathbf{c}}}_{2}}$ is an i.i.d. standard normal vector and $\rho_{\tilde{\mathbf{c}}} = \frac{1}{p}\norm{\tilde{\mathbf{c}}}_{2}^{2}$ such that the optimization problem becomes
\begin{align}
    \inf_{\mathbf{m},\tilde{\mathbf{W}}}\sup_{\boldsymbol{\nu}}\mathcal{L}\left(f_{0}\left(\sqrt{\rho_{\tilde{\mathbf{w}}_{0}}}\mathbf{s}\right), \mathbf{s}\frac{\mathbf{m}^{\top}}{\sqrt{\rho_{\tilde{\mathbf{w}}_{0}}}}+\tilde{\mathbf{s}}\frac{\sqrt{\gamma\kappa}\mathbf{m}^{\top}}{\sqrt{\rho_{\tilde{\mathbf{c}}}}}+\frac{1}{\sqrt{p}}\tilde{\mathbf{Z}}\mathbf{P}_{\tilde{\mathbf{c}}}^{\perp}\Omega^{1/2}\tilde{\mathbf{W}}\right)+r(\tilde{\mathbf{W}})-\boldsymbol{\nu}^{\top}\left(\tilde{\mathbf{W}}^{\top}\mathbf{c}-\sqrt{dp}\mathbf{m}\right)
\end{align}
Now let $\mathbf{U} = \mathbf{P}_{\tilde{\mathbf{c}}}^{\perp}\Omega^{1/2}\tilde{\mathbf{W}}$, such that $\tilde{\mathbf{W}} = \Omega^{-1/2}\left(\frac{\sqrt{\gamma}\tilde{\mathbf{c}}\mathbf{m}^{\top}}{\rho_{\tilde{\mathbf{c}}}}+\mathbf{U}\right)$. The equivalent problem in $\mathbf{U}$ reads
\begin{align}
    \label{eq:AMP_target}
   \inf_{\mathbf{m},\mathbf{U}}\sup_{\boldsymbol{\nu}} \mathcal{L}\left(f_{0}\left(\sqrt{\rho_{\tilde{\mathbf{w}}_{0}}}\mathbf{s}\right), \mathbf{s}\frac{\mathbf{m}^{\top}}{\sqrt{\rho_{\tilde{\mathbf{w}}_{0}}}}+\tilde{\mathbf{s}}\frac{\sqrt{\gamma\kappa}\mathbf{m}^{\top}}{\sqrt{\rho_{\tilde{\mathbf{c}}}}}+\frac{1}{\sqrt{p}}\tilde{\mathbf{Z}}\mathbf{U}\right)+r(\Omega^{-1/2}\left(\frac{\sqrt{\gamma}\tilde{\mathbf{c}}\mathbf{m}^{\top}}{\rho_{\tilde{\mathbf{c}}}}+\mathbf{U}\right))-\boldsymbol{\nu}^{\top}\mathbf{U}^{\top}\tilde{\mathbf{c}}
\end{align}
Note that the constraint defining $\mathbf{m}$ automatically enforces the orthogonality constraint on $\mathbf{U}$ w.r.t. $\tilde{\mathbf{c}}$. The following lemma characterizes properties of the feasibility sets of $\mathbf{U}, \mathbf{m}, \boldsymbol{\nu}$.
\begin{lemma}
\label{lemma_compact}
Consider the optimization problem Eq.\eqref{eq:AMP_target}. Then there exist constants $C_{\mathbf{U}}, C_{\mathbf{m}}, C_{\boldsymbol{\nu}}$ such that
\begin{equation}
    \frac{1}{\sqrt{p}}\norm{\mathbf{U}}_{F}\leqslant C_{\mathbf{U}}, \quad \norm{\mathbf{m}}_{2} \leqslant C_{\mathbf{m}}, \quad \norm{\boldsymbol{\nu}}_{2}\leqslant C_{\boldsymbol{\nu}}
\end{equation}
with high probability as $n,p,d\to \infty$. 
\end{lemma}
\begin{proof}
Consider the optimization problem defining $\tilde{\mathbf{W}}^{*}$
\begin{align}
    \tilde{\mathbf{W}}^{*} \in \underset{{\tilde{\mathbf{W}}\in \mathbb{R}^{Kp \times K}}}{\arg\min} \mathcal{L}(\mathbf{y},\mathbf{X}\tilde{\mathbf{W}})+\tilde{r}_{0}(\mathbf{W})+\frac{\lambda_{2}}{2}\norm{\tilde{\mathbf{W}}}_{F}^{2}
\end{align}
which, owing to the convexity of the cost function, verifies
\begin{equation}
    \frac{1}{p}\left(\mathcal{L}(\mathbf{y},\mathbf{X}\tilde{\mathbf{W}}^{*})+\tilde{r}_{0}(\tilde{\mathbf{W}}^{*})+\frac{\lambda_{2}}{2}\norm{\tilde{\mathbf{W}}^{*}}_{F}^{2}\right) \leqslant \frac{1}{p}\left(\mathcal{L}(\mathbf{y},0)+\tilde{r}_{0}(0)\right)
\end{equation}
The functions $\mathcal{L}$ and $\tilde{r}_{0}$ are assumed to be proper, thus their sum is bounded below for any value of their arguments and we may write 
\begin{equation}
    \frac{1}{p}\frac{\lambda_{2}}{2}\norm{\tilde{\mathbf{W}}^{*}}_{F}^{2} \leqslant \frac{1}{p}\left(\mathcal{L}(\mathbf{y},0)+\tilde{r}_{0}(0)\right)
\end{equation}
The pseudo-Lipschitz assumption on $\mathcal{L}$ and $\tilde{r}_{0}$ then implies that there exist positive constants $C_{\mathcal{L}}$ and $C_{\tilde{r}_{0}}$ such that 
\begin{align}
    \frac{1}{p}\frac{\lambda_{2}}{2}\norm{\tilde{\mathbf{W}}^{*}}_{F}^{2} &\leqslant \frac{1}{p}\left(C_{\mathcal{L}}\left(1+\norm{\mathbf{y}}^{2}_{2}\right)\right)+C_{\tilde{r}_{0}} \\
    &\leqslant \frac{1}{p}\left(C_{\mathcal{L}}\left(1+C_{f_{0}}\norm{\frac{1}{\sqrt{d}}\mathbf{X}_{0}\mathbf{w}_{0}}_{2}^{2}+C_{f_{0}}\norm{\boldsymbol{\epsilon}_{0}^{2}}\right)\right)+C_{\tilde{r}_{0}}
\end{align}
where the second line follows from the scaling assumption on the teacher function $f_{0}$. Hence
\begin{equation}
    \frac{1}{p}\frac{\lambda_{2}}{2}\norm{\tilde{\mathbf{W}}^{*}}_{F}^{2} \leqslant C_{\mathcal{L}}\left(1+C_{f_{0}}\norm{\frac{1}{\sqrt{d}}\mathbf{A}}^{2}_{op}\norm{\Psi^{1/2}}^{2}_{op}\frac{\gamma}{d}\norm{\mathbf{w}}^{2}_{0}+C_{f_{0}}\frac{\alpha}{n}\norm{\boldsymbol{\epsilon}_{0}}_{2}^{2}\right)+C_{\tilde{r}_{0}}
\end{equation}
where $\norm{\bullet}_{op}$ denotes the operator norm of a given matrix, and we remind that $\mathbf{A}$ has i.i.d. $\mathcal{N}(0,1)$ elements. By assumption the maximum singular value of $\Psi^{1/2}$ is bounded. The maximum singular value of a random matrix with i.i.d. $\mathcal{N}(0,\frac{1}{d})$ elements is bounded with high probability as $n,p,d \to \infty$, see e.g., \citep{vershynin2010introduction}. Finally, $\mathbf{w}_{0}$ and $\boldsymbol{\epsilon}_{0}$ are sampled from subgaussian probability distributions, thus their scaled norms are bounded with high probability as $n,p,d \to \infty$ according to Bernstein's inequality, see e.g., \citep{vershynin2018high}. An application of the union bound then leads to the following statement:
there exists a constant $C_{\tilde{\mathbf{W}}}$ such that $\frac{1}{p}\norm{\tilde{\mathbf{W}}}_{2}^{2}\leqslant C_{\tilde{\mathbf{W}}}$, with high probability as $n,p,d \to \infty$. Now using the definition of $\mathbf{U}$
\begin{align}
    \frac{1}{p}\norm{\mathbf{U}}_{F}^{2} &= \frac{1}{p}\norm{\mathbf{P}_{\tilde{\mathbf{c}}}^{\perp}\Omega^{1/2}\tilde{\mathbf{W}}}_{F}^{2} \\
    & \leqslant \norm{\mathbf{P}_{\tilde{\mathbf{c}}}^{\perp}}^{2}_{op}\norm{\Omega^{1/2}}^{2}_{op}\frac{1}{p}\norm{\tilde{\mathbf{W}}}_{F}
\end{align} 
where the singular values of $\mathbf{P}_{\tilde{\mathbf{c}}}^{\perp}$ and $\Omega^{1/2}$ are bounded with probability one. Therefore there exists a constant $C_{\mathbf{U}}$ such that $\frac{1}{\sqrt{p}}\norm{\mathbf{U}} \leqslant C_{\mathbf{U}}$ with high probability as $n,p,d \to \infty$. Then, by definition of $\mathbf{m}$ and the Cauchy-Schwarz inequality
\begin{align}
    \norm{\mathbf{m}}_{2}^{2} &\leqslant \frac{1}{d}\norm{\mathbf{c}}_{2}^{2}\frac{1}{p}\norm{\tilde{\mathbf{W}}}_{F}^{2} \\
    &\leqslant \norm{\Phi}^{2}_{op}\frac{1}{d}\norm{\mathbf{w}_{0}}_{2}^{2}\frac{1}{p}\norm{\tilde{\mathbf{W}}}_{F}^{2} 
\end{align}
combining the results previously established on $\tilde{\mathbf{W}}$ and $\mathbf{w}_{0}$ with the fact that the maximum singular value of $\Phi$ is bounded, there exists a positive constant $C_{\mathbf{m}}$ such that $\norm{\mathbf{m}}_{2} \leqslant C_{\mathbf{m}}$ with high probability as $n,p,d \to \infty$. We finally turn to $\boldsymbol{\nu}$. The optimality condition for $\mathbf{m}$ in problem Eq.\eqref{eq_inter1} gives
\begin{equation}
    \boldsymbol{\nu} = -\frac{1}{\sqrt{dp}}\frac{\mathbf{s}^\top}{\sqrt{\rho_{\tilde{\mathbf{w}}_{0}}}}\partial \mathcal{L}\left(\mathbf{y},\frac{\mathbf{s}\mathbf{m}^\top}{\sqrt{\rho_{\tilde{\mathbf{w}}_{0}}}}+\frac{1}{\sqrt{p}}\mathbf{Z}\mathbf{C}^{1/2}\tilde{\mathbf{W}}^{*}\right)
\end{equation}
The pseudo-Lipschtiz assumption on $\mathcal{L}$ implies that we can find a constant $\mathcal{C}_{\partial \mathcal{L}}$ such that 
\begin{equation}
    \norm{\boldsymbol{\nu}}_{2}^{2} = \frac{1}{dp}\frac{\norm{\mathbf{s}}_{2}^{2}}{\rho_{\tilde{\mathbf{w}}_{0}}}C_{\mathcal{L}}\left(1+\norm{\mathbf{y}}_{2}^{2}+\norm{\frac{\mathbf{s}\mathbf{m}^\top}{\sqrt{\rho_{\tilde{\mathbf{w}}_{0}}}}+\frac{1}{\sqrt{p}}\mathbf{Z}\mathbf{C}^{1/2}\tilde{\mathbf{W}}^{*}}_{2}^{2}\right)
\end{equation}
the last bound then follows from similar arguments as those employed above.
\end{proof}
The optimization problem Eq.\eqref{eq:AMP_target} is convex and feasible. Furthermore, we may reduce the feasibility sets of $\mathbf{m},\boldsymbol{\nu}$ to compact spaces, and the function of $\mathbf{U}$ is coercive and thus has bounded lower level sets. Strong duality then implies we can invert the order of minimization to obtain the equivalent problem
\begin{align}
    \label{eq:AMP_target_inv}
   \inf_{\mathbf{m}}\sup_{\boldsymbol{\nu}} \inf_{\mathbf{U}} \mathcal{L}\left(f_{0}\left(\sqrt{\rho_{\tilde{\mathbf{w}}_{0}}}\mathbf{s}\right), \mathbf{s}\frac{\mathbf{m}^{\top}}{\sqrt{\rho_{\tilde{\mathbf{w}}_{0}}}}+\tilde{\mathbf{s}}\frac{\sqrt{\gamma\kappa}\mathbf{m}^{\top}}{\sqrt{\rho_{\tilde{\mathbf{c}}}}}+\frac{1}{\sqrt{p}}\tilde{\mathbf{Z}}\mathbf{U}\right)+r(\Omega^{-1/2}\left(\frac{\sqrt{\gamma}\tilde{\mathbf{c}}\mathbf{m}^{\top}}{\rho_{\tilde{\mathbf{c}}}}+\mathbf{U}\right))-\boldsymbol{\nu}^{\top}\mathbf{U}^{\top}\tilde{\mathbf{c}}
\end{align}
and study the optimization problem in $\mathbf{U}$ at fixed $\mathbf{m},\boldsymbol{\nu}$:
\begin{align}
    \label{student1}
    \inf_{\mathbf{U} \in \mathbb{R}^{Kp \times K}} \tilde{\mathcal{L}}(\frac{1}{\sqrt{p}}\tilde{\mathbf{Z}}\mathbf{U})+\tilde{r}(\mathbf{U})
\end{align}
where we defined the functions
\begin{align}
    \tilde{\mathcal{L}} : \mathbb{R}^{n \times K} &\to \mathbb{R} \\
    \frac{1}{\sqrt{p}}\tilde{\mathbf{Z}}\mathbf{U} &\to \mathcal{L}\left(f_{0}\left(\sqrt{\rho_{\tilde{\mathbf{w}}_{0}}}\mathbf{s}\right), \mathbf{s}\frac{\mathbf{m}^{\top}}{\sqrt{\rho_{\tilde{\mathbf{w}}_{0}}}}+\tilde{\mathbf{s}}\frac{\sqrt{\gamma\kappa}\mathbf{m}^{\top}}{\sqrt{\rho_{\tilde{\mathbf{c}}}}}+\frac{1}{\sqrt{p}}\tilde{\mathbf{Z}}\mathbf{U}\right) \\
    \tilde{r} : \mathbb{R}^{Kp \times K} &\to \mathbb{R} \\
    \mathbf{U} &\to r(\Omega^{-1/2}\left(\frac{\sqrt{\gamma}\tilde{\mathbf{c}}\mathbf{m}^{\top}}{\rho_{\tilde{\mathbf{c}}}}+\mathbf{U}\right))-\boldsymbol{\nu}^{\top}\mathbf{U}^{\top}\tilde{\mathbf{c}}
\end{align}
and the random matrix $\tilde{\mathbf{Z}}$ with i.i.d. $\mathcal{N}(0,1)$ elements is independent from all other random quantities in the problem. The asymptotic properties of the unique solution to this optimization problem can now be studied with a non-separable, matrix-valued approximate message passing iteration. The AMP iteration solving problem Eq.\eqref{student1} is given in the following lemma
\begin{lemma}
Consider the following AMP iteration
\begin{align}
    \label{eq:AMP1}
	&\hspace{1cm} \bu^{t+1} = \tilde{\bZ}^{\top}\bh_{t}(\bv^{t})-\be_{t}(\bu^{t})\langle \bh_{t}'\rangle^\top \\
	&\hspace{1cm} \bv^{t} = \tilde{\bZ}\be_{t}(\bu^{t})-\bh_{t-1}(\bv^{t-1})\langle \be_{t}'\rangle^\top
\end{align}
where for any $t \in \mathbb{N}$
\begin{align}
    &\bh_{t}(\mathbf{v}^{t}) = \left(\bR_{\mathcal{L}(\mathbf{y},.),\mathbf{S}^{t}}(\mathbf{s}\frac{\mathbf{m}^{\top}}{\sqrt{\rho_{\tilde{\mathbf{w}}_{0}}}}+\tilde{\mathbf{s}}\frac{\sqrt{\gamma\kappa}\mathbf{m}^{\top}}{\sqrt{\rho_{\tilde{\mathbf{c}}}}}+\mathbf{v}^{t})-\left(\mathbf{s}\frac{\mathbf{m}^{\top}}{\sqrt{\rho_{\tilde{\mathbf{w}}_{0}}}}+\tilde{\mathbf{s}}\frac{\sqrt{\gamma\kappa}\mathbf{m}^{\top}}{\sqrt{\rho_{\tilde{\mathbf{c}}}}}+\mathbf{v}^{t}\right)\right)(\mathbf{S}^{t})^{-1} \\
    &\be_{t}(\mathbf{u}^{t}) = \bR_{r(\Omega^{-1/2}.),\hat{\mathbf{S}}^{t}}\left(\mathbf{u}^{t}\hat{\mathbf{S}}^{t}+\Omega^{-1/2}\mathbf{c}\boldsymbol{\nu}^{\top}\hat{\mathbf{S}}^{t}+\frac{\sqrt{\gamma}\tilde{\mathbf{c}}\mathbf{m}^{\top}}{\rho_{\tilde{\mathbf{c}}}}\right)-\frac{\sqrt{\gamma}\tilde{\mathbf{c}}\mathbf{m}^{\top}}{\rho_{\tilde{\mathbf{c}}}} \\
    &\mbox{and} \quad \mathbf{S}^{t} = \langle (\mathbf{e}^{t})'\rangle^{\top}, \quad \hat{\mathbf{S}}^{t} = -\left(\langle (\mathbf{h}^{t})'\rangle^{\top} \right)^{-1}
    \label{eq:AMP2}
\end{align}
Then the fixed point $(\mathbf{u}^{\infty}, \mathbf{v}^{\infty})$ of this iteration verifies 
\begin{align}
    &\bR_{r(\Omega^{-1/2}.),\hat{\mathbf{S}}^{\infty}}\left(\mathbf{u}^{\infty}\hat{\mathbf{S}}^{\infty}+\Omega^{-1/2}\mathbf{c}\boldsymbol{\nu}^{\top}\hat{\mathbf{S}}^{\infty}+\frac{\sqrt{\gamma}\tilde{\mathbf{c}}\mathbf{m}^{\top}}{\rho_{\tilde{\mathbf{c}}}}\right)-\frac{\sqrt{\gamma}\tilde{\mathbf{c}}\mathbf{m}^{\top}}{\rho_{\tilde{\mathbf{c}}}} = \mathbf{U}^{*}\\
    &\bR_{\mathcal{L}(\mathbf{y},.),\mathbf{S}^{\infty}}(\mathbf{s}\frac{\mathbf{m}^{\top}}{\sqrt{\rho_{\tilde{\mathbf{w}}_{0}}}}+\tilde{\mathbf{s}}\frac{\sqrt{\gamma\kappa}\mathbf{m}^{\top}}{\sqrt{\rho_{\tilde{\mathbf{c}}}}}+\mathbf{v}^{\infty})-\mathbf{s}\frac{\mathbf{m}^{\top}}{\sqrt{\rho_{\tilde{\mathbf{w}}_{0}}}}+\tilde{\mathbf{s}}\frac{\sqrt{\gamma\kappa}\mathbf{m}^{\top}}{\sqrt{\rho_{\tilde{\mathbf{c}}}}} = \tilde{\mathbf{Z}}\mathbf{U}^{*}
\end{align}
where $\mathbf{U}^{*}$ is the unique solution to the optimization problem Eq.\eqref{student1}.
\end{lemma}

\begin{proof}
To find the correct form of the non-linearities in the AMP iteration, we match the optimality condition of problem Eq.\eqref{student1} with the generic form of the fixed point of the AMP iteration Eq.\eqref{canon_AMP}.
In the subsequent derivation, we absorb the scaling $\frac{1}{\sqrt{d}}$ in the matrix $\tilde{\bZ}$, such that its elements are i.i.d. $\mathcal{N}(0,1/d)$, and omit time indices for simplicity.
Going back to problem Eq.~\eqref{student1}, its optimality condition reads :
\begin{align}
	&\tilde{\bZ}^\top \partial \tilde{\mathcal{L}}(\tilde{\bZ}\bU)+\partial\tilde{r}(\bU) = 0 
\end{align}

For any pair of $K \times K$ symmetric positive definite matrices $\mathbf{S}, \hat{\mathbf{S}}$, this optimality condition is equivalent to 
\begin{equation}
\label{eq:inter_opt}
    \tilde{\bZ}^\top\left(\partial \tilde{\mathcal{L}}(\tilde{\bZ}\bU)\mathbf{S}+\tilde{\mathbf{Z}}\mathbf{U}\right)\mathbf{S}^{-1} +\left(\partial\tilde{r}(\bU)\hat{\mathbf{S}}+\mathbf{U}\right)\hat{\mathbf{S}}^{-1}= \tilde{\mathbf{Z}}^{\top}\tilde{\mathbf{Z}}\mathbf{U}\mathbf{S}^{-1}+\mathbf{U}\hat{\mathbf{S}}^{-1}
\end{equation}
where we added the same quantity on both sides of the equality.
For the loss function, we can then introduce the resolvent, formally D-resolvent: 
\begin{equation}
	\hat{\bv} = \partial \tilde{\mathcal{L}}(\tilde{\bZ}\bU)\bS+\bZ\bU \iff \tilde{\bZ}\bU = \bR_{\tilde{\mathcal{L}},\bS}(\hat{\bv})
\end{equation}
such that
\begin{equation}
\label{block_res_loss}
\begin{split}
	\bR_{\tilde{\mathcal{L}},\bS}(\hat{\bv}) &= (\mathrm{Id}+\partial \tilde{\mathcal{L}}(\bullet)\bS)^{-1}(\hat{\bv}) = \underset{\bT \in \mathbb{R}^{n \times K}}{\arg\min}\left\{\tilde{\mathcal{L}}(\bT)+\frac{1}{2}\mathrm{tr}\left((
        \bT-\hat{\bv})\bS^{-1}(\bT-\hat{\bv})^\top \right)\right\} \\
\end{split}
\end{equation}
Similarly for the regularisation, introduce
\begin{equation}
\label{eq:res_reg}
\hat{\mathbf{u}} \equiv \left(\mathrm{Id}+\partial \tilde{r}(\bullet)\bhS\right)(\bU) \qquad \bU = \bR_{\tilde{r},\bhS}(\hat{\mathbf{u}})
\end{equation}
where $\mathbf{S} \in \mathbb{R}^{K \times K}$ is a positive definite matrix, and 
\begin{equation}
\bR_{\tilde{r},\bhS}(\hat{\mathbf{v}}) = \left(\mathrm{Id}+\partial \tilde{r}(\bullet)\bhS\right)^{-1}(\hat{\mathbf{v}})= \underset{\bT \in \mathbb{R}^{Kp \times K}}{\arg\min}\left\{\tilde{r}(\bT)+\frac{1}{2}\mathrm{tr}\left((
	\bT-\hat{\mathbf{v}})\bhS^{-1}(\bT-\hat{\mathbf{v}})^\top \right)\right\}
\end{equation}
where $\bhS \in \mathbb{R}^{K\times K}$ is a positive definite matrix, and $\hat{\mathbf{v}}\in \mathbb{R}^{d\times K}$.
The optimality condition Eq.\eqref{eq:inter_opt} may then be rewritten as:
\begin{align}
\label{prox-opti}
	\tilde{\bZ}^\top \left(\bR_{\tilde{\mathcal{L}},\bS}(\hat{\mathbf{v}})-\hat{\mathbf{v}}\right)\bS^{-1} &= (\hat{\mathbf{u}}-\bR_{\tilde{r},\bhS}(\hat{\mathbf{u}}))\bhS^{-1} \\
	\tilde{\bZ}\bR_{\tilde{r},\bhS}(\hat{\mathbf{u}}) &= \bR_{\tilde{\mathcal{L}},\bS}(\hat{\mathbf{v}})
\end{align}
where both equations should be satisfied. We can now define update functions based on the previously obtained block decomposition. The fixed point of the matrix-valued AMP Eq.(\ref{canon_AMP}), omitting the time indices for simplicity, reads:
\begin{align}
	\bu+\be(\bu)\langle \bh' \rangle^\top  &= \tilde{\bZ}^\top \bh(\bv) \\
	\bv+\bh(\bv)\langle \be' \rangle^\top  &= \tilde{\bZ}\be(\bu) 
\end{align}
Matching this fixed point with the optimality condition Eq.(\ref{prox-opti}) suggests the following mapping:
\begin{equation}
\begin{split}
\bh(\mathbf{v}) &= \left(\bR_{\tilde{\mathcal{L}},\bS}(\mathbf{v})-\mathbf{v}\right)\bS^{-1}, \\
\be(\mathbf{u}) &= \bR_{\tilde{r},\bhS}(\mathbf{u}\bhS),
\end{split}\qquad
\begin{split}
\bS &= \langle \be' \rangle^{\top},\\
\bhS &= -(\langle \bh '\rangle^{\top})^{-1},
\end{split}
\end{equation}
where we redefined $\hat{\mathbf{u}}\equiv \hat{\mathbf{u}}\bhS$ in \eqref{eq:res_reg}. We are now left with the task of evaluating the resolvents of $\tilde{\mathcal{L}}, \tilde{r}$ as expressions of the original functions $\mathcal{L},r$. Starting with the loss function, we get 
\begin{align}
    &\bR_{\tilde{\mathcal{L}},\bS}(\mathbf{v}) = \underset{\mathbf{x} \in \mathbb{R}^{n \times K}}{\arg\min} \left\{\mathcal{L}\left(f_{0}\left(\sqrt{\rho_{\tilde{\mathbf{w}}_{0}}}\mathbf{s}\right), \mathbf{s}\frac{\mathbf{m}^{\top}}{\sqrt{\rho_{\tilde{\mathbf{w}}_{0}}}}+\tilde{\mathbf{s}}\frac{\sqrt{\gamma\kappa}\mathbf{m}^{\top}}{\sqrt{\rho_{\tilde{\mathbf{c}}}}}+\mathbf{x}\right)+\frac{1}{2}\mbox{tr}\left((\mathbf{x}-\mathbf{v})\mathbf{S}^{-1}(\mathbf{x}-\mathbf{v})\right)^{\top}\right\}
\end{align}
letting $\tilde{\mathbf{x}} = \mathbf{s}\frac{\mathbf{m}^{\top}}{\sqrt{\rho_{\tilde{\mathbf{w}}_{0}}}}+\tilde{\mathbf{s}}\frac{\sqrt{\gamma\kappa}\mathbf{m}^{\top}}{\sqrt{\rho_{\tilde{\mathbf{c}}}}}+\mathbf{x}$, the problem is equivalent to 
\begin{align}
    &\bR_{\tilde{\mathcal{L}},\bS}(\mathbf{v}) = \underset{\tilde{\mathbf{x}} \in \mathbb{R}^{n \times K}}{\arg\min} \bigg\{\mathcal{L}\left(f_{0}\left(\sqrt{\rho_{\tilde{\mathbf{w}}_{0}}}\mathbf{s}\right), \tilde{\mathbf{x}}\right) \notag \\
    &+\frac{1}{2}\mbox{tr}\left((\tilde{\mathbf{x}}-(\mathbf{s}\frac{\mathbf{m}^{\top}}{\sqrt{\rho_{\tilde{\mathbf{w}}_{0}}}}+\tilde{\mathbf{s}}\frac{\sqrt{\gamma\kappa}\mathbf{m}^{\top}}{\sqrt{\rho_{\tilde{\mathbf{c}}}}}+\mathbf{v}))\mathbf{S}^{-1}(\tilde{\mathbf{x}}-(\mathbf{s}\frac{\mathbf{m}^{\top}}{\sqrt{\rho_{\tilde{\mathbf{w}}_{0}}}}+\tilde{\mathbf{s}}\frac{\sqrt{\gamma\kappa}\mathbf{m}^{\top}}{\sqrt{\rho_{\tilde{\mathbf{c}}}}}+\mathbf{v}))^{\top}\right)\bigg\} \notag \\
    &-\mathbf{s}\frac{\mathbf{m}^{\top}}{\sqrt{\rho_{\tilde{\mathbf{w}}_{0}}}}-\tilde{\mathbf{s}}\frac{\sqrt{\gamma\kappa}\mathbf{m}^{\top}}{\sqrt{\rho_{\tilde{\mathbf{c}}}}} \\
    & = \bR_{\mathcal{L}(\mathbf{y},.),\mathbf{S}}(\mathbf{s}\frac{\mathbf{m}^{\top}}{\sqrt{\rho_{\tilde{\mathbf{w}}_{0}}}}+\tilde{\mathbf{s}}\frac{\sqrt{\gamma\kappa}\mathbf{m}^{\top}}{\sqrt{\rho_{\tilde{\mathbf{c}}}}}+\mathbf{v})-\mathbf{s}\frac{\mathbf{m}^{\top}}{\sqrt{\rho_{\tilde{\mathbf{w}}_{0}}}}-\tilde{\mathbf{s}}\frac{\sqrt{\gamma\kappa}\mathbf{m}^{\top}}{\sqrt{\rho_{\tilde{\mathbf{c}}}}}
\end{align}
and the corresponding non-linearity will then be 
\begin{equation}
    \bh(\mathbf{v}) = \left(R_{\mathcal{L}(\mathbf{y},.),\mathbf{S}}(\mathbf{s}\frac{\mathbf{m}^{\top}}{\sqrt{\rho_{\tilde{\mathbf{w}}_{0}}}}+\tilde{\mathbf{s}}\frac{\sqrt{\gamma\kappa}\mathbf{m}^{\top}}{\sqrt{\rho_{\tilde{\mathbf{c}}}}}+\mathbf{v})-\left(\mathbf{s}\frac{\mathbf{m}^{\top}}{\sqrt{\rho_{\tilde{\mathbf{w}}_{0}}}}+\tilde{\mathbf{s}}\frac{\sqrt{\gamma\kappa}\mathbf{m}^{\top}}{\sqrt{\rho_{\tilde{\mathbf{c}}}}}+\mathbf{v}\right)\right)\mathbf{S}^{-1}
\end{equation}
Moving to the regularization, the resolvent reads 
\begin{align}
    &\bR_{\tilde{r},\bhS}(\mathbf{u}) = \underset{\mathbf{x} \in \mathbb{R}^{Kp \times K}}{\arg\min} \bigg\{r\left(\Omega^{-1/2}\left(\frac{\sqrt{\gamma}\tilde{\mathbf{c}}\mathbf{m}^{\top}}{\rho_{\tilde{\mathbf{c}}}}+\mathbf{x}\right)\right)-\boldsymbol{\nu}^{\top}\mathbf{x}^{\top}\Omega^{-1/2}\mathbf{c}+\frac{1}{2}\mbox{tr}\left((\mathbf{x}-\mathbf{u})\hat{\mathbf{S}}^{-1}(\mathbf{x}-\mathbf{u})^{\top}\right)\bigg\}
\end{align}
letting $\tilde{\mathbf{x}} = \frac{\sqrt{\gamma}\tilde{\mathbf{c}}\mathbf{m}^{\top}}{\rho_{\tilde{\mathbf{c}}}}+\mathbf{x}$, we obtain
\begin{align}
    \bR_{\tilde{r},\bhS}(\mathbf{u}) &= \underset{\tilde{\mathbf{x}} \in \mathbb{R}^{Kp \times K}}{\arg\min} \bigg\{r\left(\Omega^{-1/2}\tilde{\mathbf{x}}\right)-\boldsymbol{\nu}^{\top}\tilde{\mathbf{x}}^{\top}\Omega^{-1/2}\mathbf{c}\\
    &\hspace{1cm}+\frac{1}{2}\mbox{tr}\left((\tilde{\mathbf{x}}-\left(\mathbf{u}+\frac{\sqrt{\gamma}\tilde{\mathbf{c}}\mathbf{m}^{\top}}{\rho_{\tilde{\mathbf{c}}}}\right))\hat{\mathbf{S}}^{-1}(\tilde{\mathbf{x}}-\left(\mathbf{u}+\frac{\sqrt{\gamma}\tilde{\mathbf{c}}\mathbf{m}^{\top}}{\rho_{\tilde{\mathbf{c}}}}\right))^{\top}\right)\bigg\}-\frac{\sqrt{\gamma}\tilde{\mathbf{c}}\mathbf{m}^{\top}}{\rho_{\tilde{\mathbf{c}}}} \\
    &=\underset{\tilde{\mathbf{x}} \in \mathbb{R}^{Kp \times K}}{\arg\min} \bigg\{r\left(\Omega^{-1/2}\tilde{\mathbf{x}}\right)\\
    &+\frac{1}{2}\mbox{tr}\left((\tilde{\mathbf{x}}-\left(\mathbf{u}+\Omega^{-1/2}\mathbf{c}\boldsymbol{\nu}^{\top}\hat{\mathbf{S}}+\frac{\sqrt{\gamma}\tilde{\mathbf{c}}\mathbf{m}^{\top}}{\rho_{\tilde{\mathbf{c}}}}\right))\hat{\mathbf{S}}^{-1}(\tilde{\mathbf{x}}-\left(\mathbf{u}+\Omega^{-1/2}\mathbf{c}\boldsymbol{\nu}^{\top}\hat{\mathbf{S}}+\frac{\sqrt{\gamma}\tilde{\mathbf{c}}\mathbf{m}^{\top}}{\rho_{\tilde{\mathbf{c}}}}\right))^{\top}\right)\bigg\}\\
    &\hspace{5cm}-\frac{\sqrt{\gamma}\tilde{\mathbf{c}}\mathbf{m}^{\top}}{\rho_{\tilde{\mathbf{c}}}} \\
    &\bR_{r(\Omega^{-1/2}.),\hat{\mathbf{S}}}\left(\mathbf{u}+\Omega^{-1/2}\mathbf{c}\boldsymbol{\nu}^{\top}\hat{\mathbf{S}}+\frac{\sqrt{\gamma}\tilde{\mathbf{c}}\mathbf{m}^{\top}}{\rho_{\tilde{\mathbf{c}}}}\right)-\frac{\sqrt{\gamma}\tilde{\mathbf{c}}\mathbf{m}^{\top}}{\rho_{\tilde{\mathbf{c}}}}
\end{align}
Which gives the following non-linearity for the AMP iteration
\begin{equation}
    e(\mathbf{u}) = \bR_{r(\Omega^{-1/2}.),\hat{\mathbf{S}}}\left(\mathbf{u}\hat{\mathbf{S}}+\Omega^{-1/2}\mathbf{c}\boldsymbol{\nu}^{\top}\mathbf{V}+\frac{\sqrt{\gamma}\tilde{\mathbf{c}}\mathbf{m}^{\top}}{\rho_{\tilde{\mathbf{c}}}}\right)-\frac{\sqrt{\gamma}\tilde{\mathbf{c}}\mathbf{m}^{\top}}{\rho_{\tilde{\mathbf{c}}}}
\end{equation}
\end{proof}
The following lemma then gives the exact asymptotics at each time step of the AMP iteration solving problem Eq.\eqref{student1} : its \emph{state evolution equations}.
\begin{lemma}
    \label{lemma:SE_inter}
Consider the AMP iteration Eq.(\ref{eq:AMP1}-\ref{eq:AMP2}). Assume it is initialized with $\mathbf{u}^{0}$ such that \\
 $\lim_{d \to \infty}\frac{1}{d}\norm{\be_{0}(\bu^{0})^\top\be_{0}(\bu^{0})}_{\rm F}$ exists, a positive 
definite matrix $\hat{\mathbf{S}}_{0}$, and $\bh_{-1}\equiv 0$. Then for any $t \in \mathbb{N}$, and any pair of seqeunces of uniformly pseudo-Lipschitz functions $\phi_{1,n} : \mathbb{R}^{Kp \times K}$ and $\phi_{2,n} : \mathbb{R}^{n \times K}$, the following holds
\begin{align}
    &\phi_{1,n}\left(\mathbf{u}^{t}\right) \stackrel{P}{\simeq}\mathbb{E}\left[\phi_{1,n}\left(\mathbf{G}(\hat{\mathbf{Q}}^{t})^{1/2}\right)\right]\\
    &\phi_{2,n}\left(\mathbf{v}^{t}\right) \stackrel{P}{\simeq}\mathbb{E}\left[\phi_{2,n}\left(\mathbf{H}(\mathbf{Q}^{t})^{1/2}\right)\right]
\end{align}
where $\mathbf{G} \in \mathbb{R}^{Kp \times K}$ and $\mathbf{H} \in \mathbb{R}^{n \times K}$ are independent random matrices with i.i.d. standard normal elements, and $\mathbf{Q}^{t}, \hat{\mathbf{Q}}^{t}, \mathbf{V}^{t}, \hat{\mathbf{V}}^{t}$ are given by the equations
\begin{align}
&\mathbf{Q}^{t} = \frac{1}{p}\mathbb{E}\bigg[\left(\bR_{r(\Omega^{-1/2}.),(\hat{\mathbf{V}}^{t})^{-1}}\left(\mathbf{G}(\hat{\mathbf{Q}}^{t})^{1/2}(\hat{\mathbf{V}}^{t})^{-1}+\Omega^{-1/2}\mathbf{c}\boldsymbol{\nu}^{\top}(\hat{\mathbf{V}}^{t})^{-1}+\frac{\sqrt{\gamma}\tilde{\mathbf{c}}\mathbf{m}^{\top}}{\rho_{\tilde{\mathbf{c}}}}\right)-\frac{\sqrt{\gamma}\tilde{\mathbf{c}}\mathbf{m}^{\top}}{\rho_{\tilde{\mathbf{c}}}}\right)^{\top} \notag \\
&\hspace{1cm}\left(\bR_{r(\Omega^{-1/2}.),(\hat{\mathbf{V}}^{t})^{-1}}\left(\mathbf{G}(\hat{\mathbf{Q}}^{t})^{1/2}(\hat{\mathbf{V}}^{t})^{-1}+\Omega^{-1/2}\mathbf{c}\boldsymbol{\nu}^{\top}(\hat{\mathbf{V}}^{t})^{-1}+\frac{\sqrt{\gamma}\tilde{\mathbf{c}}\mathbf{m}^{\top}}{\rho_{\tilde{\mathbf{c}}}}\right)-\frac{\sqrt{\gamma}\tilde{\mathbf{c}}\mathbf{m}^{\top}}{\rho_{\tilde{\mathbf{c}}}}\right)\bigg] \\
&\hat{\mathbf{Q}}^{t} = \frac{1}{p}\mathbb{E}\bigg[\left(\left(\bR_{\mathcal{L}(\mathbf{y},.),\mathbf{V}^{t-1}}(.)-Id\right)\left(\mathbf{s}\frac{\mathbf{m}^{\top}}{\sqrt{\rho_{\tilde{\mathbf{w}}_{0}}}}+\tilde{\mathbf{s}}\frac{\sqrt{\gamma\kappa}\mathbf{m}^{\top}}{\sqrt{\rho_{\tilde{\mathbf{c}}}}}+\mathbf{H}(\mathbf{Q}^{t-1})^{1/2}\right)(\mathbf{V}^{t-1})^{-1}\right)^{\top}\\
&\hspace{1cm}\left(\bR_{\mathcal{L}(\mathbf{y},.),\mathbf{V}^{t-1}}(.)-Id\right)\left(\mathbf{s}\frac{\mathbf{m}^{\top}}{\sqrt{\rho_{\tilde{\mathbf{w}}_{0}}}}+\tilde{\mathbf{s}}\frac{\sqrt{\gamma\kappa}\mathbf{m}^{\top}}{\sqrt{\rho_{\tilde{\mathbf{c}}}}}+\mathbf{H}(\mathbf{Q}^{t-1})^{1/2}\right)(\mathbf{V}^{t-1})^{-1}\bigg] \\
&\mathbf{V}^{t} = \frac{1}{p}\mathbb{E}\left[(\hat{\mathbf{Q}}^{t})^{-1/2}\mathbf{G}^{\top}R_{r(\Omega^{-1/2}.),(\hat{\mathbf{V}}^{t})^{-1}}\left(\mathbf{G}(\hat{\mathbf{Q}}^{t})^{1/2}(\hat{\mathbf{V}}^{t})^{-1}+\Omega^{-1/2}\mathbf{c}\boldsymbol{\nu}^{\top}(\hat{\mathbf{V}}^{t})^{-1}+\frac{\sqrt{\gamma}\tilde{\mathbf{c}}\mathbf{m}^{\top}}{\rho_{\tilde{\mathbf{c}}}}\right)\right] \\
&\hat{\mathbf{V}}^{t} =- \frac{1}{p}\mathbb{E}\bigg[(\mathbf{Q}^{t-1})^{-1/2}\mathbf{H}^{\top}\left(\left(\bR_{\mathcal{L}(\mathbf{y},.),\mathbf{V}^{t-1}}(.)-Id\right)\left(\mathbf{s}\frac{\mathbf{m}^{\top}}{\sqrt{\rho_{\tilde{\mathbf{w}}_{0}}}}+\tilde{\mathbf{s}}\frac{\sqrt{\gamma\kappa}\mathbf{m}^{\top}}{\sqrt{\rho_{\tilde{\mathbf{c}}}}}+\mathbf{H}(\mathbf{Q}^{t-1})^{1/2}\right)\right) \notag \\
&\hspace{1.5cm}(\mathbf{V}^{t-1})^{-1}\bigg]
\end{align}
\end{lemma}
\begin{proof}
    Owing to the properties of Bregman proximity operators \citep{bauschke2003bregman,bauschke2006joint}, the update functions in the AMP iteration 
    Eq.(\ref{eq:AMP1}-\ref{eq:AMP2}) are Lipschitz continuous. Thus under the assumptions made on the initialization, the assumptions of Theorem \ref{th:SE} are verified, 
    which gives the desired result.
\end{proof}
\begin{lemma}
\label{lemma:conv_traj}
Consider iteration Eq.(\ref{eq:AMP1}-\ref{eq:AMP2}), where the parameters $\bQ,\hat{\bQ},\bV,\hat{\bV}$ are initialized at any fixed point of the state evolution equations of Lemma \ref{lemma:SE_inter}. For any sequence initialized with $\bhV_{0} = \bhV$ and  $\bu_{0}$ such that
\begin{equation}
\lim_{d \to \infty}\frac{1}{d}\be_{0}({\bu_{0}})^{\top}\be_{0}(\bu_{0}) = \bQ
\end{equation}
the following holds
\begin{equation}
\lim_{t\to \infty}\lim_{p \to \infty}\frac{1}{\sqrt{p}}\norm{\bu^{t}-\bu^{\star}}_{\rm F} = 0 \quad \lim_{t\to \infty}\lim_{d \to \infty}\frac{1}{\sqrt{p}}\norm{\bv^{t}-\bv^{\star}}_{\rm F} = 0
\end{equation}
\end{lemma}
\begin{proof}
    The proof of this lemma is identical to that of Lemma 7 from \citep{loureiro2021learning}.
\end{proof}
Combining these results, we obtain the following asymptotic characterization of $\mathbf{U}^{*}$.
\begin{lemma}
\label{lemma:asy_inter}
For any fixed $\mathbf{m}$ and $\boldsymbol{\nu}$ in their feasibility sets, let $\mathbf{U}^{*}$ be the unique solution to the optimization problem Eq.\eqref{student1}. Then, for any sequences (in the problem dimension) of pseudo-Lipschitz functions of order 2 $\phi_{1,n} : \mathbb{R}^{n \times K} \to \mathbb{R}$ and $\phi_{2,n} : \mathbb{R}^{Kp \times K} \to \mathbb{R}$, the following holds 
\begin{align}
        &\phi_{1,n}\left(\mathbf{U}^{*}\right) \stackrel{P}{\simeq}\mathbb{E}\left[\phi_{1,n}\left( \bR_{r(\Omega^{-1/2}.),\hat{\mathbf{V}}^{-1}}\left(\mathbf{G}\hat{\mathbf{Q}}^{1/2}\hat{\mathbf{V}}^{-1}+\Omega^{-1/2}\mathbf{c}\boldsymbol{\nu}^{\top}\hat{\mathbf{V}}^{-1}+\frac{\sqrt{\gamma}\tilde{\mathbf{c}}\mathbf{m}^{\top}}{\rho_{\tilde{\mathbf{c}}}}\right)-\frac{\sqrt{\gamma}\tilde{\mathbf{c}}\mathbf{m}^{\top}}{\rho_{\tilde{\mathbf{c}}}}\right)\right]\\
        &\phi_{2,n}\left(\frac{1}{\sqrt{p}}\tilde{\mathbf{Z}}\mathbf{U}^{*}\right) \stackrel{P}{\simeq}\mathbb{E}\left[\phi_{2,n}\left( \bR_{\mathcal{L}(\mathbf{y},.),\mathbf{V}}(\mathbf{s}\frac{\mathbf{m}^{\top}}{\sqrt{\rho_{\tilde{\mathbf{w}}_{0}}}}+\tilde{\mathbf{s}}\frac{\sqrt{\gamma\kappa}\mathbf{m}^{\top}}{\sqrt{\rho_{\tilde{\mathbf{c}}}}}+\mathbf{H}\hat{\mathbf{Q}}^{1/2})-\mathbf{s}\frac{\mathbf{m}^{\top}}{\sqrt{\rho_{\tilde{\mathbf{w}}_{0}}}}-\tilde{\mathbf{s}}\frac{\sqrt{\gamma\kappa}\mathbf{m}^{\top}}{\sqrt{\rho_{\tilde{\mathbf{c}}}}}\right)\right]
\end{align}
where $\mathbf{G} \in \mathbb{R}^{Kp \times K}$ and $\mathbf{H} \in \mathbb{R}^{n \times K}$ are independent random matrices with i.i.d. standard normal elements, and $\mathbf{Q}, \hat{\mathbf{Q}}, \mathbf{V}, \hat{\mathbf{V}}$ are given by the fixed point (assumed to be unique) of the following set of self consistent equations
\begin{align}
    &\mathbf{Q} = \frac{1}{p}\mathbb{E}\bigg[\left(\bR_{r(\Omega^{-1/2}.),\hat{\mathbf{V}}^{-1}}\left(\mathbf{G}\hat{\mathbf{Q}}^{1/2}\hat{\mathbf{V}}^{-1}+\Omega^{-1/2}\mathbf{c}\boldsymbol{\nu}^{\top}\hat{\mathbf{V}}^{-1}+\frac{\sqrt{\gamma}\tilde{\mathbf{c}}\mathbf{m}^{\top}}{\rho_{\tilde{\mathbf{c}}}}\right)-\frac{\sqrt{\gamma}\tilde{\mathbf{c}}\mathbf{m}^{\top}}{\rho_{\tilde{\mathbf{c}}}}\right)^{\top}\\
    &\hspace{3cm}\left(\bR_{r(\Omega^{-1/2}.),\hat{\mathbf{V}}^{-1}}\left(\mathbf{G}\hat{\mathbf{Q}}^{1/2}\hat{\mathbf{V}}^{-1}+\Omega^{-1/2}\mathbf{c}\boldsymbol{\nu}^{\top}\hat{\mathbf{V}}^{-1}+\frac{\sqrt{\gamma}\tilde{\mathbf{c}}\mathbf{m}^{\top}}{\rho_{\tilde{\mathbf{c}}}}\right)-\frac{\sqrt{\gamma}\tilde{\mathbf{c}}\mathbf{m}^{\top}}{\rho_{\tilde{\mathbf{c}}}}\right)\bigg] \\
    &\hat{\mathbf{Q}} = \frac{1}{p}\mathbb{E}\bigg[\left(\left(\bR_{\mathcal{L}(\mathbf{y},.),\mathbf{V}}(.)-Id\right)\left(\mathbf{s}\frac{\mathbf{m}^{\top}}{\sqrt{\rho_{\tilde{\mathbf{w}}_{0}}}}+\tilde{\mathbf{s}}\frac{\sqrt{\gamma\kappa}\mathbf{m}^{\top}}{\sqrt{\rho_{\tilde{\mathbf{c}}}}}+\mathbf{H}\mathbf{Q}^{1/2}\right)\mathbf{V}^{-1}\right)^{\top} \notag \\
    &\hspace{3cm}\left(\left(\bR_{\mathcal{L}(\mathbf{y},.),\mathbf{V}}(.)-Id\right)\left(\mathbf{s}\frac{\mathbf{m}^{\top}}{\sqrt{\rho_{\tilde{\mathbf{w}}_{0}}}}+\tilde{\mathbf{s}}\frac{\sqrt{\gamma\kappa}\mathbf{m}^{\top}}{\sqrt{\rho_{\tilde{\mathbf{c}}}}}+\mathbf{H}\mathbf{Q}^{1/2}\right)\mathbf{V}^{-1}\right)\bigg] \\
    &\mathbf{V} = \frac{1}{p}\mathbb{E}\left[\hat{\mathbf{Q}}^{-1/2}\mathbf{G}^{\top}\bR_{r(\Omega^{-1/2}.),\hat{\mathbf{V}}^{-1}}\left(\mathbf{G}\hat{\mathbf{Q}}^{1/2}\hat{\mathbf{V}}^{-1}+\Omega^{-1/2}\mathbf{c}\boldsymbol{\nu}^{\top}\hat{\mathbf{V}}^{-1}+\frac{\sqrt{\gamma}\tilde{\mathbf{c}}\mathbf{m}^{\top}}{\rho_{\tilde{\mathbf{c}}}}\right)\right] \\
    &\hat{\mathbf{V}} =- \frac{1}{p}\mathbb{E}\left[\mathbf{Q}^{-1/2}\mathbf{H}^{\top}\left(\left(\bR_{\mathcal{L}(\mathbf{y},.),\mathbf{V}}(.)-Id\right)\left(\mathbf{s}\frac{\mathbf{m}^{\top}}{\sqrt{\rho_{\tilde{\mathbf{w}}_{0}}}}+\tilde{\mathbf{s}}\frac{\sqrt{\gamma\kappa}\mathbf{m}^{\top}}{\sqrt{\rho_{\tilde{\mathbf{c}}}}}+\mathbf{H}\mathbf{Q}^{1/2}\right)\mathbf{V}^{-1}\right)\right]
\end{align}
\end{lemma}
\begin{proof}
    Combining the results of the previous lemmas, this proof is close to that of Theorem 1.5 in \citep{bayati2011lasso}.
\end{proof}
Returning to the optimization problem on $\mathbf{m}, \boldsymbol{\nu}$ in Eq.\eqref{eq:AMP_target_inv}, the solution $\mathbf{U}^{*}$, at any dimension, verifies the zero gradient conditions on $\mathbf{m}, \boldsymbol{\nu}$:
\begin{align}
    \label{eq:subgrad_Lip}
    &\partial \boldsymbol{\nu} = 0 \iff (\mathbf{U^{*}})^{\top}\tilde{\mathbf{c}} = 0 \\
    &\partial \mathbf{m} = 0 \iff \left(\frac{\mathbf{s}}{\sqrt{\rho_{\tilde{\mathbf{w}}_{0}}}}+\frac{\tilde{\mathbf{s}}\sqrt{\gamma\kappa}}{\rho_{\tilde{\mathbf{c}}}}\right)^{\top}\mathcal{L}\left(f_{0}\left(\sqrt{\rho_{\tilde{\mathbf{w}}_{0}}}\mathbf{s}\right), \mathbf{s}\frac{\mathbf{m}^{\top}}{\sqrt{\rho_{\tilde{\mathbf{w}}_{0}}}}+\tilde{\mathbf{s}}\frac{\sqrt{\gamma\kappa}\mathbf{m}^{\top}}{\sqrt{\rho_{\tilde{\mathbf{c}}}}}+\frac{1}{\sqrt{p}}\tilde{\mathbf{Z}}\mathbf{U}\right) \notag \\
    &\hspace{2.5cm}+\frac{\sqrt{\gamma}\tilde{\mathbf{v}}^{\top}}{\rho_{\tilde{\mathbf{c}}}}\Omega^{-1/2}\partial r(\Omega^{-1/2}\left(\frac{\sqrt{\gamma}\tilde{\mathbf{c}}\mathbf{m}^{\top}}{\rho_{\tilde{\mathbf{c}}}}+\mathbf{U}\right)) = 0
\end{align}
Using Lemma \ref{lemma:asy_inter} while assuming the subgradients of $\mathcal{L}, r$ are pseudo-Lipschitz (we discuss this assumption in subsection \ref{subsec:Lip_com}), we obtain for $\mathbf{m}$
\begin{align}
     &\frac{1}{p}\mathbb{E}\left[\left( \bR_{r(\Omega^{-1/2}.),\hat{\mathbf{V}}^{-1}}\left(\mathbf{G}\hat{\mathbf{Q}}^{1/2}\hat{\mathbf{V}}^{-1}+\Omega^{-1/2}\mathbf{c}\boldsymbol{\nu}^{\top}\hat{\mathbf{V}}^{-1}+\frac{\sqrt{\gamma}\tilde{\mathbf{c}}\mathbf{m}^{\top}}{\rho_{\tilde{\mathbf{c}}}}\right)-\frac{\sqrt{\gamma}\tilde{\mathbf{c}}\mathbf{m}^{\top}}{\rho_{\tilde{\mathbf{c}}}}\right)^{\top}\tilde{\mathbf{c}} \right]= 0 \\
     &\iff \mathbf{m} = \frac{1}{\sqrt{dp}}\mathbb{E}\left[\tilde{\mathbf{c}}^{\top}\bR_{r(\Omega^{-1/2}.),\hat{\mathbf{V}}^{-1}}\left(\mathbf{G}\hat{\mathbf{Q}}^{1/2}\hat{\mathbf{V}}^{-1}+\Omega^{-1/2}\mathbf{c}\boldsymbol{\nu}^{\top}\hat{\mathbf{V}}^{-1}+\frac{\sqrt{\gamma}\tilde{\mathbf{c}}\mathbf{m}^{\top}}{\rho_{\tilde{\mathbf{c}}}}\right)\right]
\end{align}
 and for $\boldsymbol{\nu}$
 \begin{align}
     &\frac{1}{p}\mathbb{E}\bigg[\left(\frac{\mathbf{s}}{\sqrt{\rho_{\tilde{\mathbf{w}}_{0}}}}+\frac{\tilde{\mathbf{s}}\sqrt{\gamma\kappa}}{\rho_{\tilde{\mathbf{c}}}}\right)^{\top}\partial\mathcal{L}\left(f_{0}\left(\sqrt{\rho_{\tilde{\mathbf{w}}_{0}}}\mathbf{s}\right), \bR_{\mathcal{L}(\mathbf{y},.),\mathbf{V}}(\mathbf{s}\frac{\mathbf{m}^{\top}}{\sqrt{\rho_{\tilde{\mathbf{w}}_{0}}}}+\tilde{\mathbf{s}}\frac{\sqrt{\gamma\kappa}\mathbf{m}^{\top}}{\sqrt{\rho_{\tilde{\mathbf{c}}}}}+\mathbf{H}\hat{\mathbf{Q}}^{1/2})\right)\\
    &\hspace{1cm}+\frac{\sqrt{\gamma}\tilde{\mathbf{c}}^{\top}}{\rho_{\tilde{\mathbf{c}}}}\Omega^{-1/2}\partial r\left(\Omega^{-1/2}\left(\bR_{r(\Omega^{-1/2}.),\hat{\mathbf{V}}^{-1}}\left(\mathbf{G}\hat{\mathbf{Q}}^{1/2}\hat{\mathbf{V}}^{-1}+\Omega^{-1/2}\mathbf{c}\boldsymbol{\nu}^{\top}\hat{\mathbf{V}}^{-1}+\frac{\sqrt{\gamma}\tilde{\mathbf{c}}\mathbf{m}^{\top}}{\rho_{\tilde{\mathbf{c}}}}\right)\right)\right)\bigg] = 0
 \end{align}
Using the definition of D-resolvents, this is equivalent to 
 \begin{align}
     &\frac{1}{p}\mathbb{E}\bigg[\left(\frac{\mathbf{s}}{\sqrt{\rho_{\tilde{\mathbf{w}}_{0}}}}+\frac{\tilde{\mathbf{s}}\sqrt{\gamma\kappa}}{\rho_{\tilde{\mathbf{c}}}}\right)^{\top}\left(Id-\bR_{\mathcal{L}(\mathbf{y},.),\mathbf{V}}\left(.\right)\right)\left(\mathbf{s}\frac{\mathbf{m}^{\top}}{\sqrt{\rho_{\tilde{\mathbf{w}}_{0}}}}+\tilde{\mathbf{s}}\frac{\sqrt{\gamma\kappa}\mathbf{m}^{\top}}{\sqrt{\rho_{\tilde{\mathbf{c}}}}}+\mathbf{H}\hat{\mathbf{Q}}^{1/2}\right)\mathbf{V}^{-1}\\
    &+\frac{\sqrt{\gamma}\tilde{\mathbf{c}}^{\top}}{\rho_{\tilde{\mathbf{c}}}}\bigg(Id-\bR_{r(\Omega^{-1/2}.),\hat{\mathbf{V}}^{-1}}\left(.\right)\bigg)\left(\mathbf{G}\hat{\mathbf{Q}}^{1/2}\hat{\mathbf{V}}^{-1}+\Omega^{-1/2}\mathbf{c}\boldsymbol{\nu}^{\top}\hat{\mathbf{V}}^{-1}+\frac{\sqrt{\gamma}\tilde{\mathbf{c}}\mathbf{m}^{\top}}{\rho_{\tilde{\mathbf{c}}}}\right)\hat{\mathbf{V}}\bigg] = 0
 \end{align}
 which simplifies to 
 \begin{align}
     \boldsymbol{\nu}^{\top} = -\frac{1}{\sqrt{\gamma}p}\mathbb{E}\bigg[\left(\frac{\mathbf{s}}{\sqrt{\rho_{\tilde{\mathbf{w}}_{0}}}}+\frac{\tilde{\mathbf{s}}\sqrt{\gamma\kappa}}{\rho_{\tilde{\mathbf{c}}}}\right)^{\top}\left(Id-\bR_{\mathcal{L}(\mathbf{y},.),\mathbf{V}}\left(.\right)\right)\left(\mathbf{s}\frac{\mathbf{m}^{\top}}{\sqrt{\rho_{\tilde{\mathbf{w}}_{0}}}}+\tilde{\mathbf{s}}\frac{\sqrt{\gamma\kappa}\mathbf{m}^{\top}}{\sqrt{\rho_{\tilde{\mathbf{c}}}}}+\mathbf{H}\hat{\mathbf{Q}}^{1/2}\right)\mathbf{V}^{-1}\bigg]
 \end{align}
 which brings us to the following set of six self consistent equations
 \begin{align}
 \label{eq:SE_th}
        &\mathbf{Q} = \frac{1}{p}\mathbb{E}\bigg[\left(\bR_{r(\Omega^{-1/2}.),\hat{\mathbf{V}}^{-1}}\left(\mathbf{G}\hat{\mathbf{Q}}^{1/2}\hat{\mathbf{V}}^{-1}+\Omega^{-1/2}\mathbf{c}\boldsymbol{\nu}^{\top}\hat{\mathbf{V}}^{-1}+\frac{\sqrt{\gamma}\tilde{\mathbf{c}}\mathbf{m}^{\top}}{\rho_{\tilde{\mathbf{c}}}}\right)-\frac{\sqrt{\gamma}\tilde{\mathbf{c}}\mathbf{m}^{\top}}{\rho_{\tilde{\mathbf{c}}}}\right)^{\top}\\
    &\hspace{3cm}\left(\bR_{r(\Omega^{-1/2}.),\hat{\mathbf{V}}^{-1}}\left(\mathbf{G}\hat{\mathbf{Q}}^{1/2}\hat{\mathbf{V}}^{-1}+\Omega^{-1/2}\mathbf{c}\boldsymbol{\nu}^{\top}\hat{\mathbf{V}}^{-1}+\frac{\sqrt{\gamma}\tilde{\mathbf{c}}\mathbf{m}^{\top}}{\rho_{\tilde{\mathbf{c}}}}\right)-\frac{\sqrt{\gamma}\tilde{\mathbf{c}}\mathbf{m}^{\top}}{\rho_{\tilde{\mathbf{c}}}}\right)\bigg] \\
    &\hat{\mathbf{Q}} = \frac{1}{p}\mathbb{E}\bigg[\left(\left(\bR_{\mathcal{L}(\mathbf{y},.),\mathbf{V}}(.)-Id\right)\left(\mathbf{s}\frac{\mathbf{m}^{\top}}{\sqrt{\rho_{\tilde{\mathbf{w}}_{0}}}}+\tilde{\mathbf{s}}\frac{\sqrt{\gamma\kappa}\mathbf{m}^{\top}}{\sqrt{\rho_{\tilde{\mathbf{c}}}}}+\mathbf{H}\mathbf{Q}^{1/2}\right)\mathbf{V}^{-1}\right)^{\top} \notag \\
    &\hspace{3cm}\left(\left(\bR_{\mathcal{L}(\mathbf{y},.),\mathbf{V}}(.)-Id\right)\left(\mathbf{s}\frac{\mathbf{m}^{\top}}{\sqrt{\rho_{\tilde{\mathbf{w}}_{0}}}}+\tilde{\mathbf{s}}\frac{\sqrt{\gamma\kappa}\mathbf{m}^{\top}}{\sqrt{\rho_{\tilde{\mathbf{c}}}}}+\mathbf{H}\mathbf{Q}^{1/2}\right)\mathbf{V}^{-1}\right)\bigg] \\
    &\mathbf{V} = \frac{1}{p}\mathbb{E}\left[\hat{\mathbf{Q}}^{-1/2}\mathbf{G}^{\top}\bR_{r(\Omega^{-1/2}.),\hat{\mathbf{V}}^{-1}}\left(\mathbf{G}\hat{\mathbf{Q}}^{1/2}\hat{\mathbf{V}}^{-1}+\Omega^{-1/2}\mathbf{c}\boldsymbol{\nu}^{\top}\hat{\mathbf{V}}^{-1}+\frac{\sqrt{\gamma}\tilde{\mathbf{c}}\mathbf{m}^{\top}}{\rho_{\tilde{\mathbf{c}}}}\right)\right] \\
    &\hat{\mathbf{V}} =- \frac{1}{p}\mathbb{E}\left[\mathbf{Q}^{-1/2}\mathbf{H}^{\top}\left(\left(\bR_{\mathcal{L}(\mathbf{y},.),\mathbf{V}}(.)-Id\right)\left(\mathbf{s}\frac{\mathbf{m}^{\top}}{\sqrt{\rho_{\tilde{\mathbf{w}}_{0}}}}+\tilde{\mathbf{s}}\frac{\sqrt{\gamma\kappa}\mathbf{m}^{\top}}{\sqrt{\rho_{\tilde{\mathbf{c}}}}}+\mathbf{H}\mathbf{Q}^{1/2}\right)\mathbf{V}^{-1}\right)\right]\\
    &\mathbf{m} = \frac{1}{\sqrt{dp}}\mathbb{E}\left[\tilde{\mathbf{c}}^{\top}\bR_{r(\Omega^{-1/2}.),\hat{\mathbf{V}}^{-1}}\left(\mathbf{G}\hat{\mathbf{Q}}^{1/2}\hat{\mathbf{V}}^{-1}+\Omega^{-1/2}\mathbf{c}\boldsymbol{\nu}^{\top}\hat{\mathbf{V}}^{-1}+\frac{\sqrt{\gamma}\tilde{\mathbf{c}}\mathbf{m}^{\top}}{\rho_{\tilde{\mathbf{c}}}}\right)\right] \\
    &\boldsymbol{\nu}^{\top} = -\frac{1}{\sqrt{\gamma}p}\mathbb{E}\bigg[\left(\frac{\mathbf{s}}{\sqrt{\rho_{\tilde{\mathbf{w}}_{0}}}}+\frac{\tilde{\mathbf{s}}\sqrt{\gamma\kappa}}{\rho_{\tilde{\mathbf{c}}}}\right)^{\top}\left(Id-\bR_{\mathcal{L}(\mathbf{y},.),\mathbf{V}}\left(.\right)\right)\left(\mathbf{s}\frac{\mathbf{m}^{\top}}{\sqrt{\rho_{\tilde{\mathbf{w}}_{0}}}}+\tilde{\mathbf{s}}\frac{\sqrt{\gamma\kappa}\mathbf{m}^{\top}}{\sqrt{\rho_{\tilde{\mathbf{c}}}}}+\mathbf{H}\hat{\mathbf{Q}}^{1/2}\right)\mathbf{V}^{-1}\bigg]
\end{align}
This set of equations then characterizes the asymptotic distribution of the estimator $\mathbf{U}^{*}$ in the sense of Lemma \ref{lemma:asy_inter}, with the optimal values of $\mathbf{m}$ and $\boldsymbol{\nu}$. Using the definition of $\mathbf{U}^{*}$ and $\tilde{\mathbf{Z}}\mathbf{U}^{*}$, along with the definition of the function $r$ w.r.t. the original regularization function, a tedious but straightforward calculation allows reconstruct the asymptotic properties of $\mathbf{W}^{*}$ and of the set $\left\{\mathbf{X}_{k}\mathbf{w}_{k}^{*}\right\}_{k=1}^{K}$ given in the main text.
\subsection{Relaxing the strong convexity constraint}
Assuming the set of self consistent equations \eqref{eq:SE_th} have a unique fixed point regardless of the strong convexity assumption, this solution defines a unique set of six order parameters for the $\lambda_{2} = 0$ case. Furthermore, using 
Proposition \ref{approx_l2}, the unique estimator $\mathbf{W}^{*}(\lambda_{2})$ solving problem Eq.\eqref{eq:sc_student} for strictly positive $\lambda_{2}$ converges to the least-norm solution to the convex (but not strongly) Eq.\eqref{eq:student}. Thus, for any pseudo-Lipschitz observable of $\mathbf{U}^{*}(\lambda_{2})$, we have, one the one side a continuous function of $\lambda_{2}$ with a unique continuous extension at $\lambda_{2}=0$, and on the other side a function of $\lambda_{2}$ prescribed by the expectation taken w.r.t. the asymptotic Gaussian model parametrised by the state evolution parameters which is defined for all positive values of $\lambda_{2}$. Since both functions match for any strictly positive $\lambda_{2}$, continuity implies they also match for $\lambda_{2} = 0$ and we obtain the exact asymptotics of the least $\ell_{2}$ norm solution of problem Eq.\eqref{eq:student}. Regarding the uniqueness of the solution to the fixed point equations $\eqref{eq:SE_th}$, it is shown in \citep{Loureiro2021} that a similar set of equations, although for a vector valued variable, i.e. no ensembling, the solution is unique even if the original problem is not strictly convex. This is proven by showing that the fixed point equations are the solution of a strictly convex problem. We expect this to be true here as well, and leave this part for a longer version of this paper.
\subsection{A comment on non-pseudo-Lipschitz subgradients}
\label{subsec:Lip_com}
Provided the subgradients in Eq.\eqref{eq:subgrad_Lip} are pseudo-Lipschitz continuous, the proof goes through. However some 
convex functions commonly used in machine learning, such as the hinge loss or the $\ell_{1}$ norm for the penalty, have non-pseudo-Lipschitz gradient. 
To circumvent this issue, one can consider the optimization problem where both loss and regularization are replaced by their Moreau envelopes with strictly positive parameters 
$\tau_{1}, \tau_{2}$, as is done in \citep{celentano2020lasso} for the LASSO. Moreau envelopes are everywhere differentiable and have Lipschitz gradient for strictly positive values of their parameter \citep{bauschke2011convex}, thus the asymptotic characterization holds. One can then take the parameters to zero, using the fact that the limit at zero in the parameters of Moreau envelopes is well defined \citep{bauschke2011convex}, recovering the original function. Since proximity operators are defined as strongly convex problems, the sequence of problems defined by the proximal operator of a Moreau envelope with decreasing parameter converges to the proximal operator of the original function when the parameter is taken to zero. Finally, inverting the expectations on random quantities with the limit taking the parameters of the Moreau envelopes to zero can be done by verifying the dominated convergence theorem using the firm-nonexpansiveness of proximity operators and the corresponding bounds on their norms, see \citep{bauschke2011convex} Chapter 4, Section 1. We leave the details of this part to a longer version of this paper.
\subsection{Toolbox}
\label{sec:req_back}
In this section, we reproduce part of the appendix of \citep{loureiro2021learning} for completeness, in order to give an overview of the main concepts and tools on approximate message passing algorithms which will be required for the proof.\vspace{0.1cm}
\paragraph{Notations ---}
For a given function $\bphi\colon \mathbb{R}^{d \times K}\to \mathbb{R}^{d \times K}$, we write :
\begin{equation}
    \bphi(\bX) = \begin{bmatrix}
    \bphi^{1}(\bX) \\
    \vdots \\
    \bphi^{d}(\bX)\end{bmatrix} \in \mathbb{R}^{d \times K}
\end{equation}
where each $\bphi^{i} \colon \mathbb{R}^{d \times K} \to \mathbb{R}^{K}$. We then write the $K \times K$ Jacobian 
\begin{equation}
\label{eq:jacob}
    \frac{\partial \bphi^{i}}{\partial \bX_{j}}(\bX) = \begin{bmatrix}\frac{\partial \phi^{i}_1(\bX)}{\partial X_{j1}} & \cdots & \frac{\partial \phi^{i}_1(\bX)}{\partial X_{jK}} \\
    \vdots&\ddots&\vdots \\
    \frac{\partial \phi^{i}_K(\bX)}{\partial X_{j1}} & \cdots & \frac{\partial \phi^{i}_K(\bX)}{\partial X_{jK}}
    \end{bmatrix} \in \mathbb{R}^{K \times K}
\end{equation}
For a given matrix $\bQ\in \mathbb{R}^{K \times K}$, we write $\bZ \in \mathbb{R}^{n \times K} \sim \mathcal{N}(\boldsymbol 0,\bQ\otimes \bI_{n})$ to denote that the lines of $\bZ$ are sampled i.i.d.~from $\mathcal{N}(\boldsymbol 0,\bQ)$. Note that this is equivalent to saying that $\bZ = \tilde{\bZ}\bQ^{1/2}$ where $\tilde{\bZ} \in \mathbb{R}^{n \times K}$ is an i.i.d.~standard normal random matrix. The notation $\stackrel{\rm P}\simeq$ denotes convergence in probability.
We start with some definitions that commonly appear in the approximate message-passing literature, see e.g., \citep{bayati2011dynamics,javanmard2013state,berthier2020state}.
The main regularity class of functions we will use is that of pseudo-Lipschitz functions, which roughly amounts to functions with polynomially bounded first derivatives. We include the required scaling w.r.t.~the dimensions in the definition for convenience.
\begin{definition}[Pseudo-Lipschitz function]
For $k,K \in \mathbb{N}^{*}$ and any $n,m \in \mathbb{N}^{*}$, a function $\bphi \colon \mathbb{R}^{n \times K} \to \mathbb{R}^{m \times K}$ is called a \emph{pseudo-Lipschitz of order $k$} if there exists a constant $L(k,K)$ such that for any $\bX,\bY \in \mathbb{R}^{n \times K}$, 
\begin{equation}
    \frac{\norm{\bphi(\bX)-\bphi(\bY)}_{\rm F}}{\sqrt{m}} \leqslant L(k,K) \left(1+\left(\frac{\norm{\bX}_{\rm F}}{\sqrt{n}}\right)^{k-1}+\left(\frac{\norm{\bY}_{\rm F}}{\sqrt{n}}\right)^{k-1}\right)\frac{\norm{\bX-\bY}_{\rm F}}{\sqrt{n}}
\end{equation}
where $\norm{\bullet }_{\rm F}$ denotes the Frobenius norm.
Since $K$ will be kept finite, it can be absorbed in any of the constants.
\end{definition}
For example, the function $f:\mathbb{R}^{n \times K}\to \mathbb{R}, \bX \mapsto \frac{1}{n}\norm{\bX}_{F}^{2}$ is pseudo-Lipshitz of order 2.

\paragraph{Moreau envelopes and Bregman proximal operators ---} In our proof, we will also frequently use the notions of Moreau envelopes and proximal operators, see e.g., \citep{parikh2014proximal,bauschke2011convex}. These elements of convex analysis are often encountered in recent works on high-dimensional asymptotics of convex problems, and more detailed analysis of their properties can be found for example in \citep{thrampoulidis2018precise,Loureiro2021}. For the sake of brevity, we will only sketch the main properties of such mathematical objects, referring to the cited literature for further details. In this proof, we will mainly use proximal operators acting on sets of real matrices endowed with their canonical scalar product. Furthermore, proximals will be defined with matrix valued parameters in the following way: for a given convex function $f \colon\mathbb{R}^{d \times K}\to\mathbb R$, a given matrix $\bX \in \mathbb{R}^{d \times K}$ and a given symmetric positive definite matrix $\bV \in \mathbb{R}^{K \times K}$ with bounded spectral norm, we will consider operators of the type
\begin{equation}
    \underset{\bT \in \mathbb{R}^{d \times K}}{\arg\min} \left\{f(\bT)+\frac{1}{2}\mathrm{tr}\left((\bT-\bX)\bV^{-1}(\bT-\bX)^{\top}\right)\right\}
\end{equation}
This operator can either be written as a standard proximal operator by factoring the matrix $\bV^{-1}$ in the arguments of the trace:
\begin{equation}
    \Prox_{f(\bullet\bV^{1/2})}(\bX\bV^{-1/2})\bV^{1/2} \in \mathbb{R}^{d \times K}
\end{equation}
or as a Bregman proximal operator \citep{bauschke2003bregman} defined with the Bregman distance induced by the strictly convex, coercive function (for positive definite $\bV$)
\begin{equation}
\label{breg_dist}
    \bX \mapsto \frac{1}{2}\mathrm{tr}(\bX\bV^{-1}\bX^\top)
\end{equation}
which justifies the use of the Bregman resolvent 
\begin{equation}
\label{bregman_res}
   \underset{\bT \in \mathbb{R}^{d \times K}}{\arg\min} \left\{f(\bT)+\frac{1}{2}\mathrm{tr}\left((\bT-\bX)\bV^{-1}(\bT-\bX)^{\top}\right)\right\}=\left(\mathrm{Id}+\partial f(\bullet)\bV\right)^{-1}(\bX)
\end{equation}
Many of the usual or similar properties to that of standard proximal operators (i.e.~firm non-expansiveness, link with Moreau/Bregman envelopes,\dots) hold for Bregman proximal operators defined with the function \eqref{breg_dist}, see e.g., \citep{bauschke2003bregman,bauschke2006joint}. In particular, we will be using the equivalent notion to firmly nonexpansive operators for Bregman proximity operators, called $\emph{D-firm}$ operators. Consider the Bregman proximal defined with a differentiable, strictly convex, coercive function $g : \mathcal{X} \to \mathbb{R}$, where $\mathcal{X}$ is a given input Hilbert space. Let $T$ be the associated Bregman proximal of a given convex function $f : \mathcal{X} \to \mathbb{R}$, i.e., for any $\mathbf{x} \in \mathcal{X}$
\begin{equation}
    T(\mathbf{x}) = \underset{\mathbf{y} \in \mathcal{X}}{\arg\min} \left\{f(\mathbf{x})+D_{g}(\mathbf{x},\mathbf{y})\right\}
\end{equation}
Then $T$ is \emph{D-firm}, meaning it verifies
\begin{equation}
\label{eq:D-firm}
    \langle T\bx-T\by, \nabla g(T\bx)-\nabla g(T\by) \rangle \leqslant  \langle T\bx-T\by, \nabla g(\bx)-\nabla g(\by) \rangle
\end{equation}
for any $\mathbf{x},\mathbf{y}$ in $\mathcal{X}$.
\paragraph{Gradients of Bregman envelopes}
Consider, for any $\mathbf{X} \in \mathbb{R}^{d \times K}$ the Bregman envelope
\begin{equation}
    \mathcal{M}_{f,\mathbf{V}}(\mathbf{X}) = \inf_{\bT \in \mathbb{R}^{d \times K}} \left\{f(\bT)+\frac{1}{2}\mathrm{tr}\left((\bT-\bX)\bV^{-1}(\bT-\bX)^{\top}\right)\right\}
\end{equation}
then 
\begin{equation}
    \nabla_{\mathbf{X}} \mathcal{M}_{f,\mathbf{V}}(\mathbf{X}) = \left(\mathbf{X}-\left(\mathrm{Id}+\partial f(\bullet)\bV\right)^{-1}(\bX)\right)\mathbf{V}^{-1}
\end{equation}
and 
\begin{equation}
    \nabla_{\mathbf{V}} \mathcal{M}_{f,\mathbf{V}}(\mathbf{X}) = -\frac{1}{2}\norm{\left(\mathbf{X}-\left(\mathrm{Id}+\partial f(\bullet)\bV\right)^{-1}(\bX)\right)\mathbf{V}^{-1}}_{F}^{2}
\end{equation}
\paragraph{Gaussian concentration ---}
Gaussian concentration properties are at the root of this proof. Such properties are reviewed in more detail, for example, in \citep{vershynin2018high}. We refer the interested reader to related works for a more detailed discussion.

\paragraph{Approximate message-passing ---}
Approximate message-passing algorithms \cite{donoho2009message,rangan2011generalized,donoho2016high} are a statistical physics inspired \citep{mezard1987spin,zdeborova2016statistical} family of iterations which can be used to solve high dimensional inference problems . One of the central objects in such algorithms are the so called \emph{state evolution equations}, a low-dimensional recursion equations which allow to exactly compute the high dimensional distribution of the iterates of the sequence.
In this proof we will use a specific form of matrix-valued approximate message-passing iteration with non-separable non-linearities. In its full generality, the validity of the state evolution equations in this case is an extension of the works of \citep{javanmard2013state,berthier2020state} included in \citep{gerbelot2021graph}. Consider a sequence Gaussian matrices $\bA(n) \in \mathbb{R}^{n\times d}$ with i.i.d.~Gaussian entries, $A_{ij}(n)\sim\mathcal{N}(0,1/d)$. For each $n,d \in \mathbb{N}$, consider two sequences of pseudo-Lipschitz functions
\begin{equation}
\{\bh_{t} : \mathbb{R}^{n \times K}\to \mathbb{R}^{n \times K}\}_{t \in \mathbb{N}}\qquad \{\be_{t} : \mathbb{R}^{d \times K}\to \mathbb{R}^{d \times K}\}_{t \in \mathbb{N}}
\end{equation}
initialized on $\bu^{0} \in \mathbb{R}^{d\times K}$ in such a way that the limit
\begin{equation}
    \lim_{d \to \infty}\frac{1}{d}\norm{\be_{0}(\bu^{0})^\top\be_{0}(\bu^{0})}_{\rm F}
\end{equation}
exists, and recursively define:
\begin{align}
\label{canon_AMP}
	&\hspace{1cm} \bu^{t+1} = \bA^{\top}\bh_{t}(\bv^{t})-\be_{t}(\bu^{t})\langle \bh_{t}'\rangle^\top \\
	&\hspace{1cm} \bv^{t} = \bA\be_{t}(\bu^{t})-\bh_{t-1}(\bv^{t-1})\langle \be_{t}'\rangle^\top
\end{align}
where the dimension of the iterates are $\bu^{t} \in \mathbb{R}^{d \times K}$ and $\bv^{t} \in \mathbb{R}^{n \times K}$. The terms in brackets are defined as:
\begin{equation}
\label{eq:mat_ons}
 \langle \bh_{t}' \rangle = \frac{1}{d}\sum_{i=1}^{n} \frac{\partial \bh_{t}^{i}}{\partial \bv_{i}}(\bv^{t})\in \mathbb{R}^{K \times K} \quad \langle \be_{t}' \rangle = \frac{1}{d}\sum_{i=1}^{d} \frac{\partial \be_{t}^{i}}{\partial \bu_{i}}(\bu^{t}) \in \mathbb{R}^{K \times K}
\end{equation}
We define now the \emph{state evolution recursion} on two sequences of matrices $\{\bQ_{r,s}\}_{s,r\geqslant0}$ and $\{\hat{\bQ}_{r,s}\}_{s,r\geqslant1}$ initialized with $\bQ_{0,0} = \lim_{d \to \infty}\frac{1}{d}\be_{0}(\bu^{0})^\top \be_{0}(\bu^{0})$:
\begin{align}
   &\bQ_{t+1,s} = \bQ_{s,t+1} = \lim_{d \to \infty} \frac{1}{d}\mathbb{E}\left[\be_{s}(\hat{\bZ}^{s})^{\top}\be_{t+1}(\hat{\bZ}^{t+1})\right] \in \mathbb{R}^{K \times K} \\
    &\hat{\bQ}_{t+1,s+1} = \hat{\bQ}_{s+1,t+1} = \lim_{d \to \infty} \frac{1}{d}\mathbb{E}\left[\bh_{s}(\bZ^{s})^{\top}\bh_{t}(\bZ^{t})\right] \in \mathbb{R}^{K \times K}
\end{align}
where $(\bZ^{0},\dots,\bZ^{t-1}) \sim \mathcal{N}(\boldsymbol 0,\{\bQ_{r,s}\}_{0\leqslant r,s \leqslant t-1} \otimes \bI_{n}),(\hat{\bZ}^{1},\dots,\hat{\bZ}^{t}) \sim \mathcal{N}(\boldsymbol 0,\{\bhQ_{r,s}\}_{1\leqslant r,s \leqslant t} \otimes \bI_{d})$. Then the following holds
\begin{theorem}
\label{th:SE}
In the setting of the previous paragraph, for any sequence of pseudo-Lipschitz functions $\phi_{n}:(\mathbb{R}^{n \times K}\times \mathbb{R}^{d \times K})^{t} \to \mathbb{R}$, for $n,d \to +\infty$:
\begin{equation}
\phi_{n}(\bu^{0},\bv^{0},\bu^{1},\bv^{1},\dots,\bv^{t-1},\bu^{t}) \stackrel{\rm P}\simeq \mathbb{E}\left[\phi_{n}\left(\bu^{0},\bZ^{0},\hat{\bZ}^{1},\bZ^{1},\dots,\bZ^{t-1},\hat{\bZ}^{t}\right)\right]
\end{equation}
where $(\bZ^{0},\dots,\bZ^{t-1}) \sim \mathcal{N}(\boldsymbol 0,\{\bQ_{r,s}\}_{0\leqslant r,s \leqslant t-1} \otimes \bI_{n}),(\hat{\bZ}^{1},\dots,\hat{\bZ}^{t}) \sim \mathcal{N}(\boldsymbol 0,\{\bhQ_{r,s}\}_{1\leqslant r,s \leqslant t} \otimes \bI_{n})$.
\end{theorem}
\paragraph{A useful result from convex analysis}
Here we remind a result from \citep{bauschke2011convex} describing the limiting behavior of regularized estimators for vanishing regularization.
\begin{proposition}(Theorem 26.20 from \citep{bauschke2011convex})
\label{approx_l2}
  Let f and h be proper, lower semi-continuous, convex functions. Suppose that $\arg\min f \cap \mbox{dom}(g) \neq \emptyset$ and that $h$ is coercive and strictly convex. Then $g$ admits a unique minimizer $\mathbf{x}_{0}$ over $\arg\min f$ and , for every $\epsilon \in ]0,1[$, the regularized problem \begin{equation}
      \arg\min_{\mathbf{x}} f(\mathbf{x})+\epsilon h(\mathbf{x})
  \end{equation}
  admits a unique solution $\mathbf{x}_{\epsilon}$. If we assume further that $h$ is uniformly convex on any closed ball of the input space, then $\lim_{\epsilon \to 0} \mathbf{x}_{\epsilon} = \mathbf{x}_{0}$.
\end{proposition}
\end{document}